\documentclass[twoside,11pt]{article}

\usepackage{blindtext}

\usepackage[abbrvbib, preprint]{jmlr2e}
\usepackage{jmlr2e}
\usepackage{lmodern}
\usepackage{pdflscape}

\usepackage[utf8]{inputenc}
\usepackage{amsmath}
\usepackage{amssymb}
\usepackage{bm} 
\usepackage{mathrsfs} 
\usepackage{algorithm}
\usepackage{multicol}
\usepackage{multirow}
\usepackage{tikz}
\usepackage{graphicx} 
\usepackage{url}
\usetikzlibrary{arrows}
\usepackage{changepage}
\usepackage{booktabs}
\usetikzlibrary{arrows.meta}
\usepackage{algpseudocode}
\usepackage{changepage} 
\newcommand{\kronecker}{\otimes}     
\newcommand{\khatri}{\odot}          
\newcommand{\hadamard}{\circledast}      
\newcommand{\qedsymbol}{\rule{2.5mm}{2.5mm}}

\usepackage{lastpage}
\jmlrheading{23}{2025}{1-\pageref{LastPage}}{1/21; Revised 5/22}{9/22}{21-0000}{Kilic Afra and Batselier Kim}

\ShortHeadings{Bayesian Tensor Network Kernel Machines}{Kilic and Batselier}
\firstpageno{1}

\begin{document}

\title{Interpretable Bayesian Tensor Network Kernel Machines with Automatic Rank and Feature Selection}

\author{\name Afra Kilic \email h.a.kilic@tudelft.nl \\
           \addr Delft Center for Systems and Control\\
       Delft University of Technology\\
       Delft, 2628 CN, The Netherlands
       \AND
       \name Kim Batselier \email k.batselier@tudelft.nl \\
       \addr Delft Center for Systems and Control\\
       Delft University of Technology\\
       Delft, 2628 CN, The Netherlands}

\editor{My editor}

\maketitle

\begin{abstract}

Tensor Network (TN) Kernel Machines speed up model learning by representing parameters as low-rank TNs, reducing computation and memory use. However, most TN-based Kernel methods are deterministic and ignore parameter uncertainty. Further, they require manual tuning of model complexity hyperparameters like tensor rank and feature dimensions, often through trial-and-error or computationally costly methods like cross-validation. We propose Bayesian Tensor Network Kernel Machines, a fully probabilistic framework that uses sparsity-inducing hierarchical priors on TN factors to automatically infer model complexity. This enables automatic inference of tensor rank and feature dimensions, while also identifying the most relevant features for prediction, thereby enhancing model interpretability. All the model parameters and hyperparameters are treated as latent variables with corresponding priors. Given the Bayesian approach and latent variable dependencies, we apply a mean-field variational  inference to approximate their posteriors. We show that applying a mean-field approximation to TN factors yields a Bayesian ALS algorithm with the same computational complexity as its deterministic counterpart, enabling uncertainty quantification at no extra computational cost. Experiments on synthetic and real-world datasets demonstrate the superior performance of our model in prediction accuracy, uncertainty quantification, interpretability, and scalability.
\end{abstract}

\vspace{3mm}

\section{Introduction}

Kernel methods, such as Support Vector Machines (SVM) and Gaussian Processes (GP), are powerful tools for nonlinear learning, capable of approximating arbitrary functions given enough data and often matching or outperforming neural networks \citep{hammer2003, garriga2018, lee2018, novak2018}. They use a kernel function \(k(x, x') = \langle \bm{\varphi}(x), \bm{\varphi}(x') \rangle\) to implicitly map inputs into a Reproducing Kernel Hilbert Space, where nonlinear problems become linear. This kernel trick avoids explicit computations in high-dimensional spaces but scales poorly with data size, conventional methods require \(\mathcal{O}(N^2)\) memory and \(\mathcal{O}(N^3)\) flops for \(N\) training points \citep{rasmussen:06, suykens1999, suykens2002}. Kernel approximation methods reduce complexity to $\mathcal{O}(NM^2)$ by projecting data onto $M$ basis function where $M \ll N$: data-dependent \citep{williams2001, drineas2005} and data-independent \citep{rahimi2007} methods both converge at rate $\mathcal{O}(1 / \sqrt{M})$. Deterministic approaches \citep{dao2017, mutny2018, hensman2017, solin2020} can achieve faster convergence, but their reliance on tensor products causes $M$ to grow exponentially with input dimension $D$, limiting their use to low-dimensional settings. 

\vspace{1mm}
Tensor Network (TN) Kernel machines deal with high-dimensional feature spaces and large-scale data sets by considering the model
\begin{equation}
    f(\mathbf{x}_n) =\bm{\varphi}(\mathbf{x}_n)^T \mathbf{w},
    \label{eq:1}
\end{equation}
with model weights \(\mathbf{w}\) and a feature map $\bm{\varphi}(\cdot) : \mathbb{R}^D \rightarrow \mathbb{R}^{M_1 M_2 \cdots M_D}$ that is restricted to a Kronecker product form
\begin{equation}
    \bm{\varphi}(\mathbf{x}_n) = \bm{\varphi}^{(1)}(x^{(1)}_n) \otimes \bm{\varphi}^{(2)}(x^{(2)}_n) \otimes \dots \otimes \bm{\varphi}^{(D)}(x^{(D)}_n),
    \label{eq:2}
\end{equation}
where each factor is a feature vector $\bm{\varphi}^{(d)} \in \mathbb{R}^{M_d}$  for all \( d \in [1, D] \).
As a result only product kernels are considered, which includes common choices such as the polynomial \citep{batselier_tensor_2017, Novikov_A._Exponential_2018} and Gaussian kernel~\citep{wesel_large-scale_2021, wesel2024quantizedfourierpolynomialfeatures}. Assume now without loss of generality that $M_1=M_2=\cdots=M_D=M$ such that $\mathbf{w} \in \mathbb{R}^{M^D}$. Tensor Network Kernel Machines impose an additional low-rank TN constraint on the weights $\mathbf{w}$. In this way the exponential storage complexity of the parameters to-be-learned is reduced from $O(M^D)$ to $O(DMR)$ or $O(DMR^2)$, depending on whether a canonical polyadic decomposition (CPD)~\citep{harshman1970parafac, carroll_analysis_1970} or tensor train (TT)~\citep{oseledets2011tensor} is used, respectively. Likewise, the Kronecker product structures in both the feature map $\bm{\varphi}(\cdot)$ and low-rank TN $\mathbf{w}$ can be exploited during training via the alternating least squares (ALS) algorithm with a computational complexity of $O(N(MR)^2+D(MR)^3)$ when using CPD and $O(N(MR^2)^2+D(MR^2)^3)$ when using TT. Furthermore, the low-rank structure captures redundancies in \( \mathbf{w} \), resulting in a parsimonious representation, while also imposing implicit regularization by constraining the model’s degrees of freedom.

\vspace{1mm}
Despite the widespread use of TNs for complexity reduction, most existing techniques operate in deterministic settings and do not account for parameter uncertainty. Uncertainty is a key challenge in modern learning due to noisy and limited data. Bayesian inference provides a principled way to handle this uncertainty. Beyond quantifying uncertainty, it also helps control model complexity, choose model size, enforce sparsity, and include prior knowledge. However, the application of TN Kernel machines to reduce model complexity in probabilistic frameworks remains largely underexplored.

\vspace{1mm}
Using TNs requires specifying model complexity in advance by tuning hyperparameters like tensor rank $R$ and feature dimension $M_d$. Choosing the right complexity is essential but challenging. Typically, this is done by trial and error, which can be time-consuming and imprecise. More principled methods like maximum likelihood might cause overfitting if only the training data is used. With ample data, multiple models can be trained on subsets, and the best is chosen based on validation performance. However, when data is limited, using most for training leaves only small validation sets, which may give unreliable results. Cross-validation improves data use by rotating validation sets but is computationally costly, especially with multiple hyperparameters. Thus, a more efficient and sophisticated method to determine model complexity is needed.

\vspace{4mm}
\noindent \textbf{In this paper}, to address the gap in probabilistic kernel methods with low-rank TNs and to enable more principled and efficient hyperparameter tuning, we introduce a fully Bayesian tensor network kernel method (BTN-Kernel machines) that automatically infers both the tensor rank $R$ and feature dimension $M_d$ for all  \( d \in [1, D] \). To achieve this, we specify a sparsity-inducing hierarchical prior over the TN components with individual parameters associated to each feature dimension $M_d$ and rank $R$ component, promoting minimal tensor rank and feature dimension of each TN component. This penalizes the irrelevant components to shrink towards zero, allowing automatic inference of the tensor rank and feature dimensions during training. Additionally, the sparsity parameters placed on the feature dimensions highlights the most relevant features for prediction, enhancing model interpretability. All model parameters are treated as random variables with corresponding priors. Due to the complex dependencies and the fully Bayesian formulation, exact inference is intractable. We therefore employ a mean field variational Bayesian inference to obtain a deterministic approximation of the posterior distributions over all parameters and hyperparameters. We show that employing a mean field approximation on the factors of the TN results in a Bayesian ALS algorithm with an identical computational complexity as the conventional ALS algorithm. In other words, the uncertainty quantification of BTN-Kernel machines can be achieved at zero additional computational cost! We demonstrate the superior performance of our proposed Bayesian model in terms of prediction accuracy, uncertainty quantification, scalability and interpretability through numerous experiments.

\subsection{Related Work}

Probabilistic models for tensor decompositions have gained attention in collaborative filtering, tensor factorization and completion. Gibbs sampling with Gaussian priors on CP factors is used in \citep{rai_leveraging_2015, rai_scalable_2014, xiong_temporal_2010}, while variational Bayes (VB) is employed in \citep{Zhao_2015_rank_det, zhao_bayesian_2016}. Orthogonal factor recovery via Stiefel manifold optimization with VB is explored in \citep{cheng_probabilistic_2017}. Bayesian low-rank Tucker models are studied using VB \citep{chu_probabilistic_2009, zhao_bayesian_2015_tucker} and Gibbs sampling \citep{hoff_equivariant_2016}. An infinite Tucker model based on a \(t\)-process is proposed in \citep{xu_bayesian_2015}. A VB Tensor Train (TT) model with von Mises–Fisher priors appears in \citep{hinrich_probabilistic_2019}. More recently, a Bayesian TT model has been introduced that sequentially infers each component's posterior, proposing a probabilistic interpretation of the alternating linear scheme (ALS). A MATLAB toolbox supporting both VB and Gibbs sampling has also been developed \citep{hinrich_probabilistic_2020}.

\vspace{1mm}

Tensor regression and classification have been explored \citep{Hoff_2015, yang_bayesian_2016, guhaniyogi_bayesian_2017}, including connections to Gaussian processes \citep{yu_tensor_2018}. VB PARAFAC2 models analyze multiple matrices with a varying mode \citep{jorgensen_probabilistic_2018, jorgensen_analysis_2019}, while approaches for multiple 3-way tensors with varying modes appear in \citep{khan2014bayesian, Khan_2016}. While tensor networks are widely used to reduce model complexity in deterministic kernel methods, their use in probabilistic models remains limited. Existing approaches, such as the TT-GP model, apply variational inference by representing the posterior mean with a TT and using a Kronecker-structured covariance matrix \citep{izmailov2018ttgp}. More recently, the Structured Posterior Bayesian Tensor Network (SP-BTN) further reduces parameter complexity by imposing a low-rank structure on the mean of a CP-decomposed weight tensor, combined with Kronecker-structured local covariances \citep{konstantinidis2022vbttn}. Although both methods perform well in terms of predictive mean, their ability to capture uncertainty is limited by the diagonal structure of the covariance matrices. In fact, uncertainty quantification in the existing models remains largely unexplored. In addition, these methods still require manually tuned hyperparameters, often relying on costly or imprecise optimization procedures.

\section{Mathematical Background}
\label{sec:Background}
\subsection{Preliminaries and Notation}
 The order of a tensor is the number of dimensions, also known as ways or modes. Scalars are denoted by lowercase letters e.g., $a$. Vectors (first-order tensors) are denoted by boldface lowercase letters, e.g., $\mathbf{a}$. Matrices (second-order tensors) are denoted by boldface capital letters, e.g., $\mathbf{A}$. Higher-order tensors (order $\geq 3$) are denoted by boldface calligraphic letters, e.g., $\bm{\mathcal{A}}$. Given an $D$th-order tensor $\bm{\mathcal{A}} \in \mathbb{R}^{I_1 \times I_2 \times \cdots \times I_D}$, the $(i_1, i_2, \ldots, i_D)$-th entry is denoted by $a_{i_1 i_2 \cdots i_D}$, where the indices range from 1 to their capital version, e.g., $i_d = 1, 2, \ldots, I_d, \forall d \in [1, D]$. Often it is easier to avoid working with the tensors directly, thus, we will consider their vectorization. The vectorization \(\operatorname{vec}(\bm{\mathcal{A}}) \in \mathbb{R}^{I_1 I_2 \dots  I_D} \) of a tensor \(\bm{\mathcal{A}} \in \mathbb{R}^{I_1 \times I_2 \times \dots \times I_D}\) is a vector such that  
\(\operatorname{vec}(\bm{\mathcal{A}})_i = a_{i_1 i_2 \dots i_D},
\)
\noindent where the relationship between the linear index \(i\) and the multi-index \((i_1, i_2, \dots, i_D)\) is given by  
\[
i = i_1 + \sum_{d=2}^{D} (i_d - 1) \prod_{k=1}^{d-1} I_k.
\]

\noindent When applied to a matrix, the operator $\operatorname{vec}(\cdot)$ performs column-wise vectorization. The reshape operator $\mathcal{R}\{\cdot\}_{I \times J}$ reorganizes the entries of its input into a matrix of size $I \times J$, preserving the column-wise order consistent with the $\operatorname{vec}(\cdot)$ operator. The operator $\operatorname{diag}(\cdot)$ returns a diagonal matrix when applied to a vector, and extracts the main diagonal as a vector when applied to a matrix. The inner product of vectors $\mathbf{a}, \mathbf{b} \in \mathbb{R}^I$ is $c = \langle \mathbf{a}, \mathbf{b} \rangle = \sum_{i=1}^{I} a_i b_i$. The Kronecker product of two matrices \( \mathbf{A} \in \mathbb{R}^{I \times J} \) and \( \mathbf{B} \in \mathbb{R}^{K \times L} \) is denoted by \( \mathbf{A} \otimes \mathbf{B} \in \mathbb{R}^{KI \times LJ} \). The Khatri-Rao product \( \mathbf{A} \khatri \mathbf{B} \) of the matrices \( \mathbf{A} \in \mathbb{R}^{I \times J} \) and \( \mathbf{B} \in \mathbb{R}^{K \times J} \) is an \( IK \times J \) matrix obtained by taking the column-wise Kronecker product. The Hadamard (element-wise) product of matrices \( \mathbf{A} \in \mathbb{R}^{I \times J} \) and \( \mathbf{B} \in \mathbb{R}^{I \times J} \) is denoted by \( \mathbf{A} \hadamard \mathbf{B} \in \mathbb{R}^{I \times J}\). The identity matrix is conventionally denoted by \( \mathbf{I} \), with its dimensions either inferred from the context or explicitly stated as a subscript.

\subsection{Tensor Decompositions}
Tensor decompositions, also known as tensor networks, generalize the Singular Value Decomposition (SVD) from matrices to higher-order tensors \citep{kolda_tensor_2009}. Three widely used decompositions are the Tucker decomposition \citep{tucker_mathematical_1966, kolda_tensor_2009}, the Tensor Train (TT) decomposition \citep{oseledets2011tensor}, and the Canonical Polyadic Decomposition (CPD) \citep{harshman1970parafac, carroll_analysis_1970}. Each of these extends different properties of the SVD to tensors.  In this subsection, we briefly introduce these decompositions and discuss their advantages and limitations for modeling $\mathbf{w}$ in (2).  For a more detailed overview, we refer to \citep{kolda_tensor_2009} and references therein. We omit Tucker decomposition due to its non-uniqueness and high storage complexity $\mathcal{O}(R^D)$, which makes it impractical for high-dimensional data. Therefore, since CPD and TT decompositions have storage complexities of \( O(DMR) \) and \( O(DMR^2) \), respectively, they are both suitable options for representing \( \mathbf{w} \) in~\eqref{eq:1}.

\begin{definition}[Canonical Polyadic Decomposition]
A rank-$R$ Canonical Polyadic Decomposition of $\mathbf{w} = \operatorname{vec}(\bm{\mathcal{W}}) \in \mathbb{R}^{M^D}$ consists of $D$ factor matrices $\mathbf{W}^{(d)} \in \mathbb{R}^{M_d \times R}$, such that  
\begin{equation}
    \mathbf{w} = (\mathbf{W}^{(1)} \khatri \mathbf{W}^{(2)} \khatri \dots \khatri \mathbf{W}^{(D)}) \bm{1}_{R} = \sum_{r=1}^{R} \mathbf{w}^{(1)}_r \kronecker \mathbf{w}^{(2)}_r \kronecker \dots \kronecker \mathbf{w}^{(D)}_r.
    \label{eq:4}
\end{equation}  
Unlike matrix factorizations, CPD is unique under mild conditions \citep{kruskal1977three}. The primary storage cost arises from the $D$ factor matrices, leading to a complexity of $\mathcal{O}(R M D)$.  
\end{definition}

\begin{definition}[Tensor Train Decomposition]
Tensor Train decomposition represents $\mathbf{w} \in \mathbb{R}^{M^D}$ using $D$ third-order tensors $\bm{\mathcal{W}}^{(d)} \in \mathbb{R}^{R_d \times M_d \times R_{d+1}}$ such that  
\begin{equation}
    w_{i_1 i_2 \dots i_D} =
    \sum_{r_1=1}^{R_1} \dots \sum_{r_{D+1}=1}^{R_{D+1}} w^{(1)}_{r_1 i_1 r_2} \dots w^{(D)}_{r_D i_D r_{D+1}}.
    \label{eq:6}
\end{equation}  
The auxiliary dimensions $R_1, R_2, \dots, R_{D+1}$ are called TT-ranks. To ensure that the right-hand side of \eqref{eq:6} is a scalar, the boundary condition $R_1 = R_{D+1} = 1$ is required. TT decomposition is non-unique, with a storage complexity of $\mathcal{O}(R^2 M D)$ due to the $D$ tensors $\bm{\mathcal{W}}^{(d)}$.  
\end{definition}
 The CP-rank $R$ and TT-ranks $R_2, \dots, R_D$ serve as additional hyperparameters, often making CPD preferable in practice. For ease of the discussion, we focus on the learning algorithm for the CPD case. In section \ref{tt-discussion}, we briefly discuss what changes to our probabilistic model when using a TT decomposition.

\section{Probabilistic Tensor Network Kernel Machines}
\label{sec:the_model}
\subsection{Probabilistic Model and Priors}

 When $N$ inputs $\{\mathbf{x}_n\} ^ {N} _ {n=1}$ and outputs $\{y_n\} ^ {N} _ {n=1}$ are available for learning then model \eqref{eq:1} can be expressed as a linear matrix equation 

\begin{equation}
    \mathbf{y} = \mathbf{\Phi}^T \mathbf{w} + \mathbf{e}.
    \label{eq:3}
\end{equation}

 \noindent where $\bm{y} \in \mathbb{R}^{N}$ and $\mathbf{e} \in \mathbb{R}^N$, with $\mathbf{e} \sim \mathcal{N}(\mathbf{0}, \tau^{-1} \mathbf{I}_N)$. The matrix \( \mathbf{\Phi} \in \mathbb{R}^{M^D \times N} \) is constructed such that its \( n \)th column contains \( \bm{\varphi}(\mathbf{x}_n) \). Due to the tensor-product structure of the feature map in \eqref{eq:2}, $\mathbf{\Phi}$ can be expressed as a Khatri-Rao product of component matrices $\mathbf{\Phi} = \mathbf{\Phi}^{(1)} \khatri  \mathbf{\Phi}^{(2)} \khatri \dots \khatri  \mathbf{\Phi}^{(D)}$, where the columns of $\mathbf{\Phi}^{(d)}$ are the feature vectors $\bm{\varphi}^{(d)}$, for all  \( d \in [1, D] \). Consequently,  assuming a CP-decomposed weight vector $\mathbf{w}$ along with the Gaussian noise model, the likelihood of the observed outputs is given by

\begin{equation}
    p(\mathbf{y} \mid \{\mathbf{W}^{(d)}\}_{d=1}^{D}, \tau) = \prod_{n=1}^{N} \mathcal{N}\left(y_n \mid \bm{\varphi}(x_n)^T\mathbf{w}, \tau^{-1} \right),
    \label{eq:7}
\end{equation}
\noindent where the parameter $\tau$ denotes the noise precision. Factor matrices \(\mathbf{W}^{(d)} \in \mathbb{R}^{M_d \times R}\) can be represented either row-wise or column-wise as
\begin{equation}
    \mathbf{W}^{(d)} = 
    \begin{bmatrix}
    \mathbf{w}^{(d)}_1 & \cdots & \mathbf{w}^{(d)}_{m_d} & \cdots & \mathbf{w}^{(d)}_{M_d}
    \end{bmatrix}^T =  
    \begin{bmatrix}
    \mathbf{w}^{(d)}_{1} & \cdots & \mathbf{w}^{(d)}_{r} & \cdots & \mathbf{w}^{(d)}_{R}
    \end{bmatrix}.
    \label{eq:8}
\end{equation}
A subscript $m_d$ will always refer to a row of the factor matrix \(\mathbf{W}^{(d)}\), while a subscript $r$ to a column. The tensor rank $R$ and feature dimensions $\{M_d\}_{d=1}^D$ are tuning parameters whose selection is challenging and computationally costly. Therefore, we seek a principled and efficient model selection method that not only infers the model complexity via \(R\) and \(\{M_d\}_{d=1}^D\) but also effectively avoids overfitting. To achieve this, we specify a sparsity-inducing hierarchical prior over the factor matrices \(\{\mathbf{W}^{(d)}\}_{d=1}^D\) with continuous hyperparameters that control the variance related to their rows and the columns. This promotes simpler models by pushing unnecessary components toward zero, allowing automatic model selection during training. This form of prior is motivated by the framework of automatic relevance determination (ARD) introduced in the context of neural networks by \citet{neal1996bayesian} and \citet{mackay1994bayesian}. A similar prior has also been applied in a Bayesian CP tensor completion method, where it enables automatic determination of the CP rank \citep{Zhao_2015_rank_det}.

\vspace{1mm}

More specifically, we define the sparsity parameters as $\bm{\lambda}_R := [\lambda_1, \lambda_2, \dots, \lambda_R]$, where each \(\lambda_r\) regulates the strength of the \(r\)th \textbf{column} $\mathbf{w}^{(d)}_{r}$ of \(\mathbf{W}^{(d)}\), for all  \( d \in [1, D] \). The parameter set \(\bm{\lambda}_R\) is shared across all factor matrices, allowing uniform regularization across different modes. Similarly, we define   $\bm{\lambda}_{M_d} := [\lambda_{1_d}, \lambda_{2_d}, \dots, \lambda_{M_d}]$, where each \(\lambda_{m_d}\) controls the regularization of the \(m_d\)th \textbf{row} $\mathbf{w}^{(d)}_{m_d}$ of \(\mathbf{W}^{(d)}\). Unlike \(\bm{\lambda}_R\), the vector of precisions \(\bm{\lambda}_{M_d}\) is specific to each factor matrix \(\mathbf{W}^{(d)}\), for all  \( d \in [1, D] \), allowing for feature-dependent regularization.
The prior distribution for the vectorized factor matrices is a zero mean Gaussian prior 

\begin{equation}
    p(\operatorname{vec}(\mathbf{W}^{(d)})\mid \bm{\lambda}_R, \bm{\lambda}_{M_d}) = \mathcal{N} \left(\operatorname{vec}(\mathbf{W}^{(d)}) \mid \bm{0},  \bm{\Lambda}_R^{-1} \otimes \bm{\Lambda}_{M_d}^{-1} \right), \quad \forall d \in [1, D],
    \label{eq:9}
\end{equation}

\noindent where \(\bm{\Lambda}_R \otimes \bm{\Lambda}_{M_d} = \operatorname{diag}(\bm{\lambda}_R)  \otimes \operatorname{diag}(\bm{\lambda}_{M_d})\) represents the inverse covariance matrix
, also known as the precision matrix.  The factor \(\bm{\Lambda}_R\) is shared by all factor matrices and determines the \textbf{CP rank}. The factor \(\bm{\Lambda}_{M_d}\) is specific to each mode and determines feature dimension $M_d$ used for each factor matrix. To allow the CP rank and feature dimensions to be inferred from data, we place Gamma hyperpriors over the sparsity parameters \( \bm{\lambda}_R \) and \( \bm{\lambda}_{M_d} \). The hyperprior over \( \bm{\lambda}_R \) is factorized across rank components

\begin{equation}
    p(\bm{\lambda}_R) = \prod_{r=1}^{R} \text{Ga}(\lambda_r \mid c_0, d_0), 
    \label{eq:10}
\end{equation}

\noindent while the hyperprior over $\bm{\lambda}_{M_d}$ is factorized over their individual feature dimensions  
\begin{equation}
    p(\bm{\lambda}_{M_d}) = \prod_{m_d=1}^{M_d} \text{Ga}(\lambda_{m_d} \mid g_0, h_0), 
    \label{eq:11}
\end{equation}

\noindent where $\text{Ga}(x \mid a, b) = b^a x^{a-1} e^{-bx}/\Gamma(a)$ denotes a Gamma distribution and $\Gamma(a)$ is the Gamma function. To complete our model with a fully Bayesian treatment, we also place a hyperprior over the noise precision
$\tau$ , that is,
\begin{equation}
        p(\tau) = \text{Ga}(\tau \mid a_0, b_0).
        \label{eq:12}
\end{equation}
For simplicity of notation, all unknowns including latent variables and hyperparameters are collected and denoted together by $\Theta = \{\mathbf{W}^{(1)}, \ldots, \mathbf{W}^{(D)}, \bm{\lambda}_{M_1, ... \bm{\lambda}_{M_D}}, \bm{\lambda}_R,  \tau  \}$. The probabilistic graph model is illustrated in Figure \ref{fig:CPD_model}, from which we can express the joint distribution as
\begin{equation}
    p(\mathbf{y}, \Theta) = p(\mathbf{y} \mid \{\mathbf{W}^{(d)}\}_{d=1}^{D}, \tau) \prod_{d=1}^{D} \left\{p(\mathbf{W}^{(d)} \mid  \bm{\lambda}_R, \bm{\lambda}_{M_d}) \, p(\bm{\lambda}_{M_d}) \right\} \, p(\bm\lambda_R)  \, p(\tau).
    \label{eq:13}
\end{equation}

Figure~\ref{fig:CPD_model} illustrates the probabilistic graphical model corresponding to the joint distribution in \eqref{eq:13}. For simplicity, the deterministic design matrices~$\{\bm{\Phi}^{(d)}\}_{d=1}^D$ are omitted, as they are not treated as random variables. The weight vector~$\bm{w}$ is represented using a CPD decomposition with factor matrices~$\{\bm{W}^{(d)}\}_{d=1}^D$. These matrices are governed by a shared sparsity term~$\bm{\lambda}_R$ and individual sparsity terms~$\bm{\lambda}_{M_d}$ for each~$\bm{W}^{(d)}$. The shared term~$\bm{\lambda}_R$ determines the CP rank by penalizing entire columns across all factor matrices, while the individual terms~$\operatorname{diag}(\bm{\lambda}_{M_d})$ control the number of active feature dimensions~$M_d$ in each~$\bm{W}^{(d)}$ by penalizing its rows. Like the observation noise precision~$\tau$, both~$\bm{\lambda}_R$ and~$\bm{\lambda}_{M_d}$ are precision parameters and are assigned Gamma priors. The shape and scale parameters of these Gamma distributions, shown above the corresponding nodes in the figure, are hyperparameters that must be initialized before training. Hyperparameter initialization is discussed in Section \ref{implementation_details}.

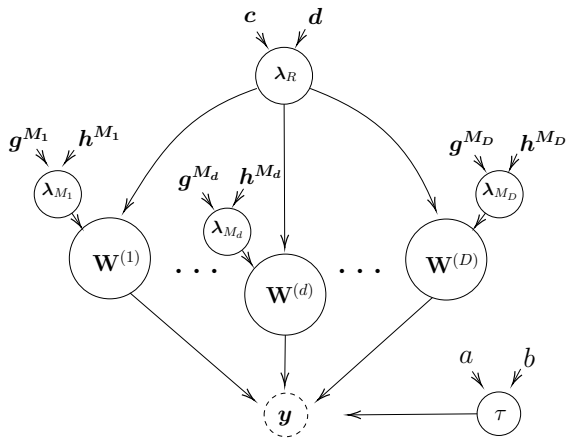
\begin{figure}[t]
    \centering
    \resizebox{0.5\textwidth}{!}{ 
        \begin{tikzpicture}[x=0.75pt,y=0.75pt,yscale=-1,xscale=1]
\draw   (167,223.5) .. controls (167,205) and (182,190) .. (200.5,190) .. controls (219,190) and (234,205) .. (234,223.5) .. controls (234,242) and (219,257) .. (200.5,257) .. controls (182,257) and (167,242) .. (167,223.5) -- cycle ;
\draw   (312,253.5) .. controls (312,235) and (327,220) .. (345.5,220) .. controls (364,220) and (379,235) .. (379,253.5) .. controls (379,272) and (364,287) .. (345.5,287) .. controls (327,287) and (312,272) .. (312,253.5) -- cycle ;
\draw   (445,224.5) .. controls (445,206) and (460,191) .. (478.5,191) .. controls (497,191) and (512,206) .. (512,224.5) .. controls (512,243) and (497,258) .. (478.5,258) .. controls (460,258) and (445,243) .. (445,224.5) -- cycle ;
\draw [dashed]  (327.99,353.1) .. controls (327.94,343.16) and (335.96,335.06) .. (345.9,335.01) .. controls (355.84,334.96) and (363.94,342.98) .. (363.99,352.93) .. controls (364.04,362.87) and (356.02,370.97) .. (346.08,371.01) .. controls (336.14,371.06) and (328.04,363.04) .. (327.99,353.1) -- cycle ;
\draw    (321,72.5) .. controls (321,59.52) and (331.52,49) .. (344.5,49) .. controls (357.48,49) and (368,59.52) .. (368,72.5) .. controls (368,85.48) and (357.48,96) .. (344.5,96) .. controls (331.52,96) and (321,85.48) .. (321,72.5) -- cycle ;
\draw    (344.5,96) -- (344.99,213) ;
\draw [shift={(345,215)}, rotate = 269.76] [color={rgb, 255:red, 0; green, 0; blue, 0 }  ][line width=0.75]    (10.93,-3.29) .. controls (6.95,-1.4) and (3.31,-0.3) .. (0,0) .. controls (3.31,0.3) and (6.95,1.4) .. (10.93,3.29)   ;
\draw   (278,201.5) .. controls (278,190.73) and (286.73,182) .. (297.5,182) .. controls (308.27,182) and (317,190.73) .. (317,201.5) .. controls (317,212.27) and (308.27,221) .. (297.5,221) .. controls (286.73,221) and (278,212.27) .. (278,201.5) -- cycle ;
\draw   (138,170.5) .. controls (138,159.73) and (146.73,151) .. (157.5,151) .. controls (168.27,151) and (177,159.73) .. (177,170.5) .. controls (177,181.27) and (168.27,190) .. (157.5,190) .. controls (146.73,190) and (138,181.27) .. (138,170.5) -- cycle ;
\draw   (503,170.5) .. controls (503,159.73) and (511.73,151) .. (522.5,151) .. controls (533.27,151) and (542,159.73) .. (542,170.5) .. controls (542,181.27) and (533.27,190) .. (522.5,190) .. controls (511.73,190) and (503,181.27) .. (503,170.5) -- cycle ;
\draw   (503.99,352.1) .. controls (503.94,342.16) and (511.96,334.06) .. (521.9,334.01) .. controls (531.84,333.96) and (539.94,341.98) .. (539.99,351.93) .. controls (540.04,361.87) and (532.02,369.97) .. (522.08,370.01) .. controls (512.14,370.06) and (504.04,362.04) .. (503.99,352.1) -- cycle ;
\draw    (217,252) -- (318.49,339.69) ;
\draw [shift={(320,341)}, rotate = 220.83] [color={rgb, 255:red, 0; green, 0; blue, 0 }  ][line width=0.75]    (10.93,-3.29) .. controls (6.95,-1.4) and (3.31,-0.3) .. (0,0) .. controls (3.31,0.3) and (6.95,1.4) .. (10.93,3.29)   ;
\draw    (467,256) -- (373.49,339.67) ;
\draw [shift={(372,341)}, rotate = 318.18] [color={rgb, 255:red, 0; green, 0; blue, 0 }  ][line width=0.75]    (10.93,-3.29) .. controls (6.95,-1.4) and (3.31,-0.3) .. (0,0) .. controls (3.31,0.3) and (6.95,1.4) .. (10.93,3.29)   ;
\draw    (503.99,352.1) -- (399,352.98) ;
\draw [shift={(397,353)}, rotate = 359.52] [color={rgb, 255:red, 0; green, 0; blue, 0 }  ][line width=0.75]    (10.93,-3.29) .. controls (6.95,-1.4) and (3.31,-0.3) .. (0,0) .. controls (3.31,0.3) and (6.95,1.4) .. (10.93,3.29)   ;
\draw    (345.5,287) -- (345.02,327) ;
\draw [shift={(345,329)}, rotate = 270.68] [color={rgb, 255:red, 0; green, 0; blue, 0 }  ][line width=0.75]    (10.93,-3.29) .. controls (6.95,-1.4) and (3.31,-0.3) .. (0,0) .. controls (3.31,0.3) and (6.95,1.4) .. (10.93,3.29)   ;
\draw    (537,137) -- (531.84,148.18) ;
\draw [shift={(531,150)}, rotate = 294.78] [color={rgb, 255:red, 0; green, 0; blue, 0 }  ][line width=0.75]    (10.93,-3.29) .. controls (6.95,-1.4) and (3.31,-0.3) .. (0,0) .. controls (3.31,0.3) and (6.95,1.4) .. (10.93,3.29)   ;
\draw    (277,169) -- (283.89,179.34) ;
\draw [shift={(285,181)}, rotate = 236.31] [color={rgb, 255:red, 0; green, 0; blue, 0 }  ][line width=0.75]    (10.93,-3.29) .. controls (6.95,-1.4) and (3.31,-0.3) .. (0,0) .. controls (3.31,0.3) and (6.95,1.4) .. (10.93,3.29)   ;
\draw    (311,166) -- (304.05,177.3) ;
\draw [shift={(303,179)}, rotate = 301.61] [color={rgb, 255:red, 0; green, 0; blue, 0 }  ][line width=0.75]    (10.93,-3.29) .. controls (6.95,-1.4) and (3.31,-0.3) .. (0,0) .. controls (3.31,0.3) and (6.95,1.4) .. (10.93,3.29)   ;
\draw    (502,138) -- (508.89,148.34) ;
\draw [shift={(510,150)}, rotate = 236.31] [color={rgb, 255:red, 0; green, 0; blue, 0 }  ][line width=0.75]    (10.93,-3.29) .. controls (6.95,-1.4) and (3.31,-0.3) .. (0,0) .. controls (3.31,0.3) and (6.95,1.4) .. (10.93,3.29)   ;
\draw    (140,133) -- (146.89,143.34) ;
\draw [shift={(148,145)}, rotate = 236.31] [color={rgb, 255:red, 0; green, 0; blue, 0 }  ][line width=0.75]    (10.93,-3.29) .. controls (6.95,-1.4) and (3.31,-0.3) .. (0,0) .. controls (3.31,0.3) and (6.95,1.4) .. (10.93,3.29)   ;
\draw    (171,134) -- (164.18,143.38) ;
\draw [shift={(163,145)}, rotate = 306.03] [color={rgb, 255:red, 0; green, 0; blue, 0 }  ][line width=0.75]    (10.93,-3.29) .. controls (6.95,-1.4) and (3.31,-0.3) .. (0,0) .. controls (3.31,0.3) and (6.95,1.4) .. (10.93,3.29)   ;
\draw    (322,83) .. controls (267.55,101.81) and (244.46,129.44) .. (212.96,183.36) ;
\draw [shift={(212,185)}, rotate = 300.19] [color={rgb, 255:red, 0; green, 0; blue, 0 }  ][line width=0.75]    (10.93,-3.29) .. controls (6.95,-1.4) and (3.31,-0.3) .. (0,0) .. controls (3.31,0.3) and (6.95,1.4) .. (10.93,3.29)   ;
\draw    (366,83) .. controls (418.47,102.8) and (448.4,129.46) .. (469.37,182.39) ;
\draw [shift={(470,184)}, rotate = 248.75] [color={rgb, 255:red, 0; green, 0; blue, 0 }  ][line width=0.75]    (10.93,-3.29) .. controls (6.95,-1.4) and (3.31,-0.3) .. (0,0) .. controls (3.31,0.3) and (6.95,1.4) .. (10.93,3.29)   ;
\draw    (169,186) -- (175.89,196.34) ;
\draw [shift={(177,198)}, rotate = 236.31] [color={rgb, 255:red, 0; green, 0; blue, 0 }  ][line width=0.75]    (10.93,-3.29) .. controls (6.95,-1.4) and (3.31,-0.3) .. (0,0) .. controls (3.31,0.3) and (6.95,1.4) .. (10.93,3.29)   ;
\draw    (310,217) -- (316.89,227.34) ;
\draw [shift={(318,229)}, rotate = 236.31] [color={rgb, 255:red, 0; green, 0; blue, 0 }  ][line width=0.75]    (10.93,-3.29) .. controls (6.95,-1.4) and (3.31,-0.3) .. (0,0) .. controls (3.31,0.3) and (6.95,1.4) .. (10.93,3.29)   ;
\draw    (511,186) -- (503.2,196.4) ;
\draw [shift={(502,198)}, rotate = 306.87] [color={rgb, 255:red, 0; green, 0; blue, 0 }  ][line width=0.75]    (10.93,-3.29) .. controls (6.95,-1.4) and (3.31,-0.3) .. (0,0) .. controls (3.31,0.3) and (6.95,1.4) .. (10.93,3.29)   ;
\draw    (322,36) -- (328.89,46.34) ;
\draw [shift={(330,48)}, rotate = 236.31] [color={rgb, 255:red, 0; green, 0; blue, 0 }  ][line width=0.75]    (10.93,-3.29) .. controls (6.95,-1.4) and (3.31,-0.3) .. (0,0) .. controls (3.31,0.3) and (6.95,1.4) .. (10.93,3.29)   ;
\draw    (362,34) -- (355.95,45.24) ;
\draw [shift={(355,47)}, rotate = 298.3] [color={rgb, 255:red, 0; green, 0; blue, 0 }  ][line width=0.75]    (10.93,-3.29) .. controls (6.95,-1.4) and (3.31,-0.3) .. (0,0) .. controls (3.31,0.3) and (6.95,1.4) .. (10.93,3.29)   ;
\draw    (500,317) -- (506.89,327.34) ;
\draw [shift={(508,329)}, rotate = 236.31] [color={rgb, 255:red, 0; green, 0; blue, 0 }  ][line width=0.75]    (10.93,-3.29) .. controls (6.95,-1.4) and (3.31,-0.3) .. (0,0) .. controls (3.31,0.3) and (6.95,1.4) .. (10.93,3.29)   ;
\draw    (540,315) -- (533.95,326.24) ;
\draw [shift={(533,328)}, rotate = 298.3] [color={rgb, 255:red, 0; green, 0; blue, 0 }  ][line width=0.75]    (10.93,-3.29) .. controls (6.95,-1.4) and (3.31,-0.3) .. (0,0) .. controls (3.31,0.3) and (6.95,1.4) .. (10.93,3.29)   ;

\draw (116,114) node [anchor=north west][inner sep=0.75pt]   [align=left] {\Large $\bm{g^{M_1}}$};
\draw (256,146) node [anchor=north west][inner sep=0.75pt]   [align=left] {\Large $\bm{g^{M_d}}$};
\draw (480,117) node [anchor=north west][inner sep=0.75pt]   [align=left] {\Large $\bm{g^{M_D}}$};
\draw (173,113) node [anchor=north west][inner sep=0.75pt]   [align=left] {\Large $\bm{h^{M_1}}$};
\draw (306,145) node [anchor=north west][inner sep=0.75pt]   [align=left] {\Large $\bm{h^{M_d}}$};
\draw (538,117) node [anchor=north west][inner sep=0.75pt]   [align=left] {\Large $\bm{h^{M_D}}$};
\draw (310,18) node [anchor=north west][inner sep=0.75pt]   [align=left] {\Large $\bm{c}$};
\draw (363,15) node [anchor=north west][inner sep=0.75pt]   [align=left] {\Large $\bm{d}$};
\draw (335,63.4) node [anchor=north west][inner sep=0.75pt]    {$\bm{\lambda}_{R}$};
\draw (143,159.4) node [anchor=north west][inner sep=0.75pt]    {$\bm{\lambda}_{M_{1}}$};
\draw (283,191.4) node [anchor=north west][inner sep=0.75pt]    {$\bm{\lambda}_{M_{d}}$};
\draw (507,159.4) node [anchor=north west][inner sep=0.75pt]    {$\bm{\lambda}_{M_{D}}$};
\draw (517,348) node [anchor=north west][inner sep=0.75pt]    {\Large$\tau$};
\draw (338,348) node [anchor=north west][inner sep=0.75pt]    {\Large$\bm{y}$};
\draw (328,242.4) node [anchor=north west][inner sep=0.75pt]  [rotate=-359.99]  {\Large$\mathbf{W}^{(d)}$};
\draw (185,215.4) node [anchor=north west][inner sep=0.75pt]  [rotate=-359.99]  {\Large$\mathbf{W}^{(1)}$};
\draw (460,217.4) node [anchor=north west][inner sep=0.75pt]  [rotate=-359.99]  {\Large$\mathbf{W}^{(D)}$};
\draw (488,299) node [anchor=north west][inner sep=0.75pt]   [align=left]  {\LARGE $a$};
\draw (541,296) node [anchor=north west][inner sep=0.75pt]   [align=left]  {\LARGE $b$};

\node at (275, 235) {\Huge $\cdots$};  
\node at (410, 235) {\Huge $\cdots$};  

        \end{tikzpicture}
    }

    \caption{\footnotesize Representation of BTN-Kernel machines with the CPD-decomposed weight vector $\mathbf{w}$ as a probabilistic graphical model showing the hierarchical sparsity inducing priors over the factor matrices $\{\bm{W}^{(d)}\}_{d=1}^D$ by the sparsity parameters $\boldsymbol{\lambda}_{R}$ and $\{\boldsymbol{\lambda}_{M_d}\}_{d=1}^D$. The dashed node denotes the observed data $\mathbf{y}$, while the solid nodes represent random variables. Shape and scale hyperparameters of the Gamma priors placed on $\boldsymbol{\lambda}_{R}$, $\{\boldsymbol{\lambda}_{M_d}\}_{d=1}^D$ and $\tau$ are shown as unbounded nodes.}

    \label{fig:CPD_model}
\end{figure}

By combining the likelihood in \eqref{eq:7}, the priors of model parameters in \eqref{eq:9} and the hyperpriors in \eqref{eq:10}, \eqref{eq:11} and \eqref{eq:12}, the logarithm of the joint distribution of the model is given by 

\begin{equation}
      \begin{aligned}
     l(\Theta) &= - \frac{\tau}{2} \| \mathbf{y} -\langle\mathbf{\Phi},\mathbf{w}\rangle \|_F^{2}   - \frac{1}{2}  \sum_{d=1}^D \sum_r \sum_{m_d}  w^{(d)}_{m_dr} \, \lambda_r\, \lambda_{m_d}\, w^{(d)}_{m_dr}  + \left( \frac{N}{2} + a_0 -1\right)  \operatorname{ln} \tau \\ 
     &+ \sum_r  \left( \frac{\sum_d M_d}{2} + (c_0^r - 1) \right)\operatorname{ln} \lambda_r 
    +  \sum_d^D \sum_{m_d} \left(\frac{R}{2} + (g_0^{dm_d} - 1) \right) \operatorname{ln}\lambda_{m_d}^{d} \\ 
    &- \sum_d^D \sum_{m_d}  h_0^{dm_d} \lambda_{m_d}^{d} - \sum_r d_0^r \lambda_r - b_0 \tau + \text{const},
      \end{aligned}
      \label{eq:14}
\end{equation}

 \noindent where $N$ denotes the total number of observations.  See Section 2 of the Appendix for a detailed derivation. Without loss of generality, we can perform maximum a posteriori (MAP) estimation of $\Theta$ by maximizing \eqref{eq:14}, which is, to some extent, equivalent to optimizing a squared error function with regularizations imposed on the factor matrices and additional constraints imposed on the regularization parameters. However, our goal is to develop a method that, instead of relying on point estimates, computes the full posterior distribution of all variables in \( \Theta \) given the observed data, that is,
\begin{equation}
    p(\Theta | \mathbf{y}) = \frac{p(\mathbf{y}, \Theta)}{\int p(\mathbf{y}, \Theta) d \Theta}. 
    \label{eq:15}
\end{equation}
Based on the posterior distribution of $\Theta$, the predictive distribution over unseen data points, denoted $\tilde{y_i}$ can be inferred by 
\begin{equation}
     p(\tilde{y_i} \mid \mathbf{y}) = \int p\left(y_i \mid \Theta\right) p\left(\Theta \mid \mathbf{y} \right)d\Theta.
     \label{eq:16} 
\end{equation}

\subsection{Model Learning}

An exact Bayesian inference in \eqref{eq:15} and \eqref{eq:16} requires integrating over all latent variables and hyperparameters, making it analytically intractable. In this section, we present the development of a deterministic approximate inference method within the variational Bayesian (VB) framework \citep{winn_variational_2005} to learn the probabilistic CP model. In this approach, we seek a variational distribution \( q(\Theta) \) that approximates the true posterior distribution \( p(\Theta \mid \mathbf{y}) \) by minimizing the Kullback–Leibler (KL) divergence, that is,

\begin{equation}
\begin{aligned}
   \text{KL}\left( q(\Theta) \, \|\, p(\Theta \mid \mathbf{y}) \right) &= \int q(\Theta) \ln \frac{q(\Theta)}{p(\Theta \mid \mathbf{y})} \, d\Theta
    \\
   &=  \underbrace{\ln p(\mathbf{y})}_{\text{evidence}} - \underbrace{\int q(\Theta) \ln \frac{p(\mathbf{y}, \Theta)}{q(\Theta)} \, d\Theta}_{\text{$\mathcal{L}(q)$}},
\end{aligned}
\label{eq:17}
\end{equation}

\noindent where \( \ln p(\mathbf{y}) \) represents the model evidence, and its lower bound is defined by \( \mathcal{L}(q) = \mathbb{E}_{q(\Theta)} [\ln p(\mathbf{y}, \Theta)] \). Since the model evidence is a constant, the maximum of the lower bound occurs when the KL divergence vanishes, implying \( q(\Theta) = p(\Theta \mid \mathbf{y}) \). We assume a mean field approximation that factorizes the variational distribution over each variable \( \theta_j \in \Theta\), allowing it to be expressed as
\begin{equation}
    q(\Theta) = \prod_{d=1}^{D} \,\left\{q_{\mathbf{W}^{(d)}}(\mathbf{W}^{(d)}) \, q_{\bm{\lambda}_{M_d}}(\bm{\lambda}_{M_d})\right\} q_{\lambda_{R}}(\bm{\lambda}_{R}) q_{\tau}(\tau).
    \label{eq:18}
\end{equation}
It is important to note that this factorization assumption is the only assumption imposed on the distribution and the specific functional forms of the individual factors \( q_j(\theta_j) \) can be explicitly derived in turn. Consider the $j$th variable $\theta_j$.  The optimal solution for $\theta_j$, obtained by maximizing \( \mathcal{L}(q)\) and the maximum occurs when 
\begin{equation}
    \ln q_j(\theta_j) = \mathbb{E}_{q(\Theta \setminus \theta_j)} [\ln p(\mathbf{y}, \Theta)] + \text{const},
    \label{eq:19}
\end{equation}
\noindent where \( \mathbb{E}_{q(\Theta \setminus \theta_j)} [\cdot] \) denotes the expectation taken with respect to the variational distributions of all variables except \( \theta_j \). For a detailed derivation and proof, see \cite{bishop2006pattern}. Since all parameter distributions belong to the exponential family and are conjugate to their corresponding prior distributions, we can derive closed-form posterior update rules for each parameter in \( \Theta \) using \eqref{eq:19}. Learning the BTN-Kernel machines can then be done by initializing the distributions $q_j(\theta_j)$ appropriately and replacing each in turn with a revised estimate given by the update rule.

\subsubsection{Posterior distribution of factor matrices}

To derive the update rule for the \( d \)th factor matrix \( \mathbf{W}^{(d)} \), we first need to introduce the following theorem.

\begin{theorem} 
 Given a set of matrices \( \mathbf{W}^{(d)} \) for all \( d \in [1,  D] \), the following linear relation holds:
\begin{equation}
    \mathbf{\Phi}^T \mathbf{w} = \operatorname{vec}(\mathbf{W}^{(d)})^T \textbf{G}^{(d)},
    \label{eq:20}
\end{equation}
where
\[
    \mathbf{G}^{(d)} := \mathbf{\Phi}^{(d)} \khatri \left( \circledast_{k \neq d} \mathbf{W}^{(k)^T} \mathbf{\Phi}^{(k)} \right) \in \mathbb{R}^{M_d R \times N},
\]
and \(\operatorname{vec}(\mathbf{W}^{(d)}) \in \mathbb{R}^{M_d R}\) is the column-wise vectorization of the \( d \)th factor matrix. The matrix \(\mathbf{G}^{(d)}\) can be interpreted as the design matrix corresponding to \(\mathbf{W}^{(d)}\).
\label{theorem3}
\end{theorem}

\begin{proof}
See Section~1 of the Appendix.
\end{proof}
The update for the $d$th factor matrix $\mathbf{W}^{(d)}$ is based on two main sources of information, as shown in Figure \ref{fig:CPD_model}. The first source comes from the observed data and related variables, including the other factor matrices $\mathbf{W}^{(k)}$ for $k \ne d$ and the hyperparameter $\tau$, which are included in the likelihood term \eqref{eq:7}. The second source comes from sparsity parameters $\bm{\Lambda}_R$ and $\bm{\Lambda}_{M_d}$, which contribute through the prior term \eqref{eq:9}. By applying \eqref{eq:19}, the posterior mean $\operatorname{vec}(\tilde{\mathbf{W}}^{(d)})$ and covariance matrix $\mathbf{\Sigma}^{(d)}$ of
\begin{equation}
    q_{\mathbf{W}^{(d)}}(\operatorname{vec}(\mathbf{W}^{(d)})) = \mathcal{N}\left(\operatorname{vec}(\mathbf{W}^{(d)}) \mid  \operatorname{vec}(\tilde{\mathbf{W}}^{(d)}), \mathbf{\Sigma}^{(d)}\right),  \quad \forall d \in [1, D]
    \label{eq:21}
\end{equation}
are updated by
\begin{equation}
    \begin{aligned}
        \operatorname{vec}(\tilde{\mathbf{W}}^{(d)}) &= \mathbb{E}_q\left[\tau\right]  \mathbf{\Sigma}^{(d)} \hspace{0.3mm}\mathbb{E}_q\left[ \textbf{G}^{(d)} \right] \hspace{1.5mm}\mathbf{y}, \\
        \mathbf{\Sigma}^{(d)} &=\left[\mathbb{E}_q\left[\tau\right] \mathbb{E}_q\left[\textbf{G}^{(d)} \textbf{G}^{(d)T} \right] + \mathbb{E}_q\left[\bm{\Lambda}_R\right] \otimes \mathbb{E}_q\left[\bm{\Lambda}_{M_d} \right]\right]^{-1}.
    \end{aligned}
    \label{eq:22}
\end{equation}
See Section 3 of the Appendix for a detailed derivation of \eqref{eq:22}. The matrix $\mathbb{E}_q\left[\mathbf{G}^{(d)} \mathbf{G}^{(d)T} \right]$ represents the posterior covariance of the model fitting term $\mathbf{\Phi}^T \mathbf{w}$ in \eqref{eq:20}, excluding the \( d \)th factor matrix $\mathbf{W}^{(d)}$. In other words, it corresponds to the posterior covariance of the design matrix $\mathbf{G}^{(d)}$  which combines the features \( \mathbf{\Phi}^{(d)} \) with all factor matrices \( \mathbf{W}^{(k)} \) for all \( k \neq d \), leaving out the \( d \)th factor matrix. This term cannot be computed straightforwardly, and therefore, we first need to introduce the following results. In order to express \( \mathbb{E}_q\left[\mathbf{G}^{(d)} \mathbf{G}^{(d)T} \right] \) in terms of the posterior parameters of the factor matrices in \eqref{eq:22}, we reformulate it by isolating the random variables in the expression as in the following theorem. 

\begin{theorem}
For any fixed \(d \in [1,  D] \), and assuming that the random matrices \(\mathbf{W}^{(k)}\) are independent for all \( k \neq d \), the following linear relation holds:
\begin{equation*}
    \mathbb{E}_q\left[\mathbf{G}^{(d)} \mathbf{G}^{(d)T} \right] = \mathcal{R}\left\{\left( \bm{\Phi}^{(d)} \khatri \bm{\Phi}^{(d)} \right) \mathop{\circledast}_{k \neq d}^{D} \left( \bm{\Phi}^{(k)} \khatri \bm{\Phi}^{(k)} \right)^T \mathbb{E}_q \left[ \mathbf{W}^{(k)} \kronecker \mathbf{W}^{(k)} \right]\right\}_{M_dR \,\times \,M_dR},
    \label{eq:23}
\end{equation*}
where
\begin{equation}
    \mathbb{E}_q \left[ \mathbf{W}^{(k)} \kronecker \mathbf{W}^{(k)} \right] = \mathbb{E}_q \left[ \mathbf{W}^{(k)} \right] \kronecker \mathbb{E}_q \left[ \mathbf{W}^{(k)} \right] + \operatorname{Var} \left[ \mathbf{W}^{(k)} \kronecker \mathbf{W}^{(k)} \right],
    \label{eq:24}
\end{equation}
and \( \mathcal{R}\left\{\cdot\right\}_{M_dR \times M_dR} \) is the operator that reshapes its \( M^2 \times R^2 \) argument into a matrix of size \( M_d R \times M_d R \).
\label{theorem4}
\end{theorem}

\begin{proof}
See Section~4 of the Appendix.
\end{proof}

\noindent We require Theorem~\ref{theorem4} to evaluate the variance term  \(\operatorname{Var}[\mathbf{W}^{(k)} \otimes \mathbf{W}^{(k)}]\)  that appears in \(\mathbb{E}_q[\mathbf{G}^{(d)} \mathbf{G}^{(d)T}]\).  Notice that when this variance is zero, we recover the standard ALS update  equation for \(\mathbf{W}^{(d)}\). Now, in order to evaluate \eqref{eq:24} in terms of the posterior parameters of the factor matrices from \eqref{eq:22}, we need to introduce the following Lemma.

\begin{lemma}
\label{lemma1}
    Let \( \mathbf{W}^{(k)} \in \mathbb{R}^{M_k \times R} \) and let \( \operatorname{vec}(\mathbf{W}^{(k)}) \in \mathbb{R}^{M_kR} \) be its column-wise vectorization. Then,  
\begin{equation}
    \mathbf{W}^{(k)} \kronecker \mathbf{W}^{(k)} =  \mathcal{R} \left\{\operatorname{vec}(\mathbf{W}^{(k)}) \operatorname{vec}(\mathbf{W}^{(k)})^T \right\}_{M_k^2 \times R^2}
    \label{eq:25}
\end{equation}  
where \( \mathcal{R} \left\{. \right\}_{M_k^2 \times R^2} \) reshapes the outer product \( \operatorname{vec}(\mathbf{W}^{(k)}) \operatorname{vec}(\mathbf{W}^{(k)})^T \) from dimensions \( M_kR \times M_kR \) to \( M_k^2 \times R^2 \).
\end{lemma}
The intuitive interpretation of Lemma~\ref{lemma1} is that the Kronecker product of \( \mathbf{W}^{(k)} \) with itself is the same as the outer product of $\operatorname{vec}(\mathbf{W}^{(k)})$ with itself, but reshaped into a structured block form. Using this identity, the variance term in \eqref{eq:24} can be computed by reshaping  the covariance matrix \( \mathbf{\Sigma}^{(k)} \) of \( \operatorname{vec}(\mathbf{W}^{(k)}) \) as 

\begin{equation}
    \operatorname{Var} \left[ \mathbf{W}^{(k)} \kronecker \mathbf{W}^{(k)} \right] = \mathcal{R} \left\{ \mathbf{\Sigma}^{(k)} \right\}_{M_k^2 \times R^2}.
    \label{eq:26}
\end{equation}  
By combining \eqref{eq:26} with $\mathbb{E}_q \left[ \mathbf{W}^{(k)} \right] = \tilde{\mathbf{W}}^{(k)}$ from \eqref{eq:22}, the term $ \mathbb{E}_q\left[\mathbf{G}^{(d)} \mathbf{G}^{(d)T} \right]$ in \eqref{eq:24} can be computed explicitly as

\begin{equation}
  \begin{aligned}
 \mathcal{R}\left\{\left( \bm{\Phi}^{(d)} \khatri \bm{\Phi}^{(d)} \right) \mathop{\circledast}_{k \neq d}^{D} \left( \bm{\Phi}^{(k)} \khatri \bm{\Phi}^{(k)} \right)^T \left( \tilde{\mathbf{W}}^{(k)} \kronecker \tilde{\mathbf{W}}^{(k)} +  \mathcal{R} \left\{ \mathbf{\Sigma}^{(k)} \right\}_{M_k^2 \times R^2} \right)\right\}_{M_dR \times M_dR}. \\
       \\
  \end{aligned}
  \label{eq:27}
\end{equation}  

\vspace{1mm} 
An intuitive interpretation of Equation~\eqref{eq:22} is as follows. The posterior covariance \( \mathbf{\Sigma}^{(d)} \) is updated by combining prior information, \( \mathbb{E}_q[\bm{\Lambda}_R] \) and \( \mathbb{E}_q[\bm{\Lambda}_{M_d}] \) with contributions from other factor matrices $\mathbb{E}_q\left[\textbf{G}^{(d)} \textbf{G}^{(d)T} \right]$. These contributions are scaled by the corresponding feature matrix, as shown in \eqref{eq:24}, and represent the data-dependent part of the update. The impact of this data-dependent term is further weighted by \( \mathbb{E}_q[\tau] \), which reflects the model's fit to the data. Hence, better model fit leads to greater reliance on information from the other factors rather than the prior.

The posterior mean \( \operatorname{vec}(\tilde{\mathbf{W}}^{(d)}) \) is computed by projecting the outcome variable \( \mathbf{y} \) onto the expected design matrix \( \mathbb{E}_q[\mathbf{G}^{(d)}] \), which captures interactions between the features and the other factor matrices except the $d$th one. This projection is then scaled by the posterior covariance \( \mathbf{\Sigma}^{(d)} \). Finally, the result is scaled by the expected noise precision \( \mathbb{E}_q[\tau] \), amplifying the influence of the data when the model fit is good.

\subsubsection{\texorpdfstring{Posterior distribution of $\bm{\Lambda}_{R}$ and $\bm{\Lambda}_{M_d}$}{Posterior distribution of Lambda\_R and Lambda\_mD}}

Instead of point estimation through optimization, learning the posterior of $\bm{\lambda}_R$ is crucial for automatic rank inference. As seen in Figure 1, the inference of $\bm{\lambda}_R$ can be done by receiving messages from all the factor matrices and incorporating the messages from its hyperprior. By applying \eqref{eq:19}, we can identify the posteriors of $\lambda_r$, $\forall r \in [1, R]$ as independent Gamma distributions, 

\begin{equation}
    q_{\bm{\lambda}_R}(\bm{\lambda}_R) = \prod_{r=1}^R \text{Ga}(\lambda_r \mid c_N^r, d_N^r), 
    \label{eq:28}
\end{equation}

\noindent where $c_N^r, d_N^r$ denote the posterior parameters learned from $N$ observations and are updated by 

\begin{equation}
    \begin{aligned}
        c_N^r &= c_0^r + \frac{1}{2} \sum_{d=1}^D M_d, \\
        d_N^r &= d_o^r + \frac{1}{2} \sum_{d=1}^D \mathbb{E}_q\left[ \mathbf{w}^{(d)T}_{r} \bm{\Lambda}_{M_d} \mathbf{w}^{(d)}_{r} \right].
    \end{aligned}
    \label{eq:29}
\end{equation}
See Section 5 of the Appendix for a detailed derivation of \eqref{eq:29}. The expectation of the inner product of the $r$th column in the $d$th factor matrix $ \mathbf{w}^{(d)}_{r}$ with respect to the $q$-distribution can be computed using the posterior parameters for $ \mathbf{W}^{(d)}$ in equation \eqref{eq:22} as

\begin{equation}
\mathbb{E}_q \left[ \mathbf{w}^{(d)T}_{r} \bm{\Lambda}_{M_d} \mathbf{w}^{(d)}_{r} \right] = \tilde{\mathbf{w}}^{(d)T}_{r} \bm{\Lambda}_{M_d} \tilde{\mathbf{w}}^{(d)}_{r} + \bm{\lambda}_{M_d}^T \, \operatorname{Var} \left(\mathbf{w}^{(d)}_{r} \right),
\label{eq:30}
\end{equation}

\noindent where $ \tilde{\mathbf{w}}^{(d)}_{r} $ denotes the $r$th column of $ \tilde{\mathbf{W}}^{(d)} $. The second term accounts for the uncertainty in the posterior distribution, where the variance of each element in \( \mathbf{w}^{(d)}_{r} \) is weighted by the corresponding \( \lambda_{m_d} \) in \( \bm{\lambda}_{M_d} \).   Notice that the diagonal elements of the posterior covariance matrix \( \mathbf{\Sigma}^{(d)} \) corresponds to the variance of every element in $\mathbf{W}^{(d)}$, i.e., $\operatorname{Var}(w^{(d)}_{m_dr})$.  Hence, the variance term $\operatorname{Var}(\mathbf{w}^{(d)}_{r})$ in \eqref{eq:30} is also contained in $\operatorname{diag}(\mathbf{\Sigma}^{(d)})$. To efficiently compute this variance $\forall r \in [1,R]$ at once, we reshape the diagonal elements of \( \mathbf{\Sigma}^{(d)} \) into a matrix \( \mathbf{V}^{(d)} \in \mathbb{R}^{M_d \times R} \), defined as  

\begin{equation}
\mathbf{V}^{(d)} := \mathcal{R} \left\{ \operatorname{diag}(\mathbf{\Sigma}^{(d)}) \right\}_{M_d \times R},
\end{equation}  

\noindent such that $r$th column of $\mathbf{V}^{(d)}$ corresponds to $\operatorname{Var}(\mathbf{w}^{(d)}_{r})$ in \eqref{eq:30}. By combining equations \eqref{eq:29}, \eqref{eq:30} with $\mathbf{V}^{(d)}$, we can further simplify the computation of $ \mathbf{d}_N = [d_N^1, \dots, d_N^R]^T $ as 

\begin{equation}
    \mathbf{d}_N = \mathbf{d}_0 +  \sum_{d=1}^D \operatorname{diag} \left( \tilde{\mathbf{W}}^{(d)T}  \bm{\Lambda}_{M_d} \tilde{\mathbf{W}}^{(d)} + \bm{\Lambda}_{M_d} \mathbf{V}^{(d)}   \right).
    \label{eq:31}
\end{equation}

The posterior expectation can be obtained by $\mathbb{E}_q[\bm{\lambda}_R] = [c_N^1 / d_N^1, ..., c_N^R/d_N^R]^T$, and thus $\mathbb{E}_q[\bm{\Lambda}_R] = \operatorname{diag}(\mathbb{E}_q[\bm{\lambda}_R])$. An intuitive interpretation of equation \eqref{eq:29} is that $ \lambda_r $ is updated based on the sum of squared $ L_2 $-norms of the $ r $th column, scaled by $ \bm{\Lambda}_{M_d} $ as in equation \eqref{eq:30}. Consequently, a smaller $ ||\mathbf{w}_{r}|| $ or $ \bm{\Lambda}_{M_d} $ leads to larger $\mathbb{E}_q[\lambda_r] $, and updated priors of factor matrices, which in turn more strongly enforces the $r$th column to be zero.

Just as $\bm{\lambda}_R$ operates on the columns of $\mathbf{W}^{(d)}$, $\bm{\lambda}_{M_d}$ acts on its rows . Therefore, learning the posterior of $\bm{\lambda}_{M_d}$ is essential for determining the feature dimension. The key difference is that $\bm{\lambda}_{M_d}$ is specific to each factor matrix, whereas $\bm{\lambda}_R$ is shared across all factor matrices. As shown in Figure 1, the inference of $\bm{\lambda}_{M_d}$ is performed by gathering information from the $d$th factor matrix and incorporating information from its hyperprior. By applying \eqref{eq:19}, we can identify the posteriors of $\lambda_{m_d}$, $\forall d \in [1, D]$ and $\forall m_d \in [1, M_d]$ as an independent Gamma distribution,

\begin{equation}
    q_{\bm{\lambda_{M}}}(\bm{\lambda}_{M_d}) = \prod_{m_d=1}^{M_d} \text{Ga}(\lambda_{m_d} \mid g^{m_d}_N, h^{m_d}_N),
    \label{eq:32}
\end{equation}

\noindent where $g^{m_d}_N, h^{m_d}_N$ denote the posterior parameters learned from $N$ observations and are updated by 

\begin{equation}
    \begin{aligned}
        g^{m_d}_N &=  g^{m_d}_0 + \frac{R}{2},\\
        h^{m_d}_N &=  h^{m_d}_0 + \mathbb{E}_q \left[  \mathbf{w}^{(d)T}_{m_d} \bm{\Lambda}_{R} \mathbf{w}^{(d)}_{m_d}  \right].
    \end{aligned}
    \label{eq:33}
\end{equation}
See Section 6 of the Appendix for a detailed derivation of \eqref{eq:33}. The expectation of the inner product of the $m_d$th row in $d$th factor matrix w.r.t. $q$ distribution can be computed using the posterior parameters $\bm{W}^{(d)}$ in \eqref{eq:22},

\begin{equation}
        \mathbb{E}_q \left[  \mathbf{w}^{(d)T}_{m_d} \bm{\Lambda}_{R} \mathbf{w}^{(d)}_{m_d}  \right] = \tilde{\mathbf{w}}^{(d)T}_{m_d} \bm{\Lambda}_{R} \tilde{\mathbf{w}}^{(d)}_{m_d} + \operatorname{Var} \left( \mathbf{w}^{(d)}_{m_d}\right)^T \bm{\lambda}_R,
    \label{eq:34}
\end{equation}

\noindent where $\tilde{\mathbf{w}}^{(d)}_{m_d}$ denotes the $m_d$th row of $\tilde{\mathbf{W}}^{(d)}$. The second term in the expression accounts for the uncertainty and is the sum of the variances of the each element in $\mathbf{w}^{(d)}_{m_d}$, each weighted by the corresponding $\lambda_{r}$ in $\bm{\lambda}_R$. To efficiently evaluate this variance $\forall m_d \in [1, M_d]$ we can make use of $\mathbf{V}^{(d)}$ as $m_d$th row of it corresponds $\operatorname{Var}( \mathbf{w}^{(d)}_{m_d})$. By combining \eqref{eq:33}, \eqref{eq:34} with $\mathbf{V}^{(d)}$, we can further simplify the computation of $\mathbf{h}^{d}_N = [h_N^{1_d}, ..., h_N^{M_d}]^T$ as 

\begin{equation}
    \mathbf{h}^{d}_N = \mathbf{h}^{d}_0 + \operatorname{diag} \left( \tilde{\mathbf{W}}^{(d)} \bm{\Lambda}_{R} \tilde{\mathbf{W}}^{(d)T} +  \mathbf{V}^{(d)} \bm{\Lambda}_{R}\right).
    \label{eq:35}
\end{equation}

The posterior expectation can be obtained by $\mathbb{E}_q[\bm{\lambda}_{M_d}] = [g_N^{1_d} / h_N^{1_d}, ..., g_N^{M_d} / h_N^{M_d}]^T$, and thus $\mathbb{E}_q[\bm{\Lambda}_{M_d}] = \operatorname{diag}(\mathbb{E}_q[\bm{\lambda}_{M_d}])$. Similar to ${\lambda_r}$, $\lambda_{m_d}$ is updated by the sum of squared $L_2$ norm of the $m_d$th row scaled by $\bm{\Lambda}_{R}$, expressed by \eqref{eq:34} from the $d$th factor matrix. Therefore, intuitively, smaller $||\mathbf{w}^{(d)}_{m_d}||$ or $\bm{\Lambda}_{R}$ leads to larger $\mathbb{E}_q[\lambda_{m_d}]$ and updated priors of factor matrices, which in turn more strongly enforces the $m_d$th row to be zero. 

\vspace{0.5mm}
The updates of $\bm{\lambda}_R$ and $\bm{\lambda}_{M_d}$ are interdependent. $\bm{\lambda}_R$ regulates the importance of rank components across all factor matrices, where a larger $\lambda_r$ implies a less important column. Conversely, $\bm{\lambda}_{M_d}$ determines the relevance of feature dimensions in the $d$th factor matrix, with larger $\lambda_{m_d}$ indicating less important rows. In updating the scale parameter $d_N$ for $\bm{\lambda}_R$ in \eqref{eq:31}, the contributions of each row's mean and variance are weighted by $\lambda_{m_d}$. Similarly, when updating $h_N^d$ for $\bm{\lambda}_{M_d}$ in \eqref{eq:35}, the contributions of each column are scaled by $\lambda_r$.

\vspace{0.5mm}
The equations \eqref{eq:31} and \eqref{eq:35} suggest a negative correlation between $\bm{\lambda}_R$ and $\bm{\lambda}_{M_d}$. To illustrate, consider the precision $\lambda_{m_d}$ of the $m_d$th row in the $d$th factor matrix, which is updated based on the squared $L_2$ norm of that row, weighted by $\bm{\Lambda}_R$ as shown in \eqref{eq:34}. If the row contains large non-zero entries while the corresponding $\lambda_r$ values are also large, these entries are heavily penalized by the large $\lambda_r$ values. Consequently, scaling them by large $\lambda_r$ values 
increases the corresponding scale parameter \texorpdfstring{$h_N^{m_d}$}{hN^{md}}, which in turn reduces the precision $\lambda_{m_d}$. This mechanism enables the model to compensate for the penalization imposed by $\lambda_r$, thus preserving the influence of important values in the row. As a result, the model gains more flexibility by balancing penalization between rank components and feature dimensions, allowing it to better capture significant patterns without overly suppressing relevant features.

\subsubsection{\texorpdfstring{Posterior distribution of noise precision $\tau$}{Posterior distribution of noise precision tau}}

The noise precision $\tau$ can be inferred by receiving information from observed data and its co-parents, including all the factor matrices, and incorporating the information from its hyperprior. Applying \eqref{eq:19}, the variational posterior is a Gamma distribution , given by 

\begin{equation}
    q_{\tau}(\tau) = \text{Ga}(\tau \mid a_N, b_N),
    \label{eq:36}
\end{equation}

\noindent where the posterior parameters are updated by

\begin{equation}
    \begin{aligned}
        a_N &= a_0 + \frac{N}{2}, \\
        b_N &= b_0 + \frac{1}{2} \mathbb{E}_q \left[\| \mathbf{y} - \mathbf{\Phi}^T \mathbf{w}\|_F^{2}\right].
 \end{aligned}
 \label{eq:37}
 \end{equation}   
See Section 7 of the Appendix for a detailed derivation of \eqref{eq:37}. The posterior expectation of model error in \ref{eq:37} cannot be computed straightforwardly and, therefore, we need to introduce the following results. First, following from Theorem \ref{theorem3}, without isolating the $d$th factor matrix, the term $\mathbf{\Phi}^T \mathbf{w}$ in \eqref{eq:20} can also be written as 

\begin{equation}
    \mathbf{\Phi}^T \mathbf{w}= \boldsymbol1_R^T \left( \mathop{\circledast}_{d=1}^D \mathbf{W}^{(d)T}  \mathbf{\Phi}^{(d)}
 \right).
 \label{eq:38}
\end{equation}
\noindent  See Section 1 of the Appendix for a detailed derivation. Using this identity, the posterior expectation of model residual can be expressed as
\begin{equation}
\begin{aligned}
\mathbb{E}_q\left[\| \mathbf{y} - \mathbf{\Phi}^T \mathbf{w}\|_F^{2}\right] = \|\mathbf{y}\|_F^{2} - 2 \, \mathbf{y}^T \, (\boldsymbol1_R^T \left( \mathop{\circledast}_{d=1}^D \mathbf{E}_q \left[\mathbf{W}^{{(d)}} \right]^T \mathbf{\Phi}^{(d)} \right))  + \mathbb{E}_q\left[ \|\boldsymbol1_R^T \left( \mathop{\circledast}_{d=1}^D \mathbf{W}^{(d)}  \mathbf{\Phi}^{(d)}
\right) \|_F^2  \right],\\
\\
\end{aligned}
 \label{eq:39}
\end{equation}
where the last term can be reformulated isolating the random variables as in the following theorem. 

\begin{theorem}
 Given a set of independent random matrices \( \mathbf{W}^{(d)} \) for all \( d \in [1,  D] \), the following linear relation holds:
\begin{equation}
   \mathbb{E}_q\left[ \|\boldsymbol{1}_R^T \left( \mathop{\circledast}_{d=1}^D \mathbf{W}^{(d)}  \mathbf{\Phi}^{(d)} \right) \|_F^2  \right] =  \mathbf{1}_{N}^T \left(  \mathop{\circledast}_{d=1}^D \left( \bm{\Phi}^{(d)} \khatri \bm{\Phi}^{(d)} \right)^T  \mathbb{E}_q \left[ \mathbf{W}^{(d)} \otimes \mathbf{W}^{(d)}\right] \right) \mathbf{1}_{R^2},
 \label{eq:40}
\end{equation}
where $\mathbb{E}_q [ \mathbf{W}^{(d)} \kronecker \mathbf{W}^{(d)} ] = \mathbb{E}_q [ \mathbf{W}^{(d)} ] \kronecker \mathbb{E}_q [ \mathbf{W}^{(d)} ] + \operatorname{Var} [ \mathbf{W}^{(d)} \kronecker \mathbf{W}^{(d)}]$, and $\mathbb{E}_q [ \mathbf{W}^{(d)} ] = \tilde{\mathbf{W}}^{(d)}$.
\label{theorem6}
\end{theorem}

\begin{proof}
See Section~8 of the Appendix.
\end{proof}
Using Lemma 1, we can explicitly evaluate the variance term in terms of $\bm{\Sigma}^{(d)}$ by reshaping it as in \eqref{eq:25}. Therefore, the expectation of the residual sum of squares w.r.t. $q$ distribution can be computed using the posterior parameters in \eqref{eq:22} as in

\begin{equation}
       \begin{aligned}
        \mathbb{E}_q\left[\| \mathbf{y} - \mathbf{\Phi}^T \mathbf{w}\|_F^{2}\right] =  &\|\mathbf{y}\|_F^{2} - 2 \, \mathbf{y}^T \,  (\boldsymbol1_R^T \left( \mathop{\circledast}_{d=1}^D \tilde{\mathbf{W}}^{(d)T} \mathbf{\Phi}^{(d)} \right)) \, \\ 
        &+ \mathbf{1}_{N}^T \left(\mathop{\circledast}_{d= 1}^{D} \left( \bm{\Phi}^{(d)} \khatri \bm{\Phi}^{(d)} \right)^T \left( \tilde{\mathbf{W}}^{(d)} \kronecker \tilde{\mathbf{W}}^{(d)} +  \mathcal{R} \left\{ \mathbf{\Sigma}^{(k)} \right\}_{M_d^2 \times R^2}\right)\right) \mathbf{1}_{R^2}.
    \end{aligned}
        \label{eq:41}
\end{equation}
  Finally, the posterior approximation of $\tau$ can be obtained $\mathbb{E}_q[\tau] = a_N / b_N$.

\subsubsection{Lower Bound Model Evidence}
The inference framework presented in the previous section can essentially maximize the lower bound of the model evidence that is defined in \eqref{eq:17}. Since the lower bound by the definition should not decrease at each iteration, it can be used as a convergence criteria. The lower bound of the log marginal likelihood is computed by 

\begin{equation}
     \mathcal{L}(q) = \mathbb{E}_{q(\Theta)} [\operatorname{ln} \hspace{0.1mm} p(\mathbf{y}, \mathbf{\Theta})] + H(q(\Theta)),
     \label{eq:42}
\end{equation}
\noindent where the first term denotes the posterior expectation of joint distribution, and the second term denotes the entropy of posterior distributions. Various terms in the lower bound are computed and derived by assuming parametric forms for the $q$ distribution, leading to the following results

\small{\begin{equation}
    \begin{aligned}
        \mathcal{L}(q)&= - \frac{a_N}{2b_N}\mathbb{E}_q\left[\| \mathbf{y} - \mathbf{\Phi}^T\mathbf{w}\|_F^{2}\right]  - \frac{1}{2} \operatorname{Tr} \bigg\{  \sum_d \big( \tilde{\bm{\Lambda}}_R \otimes \tilde{\bm{\Lambda}}_{M_d} \big) \bigg( \operatorname{vec}(\tilde{\mathbf{W}}^{(d)}) \operatorname{vec}(\tilde{\mathbf{W}}^{(d)})^T + \mathbf{V}^{(d)} \bigg)  
\bigg\} \\
    &+ \sum_r \left\{ \operatorname{ln} \hspace{1mm} \Gamma(c_N^r) + c_N^r \left( 1 - \operatorname{ln} \hspace{1mm} d_N^r - \frac{d_0^r}{d_N^r} \right) \right\}  + \sum_d \sum_{m_d} \bigg \{ \ln \Gamma(g_N^{m_d}) + g_N^{m_d} \bigg(1 - \ln h_N^{m_d} - \frac{h_0^{m_d}}{h_N^{m_d}} \bigg) \bigg\} \\
    &+ \frac{1}{2}\sum_d \operatorname{ln} \hspace{1mm} \mid \mathbf{V}^{(d)} \mid + \operatorname{ln} \hspace{1mm} \Gamma(a_N) + a_N(1- \operatorname{ln} \hspace{1mm} b_N - \frac{b_0}{b_N}) + \text{const}.
    \end{aligned} 
    \label{eq:43}
\end{equation}}
\normalsize
See Section 10 of the Appendix for a detailed derivation of \eqref{eq:43}. The posterior expectation of model residuals denoted by $\mathbb{E}_q\left[\| \mathbf{y} - \mathbf{\Phi}^T\mathbf{w}\|_F^{2}\right]$ can be computed using \eqref{eq:41}. The lower bound can be interpreted as follows. The first term captures the model residual. The second term represents a weighted sum of the squared $L_2$ norms of the components in the factor matrices, incorporating uncertainty as well. The remaining terms correspond to the negative KL divergence between the posterior and prior distributions of the hyperparameters.

\subsubsection{Implementation and Initialization of Hyperparameters}
\label{implementation_details}

The variational Bayesian inference is guaranteed to converge to a local minimum. To avoid getting stuck in poor local solutions, it is important to choose an initialization point. A commonly used strategy is to adopt \textit{uninformative priors} by setting the hyperparameters \(\mathbf{c}_0, \mathbf{d}_0, \mathbf{g}_0^{M_d}, \) and \( \mathbf{h}_0^{M_d} \) to \( 10^{-6} \). Based on this, the precision matrices for penalization are typically initialized as \( \bm{\Lambda}_{M_d} = \bm{I} \) for all \( d \in [1, D] \) and \( \bm{\Lambda}_R = \bm{I} \).

\vspace{1mm}
\begin{algorithm}[h]
\footnotesize
\caption{Learning BTN Kernel Machines}
\label{alg:fbcp}
\begin{algorithmic}[1]
\Require Inputs $\mathbf{x} = \{x_n\} ^ {N} _ {n=1}$ and outputs $\mathbf{y} = \{y_n\} ^ {N} _ {n=1}$
\State \textbf{Initialization:} $R$, $\{M_d\}_{d=1}^D$, $\mathbf{W}^{(d)}, \mathbf{V}^{(d)}$, $\forall d \in [1, D]$, 
 $a_0$, $b_0$, $\mathbf{c}_0$, $\mathbf{d}_0$, $\mathbf{g}_0^{M_d}$, $\mathbf{h}_0^{M_d}$ and set $\tau = a_0 / b_0$, $\lambda_r = c_0^r / d_0^r$, $\forall r \in [1, R]$, $\lambda_{M_d} = g_0^{m_d}/h_0^{m_d}$, $\forall d \in [1, D]$, $\forall m_d \in [1, M_D]$
\Repeat
    \For{$d = 1$ to $D$}
        \State Update the posterior $q_d(\operatorname{vec}(\mathbf{W}^{(d)}))$ using equation \eqref{eq:22}
    \EndFor
    \For{$d = 1$ to $D$}
        \State Update the posterior $q(\bm{\lambda}_{M_d})$ using equation \eqref{eq:33}
    \EndFor
    \State Update the posterior $q(\bm{\lambda}_{R})$ using equation \eqref{eq:28}
    \State Update the posterior $q(\tau)$ using equation \eqref{eq:37}
    \State Evaluate the lower bound using equation \eqref{eq:43}
    \If{truncation criterion met}
        \State Reduce rank $R$ by eliminating zero-columns of $\mathbf{W}^{(d)} \;\forall d \in [1, D]$
    \EndIf
\Until{convergence}
\State Compute the predictive distribution using \eqref{eq:44}
\end{algorithmic}
\end{algorithm}

For the noise precision, the posterior mean of \( \tau \) is given by \( a_N / b_N \), where \( a_N \) scales with the sample size and \( b_N \) decreases during training as it reflects the expected model error (see~\eqref{eq:37}). This can cause \( \tau \) to grow large, leading to overfitting. To control this, we choose the hyperparameters for $\tau$ slightly higher as $a_0 = b_0 = 10^{-3}$, which still initializes $\tau$ to 1. These values can be adjusted based on the data set. For example, setting \( a_0 \) and \( b_0 \) to give a higher initial value of \( \tau \) increases the influence of the likelihood in the posterior updates, while a smaller \( \tau \) makes the prior more dominant as can be seen in \eqref{eq:21}. Similarly, setting higher initial values for the hyperparameters of \( \bm{\lambda}_R \) or \( \bm{\lambda}_{M_d} \) increases regularization on the columns or rows of the factor matrices, promoting sparsity by pruning less relevant components.

\vspace{1mm}
The factor matrices \( \mathbf{W}^{(d)} \) for all  \( d \in [1, D] \) are initialized from \( \mathcal{N}(0, \bm{I}) \), and the covariance matrix $\Sigma^{(d)}$ is set to $\sigma^2 \bm{I}$ with $\sigma^2 = 10^{-1}$. In the CPD case, due to the Hadamard product structure, initialization $\sigma^2$ is sensitive as it can cause numerical instability when $D$ is large. As can be seen in \eqref{eq:21}, large $\sigma^2$ can lead to inflated posterior updates of the factor matrices, while small $\sigma^2$ can excessively shrink them. In other words, since these matrices are combined via elementwise multiplication, extreme values can result in outputs tending toward infinity or zero as $D$ increases. Finally, the tensor rank \( R \) and the feature dimensions \( M_d ,\, \, \, \forall d\in [1,D] \) can be manually initialized based on the available computational resources, further discussion on their initialization can be found in Section \ref{initial_rank_feature_dim}.

The complete inference procedure is summarized in Algorithm~\ref{alg:fbcp}. We begin by updating the posterior of the factor matrices, followed by the higher-order parameters in order from local to global. For instance, $q(\bm{\lambda}_{M_d})$ depends only on the corresponding factor matrix, while $q(\bm{\lambda}_R)$ gets information from all factor matrices. Since their updates are interdependent, $q(\bm{\lambda}_{M_d})$ is updated first to provide more accurate and stable input to the update of $q(\bm{\lambda}_R)$, which improves convergence and numerical stability.

\vspace{1mm}

To improve efficiency in the implementation, we avoid rebuilding \( \mathbf{G}^{(d)} \) from scratch in every iteration. Instead, we update its components incrementally. For simplicity, we assume all feature dimensions are equal, i.e., \( M_1 = \cdots = M_D = M \). Constructing \( \mathbf{G}^{(d)} = \bm{\Phi}^{(d)} \odot \left( \circledast_{k \neq d} \mathbf{W}^{(k)^T} \bm{\Phi}^{(k)} \right) \) requires \( \mathcal{O}(DNMR) \) operations per iteration, which can be costly. To reduce this cost, we reuse previously computed results. From \eqref{eq:38}, we know that \( \mathbf{G}^{(d)} \) isolates the \( d \)th factor matrix from the full Hadamard product \( \circledast_{d=1}^D \mathbf{W}^{(d)^T} \bm{\Phi}^{(d)} \). We compute this full product once, and for each mode \( d \), we divide it elementwise by \( \mathbf{W}^{(d)^T} \bm{\Phi}^{(d)} \), perform the update, and then multiply the updated version back in.

\vspace{1mm}
Likewise, to avoid recomputing the expected design matrix covariance \( \mathbb{E}_q[\mathbf{G}^{(d)} \mathbf{G}^{(d)T}] \) in \eqref{eq:27}, we calculate the full expression \( \circledast_{d=1}^{D} ( \bm{\Phi}^{(d)} \khatri \bm{\Phi}^{(d)} )^T ( \tilde{\mathbf{W}}^{(d)} \kronecker \tilde{\mathbf{W}}^{(d)} + \mathcal{R} \{ \mathbf{\Sigma}^{(d)} \}_{M_d^2 \times R^2} ) \) only once. Then, for each factor matrix update, we remove the contribution from mode \( d \) by dividing out its term and reinsert it after the update. This reuse strategy reduces redundant computations and improves the overall efficiency.

\vspace{1mm}
 We further speed up the implementation by eliminating the zero columns of \( \{\mathbf{W}^{(d)}\}_{d=1}^D \) after each iteration. The reason for eliminating only the zero columns, instead of both columns and rows, is to avoid reconstructing $\mathbf{G}^{(d)}$ from scratch at every iteration. Since $\mathbf{G}^{(d)}$ is constructed by summing the rows of $\mathbf{W}^{(d)}$ weighted by $\mathbf{\Phi}^{(d)}$, removing rows from the factor matrices means reconstructing $\mathbf{G}^{(d)}$ from scratch at every iteration, which is computationally expensive. Truncation criteria can be set manually. In our implementation, we retain all \( R \) components for the first three iterations to allow the model enough flexibility before removing potentially useful components. After that, we remove components that contribute less than \( 10^{-5} \) to the total variance. To identify low-variance components, we stack the factor matrices column-wise to form \( \tilde{\mathbf{W}} = [\tilde{\mathbf{W}}^{(1)}, \tilde{\mathbf{W}}^{(2)}, \ldots, \tilde{\mathbf{W}}^{(D)}] \), compute \( \tilde{\mathbf{W}}^\top \tilde{\mathbf{W}} \), and remove components whose diagonal values fall below the \( 10^{-5} \) variance threshold.

\subsection{Predictive Distribution}
The predictive distribution over unseen data points, given training data, can be approximated by using variational posterior distribution, that is, 

\begin{equation}
    \begin{aligned}
        p(\tilde{y_i} \mid \mathbf{y}) &= \int p\left(y_i \mid \Theta\right) p\left(\Theta \mid \mathbf{y} \right)d\Theta \\
    \vspace{2cm}
     &\simeq \int \int p \left(\tilde{y_i} \mid \left\{ \mathbf{W}^{(d)} \right\}, \tau^{-1} \right) q\left( \left\{ \mathbf{W}^{(d)} \right\} \right) q(\tau) d\left\{ \mathbf{W}^{(d)} \right\} d \tau
    \end{aligned}
    \label{eq:44}
\end{equation}

\noindent Approximation of these integrations yields a Student's t-distribution  $\tilde{y_i} \mid \mathbf{y} \sim \mathcal{T}(\tilde{y_i}, \mathcal{S}_i, \nu_y)$  with its parameters given by 

\begin{align*}
    \tilde{y_i} &= \mathop{\circledast}_{d=1}^D \tilde{\mathbf{W}}^{(d)}  \mathbf{\varphi}_i^{(d)}, \qquad \nu_y = 2a_N, \\
    \mathcal{S}_i &= \left\{ \frac{b_N}{a_N}\sum_d  \textbf{g}^{(d)} (x_n)^T \mathbf{\Sigma}^{(d)} \textbf{g}^{(d)} (x_n)\right\}^{-1}. 
 \end{align*}

\noindent See Section 11 of the Appendix for a detailed derivation. Thus, the predictive variance can be obtained by $\operatorname{Var}(y_i) = \frac{\nu_y}{\nu_y - 2} \mathcal{S}_i ^{-1}$. 

\subsection{Computational Complexity}

The total computational cost of computing the posterior parameters for all factor matrices \(\mathbf{W}^{(d)}\) in  \eqref{eq:22} is \( \mathcal{O}\left(\sum_d N M_d^2 R^2 + M_d^3 R^3\right)\), where \(N\) is the number of observations, \(M_d\) is the feature dimension of the \(d\)th feature, and \(R\) is the tensor rank. The overall cost grows linearly with the number of observations \(N\) and the input dimension \(D\), and polynomially with the model complexity parameters \(M_d\) and \(R\). The cost of computing the model complexity hyperparameters $\bm{\lambda}_{R}$ and $\bm{\lambda}_{M_d}$ is $\mathcal{O}(\sum_d M_d R^2)$ and $\mathcal{O}(\sum_dM_d^2R)$ respectively. Finally the computational cost of computing the noise precision $\tau$ is only $\mathcal{O}(NR)$. \(M_d\) and \(R\) are typically much smaller than \(N\) and the rank \(R\) is automatically inferred during training and zero components are pruned early, its value tends to decrease rapidly in the first few iterations. As a result  when $N >> M_dR$ the total computational complexity of Algorithm 1 is dominated by the factor matrix updates which is \( \mathcal{O}\left(\sum_d N M_d^2 R^2 + M_d^3 R^3\right)\), making it suitable for learning problems which are large in both $N$ and $D$.

\subsection{Tensor Train Kernel Machines}

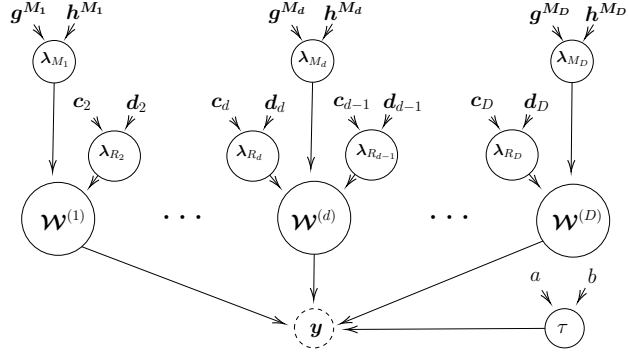
\begin{figure}[t]
    \centering
    \resizebox{0.55\textwidth}{!}{ 
        \begin{tikzpicture}[x=0.75pt,y=0.75pt,yscale=-1,xscale=1]

\draw   (64.67,240.57) .. controls (64.67,222.07) and (79.67,207.07) .. (98.17,207.07) .. controls (116.67,207.07) and (131.67,222.07) .. (131.67,240.57) .. controls (131.67,259.08) and (116.67,274.07) .. (98.17,274.07) .. controls (79.67,274.07) and (64.67,259.08) .. (64.67,240.57) -- cycle ;
\draw   (299.67,239.57) .. controls (299.67,221.07) and (314.67,206.07) .. (333.17,206.07) .. controls (351.67,206.07) and (366.67,221.07) .. (366.67,239.57) .. controls (366.67,258.08) and (351.67,273.07) .. (333.17,273.07) .. controls (314.67,273.07) and (299.67,258.08) .. (299.67,239.57) -- cycle ;
\draw   (537.33,242.24) .. controls (537.33,223.74) and (552.33,208.74) .. (570.83,208.74) .. controls (589.33,208.74) and (604.33,223.74) .. (604.33,242.24) .. controls (604.33,260.74) and (589.33,275.74) .. (570.83,275.74) .. controls (552.33,275.74) and (537.33,260.74) .. (537.33,242.24) -- cycle ;
\draw [dashed]  (314.99,341.18) .. controls (314.94,331.23) and (322.96,323.13) .. (332.9,323.09) .. controls (342.84,323.04) and (350.94,331.06) .. (350.99,341) .. controls (351.04,350.94) and (343.02,359.04) .. (333.08,359.09) .. controls (323.14,359.14) and (315.04,351.12) .. (314.99,341.18) -- cycle ;
\draw   (312.17,95.57) .. controls (312.17,84.8) and (320.9,76.07) .. (331.67,76.07) .. controls (342.44,76.07) and (351.17,84.8) .. (351.17,95.57) .. controls (351.17,106.34) and (342.44,115.07) .. (331.67,115.07) .. controls (320.9,115.07) and (312.17,106.34) .. (312.17,95.57) -- cycle ;
\draw   (75,96.24) .. controls (75,85.47) and (83.73,76.74) .. (94.5,76.74) .. controls (105.27,76.74) and (114,85.47) .. (114,96.24) .. controls (114,107.01) and (105.27,115.74) .. (94.5,115.74) .. controls (83.73,115.74) and (75,107.01) .. (75,96.24) -- cycle ;
\draw   (550.33,95.91) .. controls (550.33,85.14) and (559.06,76.41) .. (569.83,76.41) .. controls (580.6,76.41) and (589.33,85.14) .. (589.33,95.91) .. controls (589.33,106.68) and (580.6,115.41) .. (569.83,115.41) .. controls (559.06,115.41) and (550.33,106.68) .. (550.33,95.91) -- cycle ;
\draw   (544.99,341.18) .. controls (544.94,331.23) and (552.96,323.13) .. (562.9,323.09) .. controls (572.84,323.04) and (580.94,331.06) .. (580.99,341) .. controls (581.04,350.94) and (573.02,359.04) .. (563.08,359.09) .. controls (553.14,359.14) and (545.04,351.12) .. (544.99,341.18) -- cycle ;
\draw    (119.67,266.07) -- (304.13,336.29) ;
\draw [shift={(306,337)}, rotate = 200.84] [color={rgb, 255:red, 0; green, 0; blue, 0 }  ][line width=0.75]    (10.93,-3.29) .. controls (6.95,-1.4) and (3.31,-0.3) .. (0,0) .. controls (3.31,0.3) and (6.95,1.4) .. (10.93,3.29)   ;
\draw    (543,261) -- (358.85,336.24) ;
\draw [shift={(357,337)}, rotate = 337.77] [color={rgb, 255:red, 0; green, 0; blue, 0 }  ][line width=0.75]    (10.93,-3.29) .. controls (6.95,-1.4) and (3.31,-0.3) .. (0,0) .. controls (3.31,0.3) and (6.95,1.4) .. (10.93,3.29)   ;
\draw    (544.99,341.18) -- (363.99,342.16) ;
\draw [shift={(361.99,342.18)}, rotate = 359.69] [color={rgb, 255:red, 0; green, 0; blue, 0 }  ][line width=0.75]    (10.93,-3.29) .. controls (6.95,-1.4) and (3.31,-0.3) .. (0,0) .. controls (3.31,0.3) and (6.95,1.4) .. (10.93,3.29)   ;
\draw    (333.17,273.07) -- (332.69,313.07) ;
\draw [shift={(332.67,315.07)}, rotate = 270.68] [color={rgb, 255:red, 0; green, 0; blue, 0 }  ][line width=0.75]    (10.93,-3.29) .. controls (6.95,-1.4) and (3.31,-0.3) .. (0,0) .. controls (3.31,0.3) and (6.95,1.4) .. (10.93,3.29)   ;
\draw    (312.5,62.41) -- (319.39,72.74) ;
\draw [shift={(320.5,74.41)}, rotate = 236.31] [color={rgb, 255:red, 0; green, 0; blue, 0 }  ][line width=0.75]    (10.93,-3.29) .. controls (6.95,-1.4) and (3.31,-0.3) .. (0,0) .. controls (3.31,0.3) and (6.95,1.4) .. (10.93,3.29)   ;
\draw    (349.83,61.41) -- (342.88,72.7) ;
\draw [shift={(341.83,74.41)}, rotate = 301.61] [color={rgb, 255:red, 0; green, 0; blue, 0 }  ][line width=0.75]    (10.93,-3.29) .. controls (6.95,-1.4) and (3.31,-0.3) .. (0,0) .. controls (3.31,0.3) and (6.95,1.4) .. (10.93,3.29)   ;
\draw    (331.67,115.07) -- (330.04,192.57) ;
\draw [shift={(330,194.57)}, rotate = 271.2] [color={rgb, 255:red, 0; green, 0; blue, 0 }  ][line width=0.75]    (10.93,-3.29) .. controls (6.95,-1.4) and (3.31,-0.3) .. (0,0) .. controls (3.31,0.3) and (6.95,1.4) .. (10.93,3.29)   ;
\draw    (542,306.07) -- (548.89,316.41) ;
\draw [shift={(550,318.07)}, rotate = 236.31] [color={rgb, 255:red, 0; green, 0; blue, 0 }  ][line width=0.75]    (10.93,-3.29) .. controls (6.95,-1.4) and (3.31,-0.3) .. (0,0) .. controls (3.31,0.3) and (6.95,1.4) .. (10.93,3.29)   ;
\draw    (582,304.07) -- (575.95,315.31) ;
\draw [shift={(575,317.07)}, rotate = 298.3] [color={rgb, 255:red, 0; green, 0; blue, 0 }  ][line width=0.75]    (10.93,-3.29) .. controls (6.95,-1.4) and (3.31,-0.3) .. (0,0) .. controls (3.31,0.3) and (6.95,1.4) .. (10.93,3.29)   ;
\draw    (94.5,115.74) -- (93.04,195.57) ;
\draw [shift={(93,197.57)}, rotate = 271.05] [color={rgb, 255:red, 0; green, 0; blue, 0 }  ][line width=0.75]    (10.93,-3.29) .. controls (6.95,-1.4) and (3.31,-0.3) .. (0,0) .. controls (3.31,0.3) and (6.95,1.4) .. (10.93,3.29)   ;
\draw    (569.83,115.41) -- (569.02,195.57) ;
\draw [shift={(569,197.57)}, rotate = 270.58] [color={rgb, 255:red, 0; green, 0; blue, 0 }  ][line width=0.75]    (10.93,-3.29) .. controls (6.95,-1.4) and (3.31,-0.3) .. (0,0) .. controls (3.31,0.3) and (6.95,1.4) .. (10.93,3.29)   ;
\draw    (257.64,147.73) -- (263.07,155.81) ;
\draw [shift={(264.18,157.47)}, rotate = 236.08] [color={rgb, 255:red, 0; green, 0; blue, 0 }  ][line width=0.75]    (10.93,-3.29) .. controls (6.95,-1.4) and (3.31,-0.3) .. (0,0) .. controls (3.31,0.3) and (6.95,1.4) .. (10.93,3.29)   ;
\draw    (293.36,146.11) -- (288.59,154.9) ;
\draw [shift={(287.64,156.66)}, rotate = 298.5] [color={rgb, 255:red, 0; green, 0; blue, 0 }  ][line width=0.75]    (10.93,-3.29) .. controls (6.95,-1.4) and (3.31,-0.3) .. (0,0) .. controls (3.31,0.3) and (6.95,1.4) .. (10.93,3.29)   ;
\draw   (253,182.91) .. controls (253,170.16) and (263.67,159.82) .. (276.83,159.82) .. controls (290,159.82) and (300.67,170.16) .. (300.67,182.91) .. controls (300.67,195.66) and (290,206) .. (276.83,206) .. controls (263.67,206) and (253,195.66) .. (253,182.91) -- cycle ;
\draw    (370.67,199.74) -- (363.84,209.12) ;
\draw [shift={(362.67,210.74)}, rotate = 306.03] [color={rgb, 255:red, 0; green, 0; blue, 0 }  ][line width=0.75]    (10.93,-3.29) .. controls (6.95,-1.4) and (3.31,-0.3) .. (0,0) .. controls (3.31,0.3) and (6.95,1.4) .. (10.93,3.29)   ;
\draw    (292.33,200.41) -- (300.4,210.2) ;
\draw [shift={(301.67,211.74)}, rotate = 230.53] [color={rgb, 255:red, 0; green, 0; blue, 0 }  ][line width=0.75]    (10.93,-3.29) .. controls (6.95,-1.4) and (3.31,-0.3) .. (0,0) .. controls (3.31,0.3) and (6.95,1.4) .. (10.93,3.29)   ;
\draw    (75.83,62.07) -- (82.72,72.41) ;
\draw [shift={(83.83,74.07)}, rotate = 236.31] [color={rgb, 255:red, 0; green, 0; blue, 0 }  ][line width=0.75]    (10.93,-3.29) .. controls (6.95,-1.4) and (3.31,-0.3) .. (0,0) .. controls (3.31,0.3) and (6.95,1.4) .. (10.93,3.29)   ;
\draw    (113.17,61.07) -- (106.21,72.37) ;
\draw [shift={(105.17,74.07)}, rotate = 301.61] [color={rgb, 255:red, 0; green, 0; blue, 0 }  ][line width=0.75]    (10.93,-3.29) .. controls (6.95,-1.4) and (3.31,-0.3) .. (0,0) .. controls (3.31,0.3) and (6.95,1.4) .. (10.93,3.29)   ;
\draw    (552.17,63.07) -- (559.06,73.41) ;
\draw [shift={(560.17,75.07)}, rotate = 236.31] [color={rgb, 255:red, 0; green, 0; blue, 0 }  ][line width=0.75]    (10.93,-3.29) .. controls (6.95,-1.4) and (3.31,-0.3) .. (0,0) .. controls (3.31,0.3) and (6.95,1.4) .. (10.93,3.29)   ;
\draw    (589.5,62.07) -- (582.55,73.37) ;
\draw [shift={(581.5,75.07)}, rotate = 301.61] [color={rgb, 255:red, 0; green, 0; blue, 0 }  ][line width=0.75]    (10.93,-3.29) .. controls (6.95,-1.4) and (3.31,-0.3) .. (0,0) .. controls (3.31,0.3) and (6.95,1.4) .. (10.93,3.29)   ;
\draw    (366.64,146.73) -- (372.07,154.81) ;
\draw [shift={(373.18,156.47)}, rotate = 236.08] [color={rgb, 255:red, 0; green, 0; blue, 0 }  ][line width=0.75]    (10.93,-3.29) .. controls (6.95,-1.4) and (3.31,-0.3) .. (0,0) .. controls (3.31,0.3) and (6.95,1.4) .. (10.93,3.29)   ;
\draw    (402.36,145.11) -- (397.59,153.9) ;
\draw [shift={(396.64,155.66)}, rotate = 298.5] [color={rgb, 255:red, 0; green, 0; blue, 0 }  ][line width=0.75]    (10.93,-3.29) .. controls (6.95,-1.4) and (3.31,-0.3) .. (0,0) .. controls (3.31,0.3) and (6.95,1.4) .. (10.93,3.29)   ;
\draw   (362,181.91) .. controls (362,169.16) and (372.67,158.82) .. (385.83,158.82) .. controls (399,158.82) and (409.67,169.16) .. (409.67,181.91) .. controls (409.67,194.66) and (399,205) .. (385.83,205) .. controls (372.67,205) and (362,194.66) .. (362,181.91) -- cycle ;
\draw    (134.67,200.74) -- (127.84,210.12) ;
\draw [shift={(126.67,211.74)}, rotate = 306.03] [color={rgb, 255:red, 0; green, 0; blue, 0 }  ][line width=0.75]    (10.93,-3.29) .. controls (6.95,-1.4) and (3.31,-0.3) .. (0,0) .. controls (3.31,0.3) and (6.95,1.4) .. (10.93,3.29)   ;
\draw    (130.64,147.73) -- (136.07,155.81) ;
\draw [shift={(137.18,157.47)}, rotate = 236.08] [color={rgb, 255:red, 0; green, 0; blue, 0 }  ][line width=0.75]    (10.93,-3.29) .. controls (6.95,-1.4) and (3.31,-0.3) .. (0,0) .. controls (3.31,0.3) and (6.95,1.4) .. (10.93,3.29)   ;
\draw    (166.36,146.11) -- (161.59,154.9) ;
\draw [shift={(160.64,156.66)}, rotate = 298.5] [color={rgb, 255:red, 0; green, 0; blue, 0 }  ][line width=0.75]    (10.93,-3.29) .. controls (6.95,-1.4) and (3.31,-0.3) .. (0,0) .. controls (3.31,0.3) and (6.95,1.4) .. (10.93,3.29)   ;
\draw   (126,182.91) .. controls (126,170.16) and (136.67,159.82) .. (149.83,159.82) .. controls (163,159.82) and (173.67,170.16) .. (173.67,182.91) .. controls (173.67,195.66) and (163,206) .. (149.83,206) .. controls (136.67,206) and (126,195.66) .. (126,182.91) -- cycle ;
\draw    (495.64,146.73) -- (501.07,154.81) ;
\draw [shift={(502.18,156.47)}, rotate = 236.08] [color={rgb, 255:red, 0; green, 0; blue, 0 }  ][line width=0.75]    (10.93,-3.29) .. controls (6.95,-1.4) and (3.31,-0.3) .. (0,0) .. controls (3.31,0.3) and (6.95,1.4) .. (10.93,3.29)   ;
\draw    (531.36,145.11) -- (526.59,153.9) ;
\draw [shift={(525.64,155.66)}, rotate = 298.5] [color={rgb, 255:red, 0; green, 0; blue, 0 }  ][line width=0.75]    (10.93,-3.29) .. controls (6.95,-1.4) and (3.31,-0.3) .. (0,0) .. controls (3.31,0.3) and (6.95,1.4) .. (10.93,3.29)   ;
\draw   (491,181.91) .. controls (491,169.16) and (501.67,158.82) .. (514.83,158.82) .. controls (528,158.82) and (538.67,169.16) .. (538.67,181.91) .. controls (538.67,194.66) and (528,205) .. (514.83,205) .. controls (501.67,205) and (491,194.66) .. (491,181.91) -- cycle ;
\draw    (530.33,199.41) -- (538.4,209.2) ;
\draw [shift={(539.67,210.74)}, rotate = 230.53] [color={rgb, 255:red, 0; green, 0; blue, 0 }  ][line width=0.75]    (10.93,-3.29) .. controls (6.95,-1.4) and (3.31,-0.3) .. (0,0) .. controls (3.31,0.3) and (6.95,1.4) .. (10.93,3.29)   ;

\draw (290.17,40.07) node [anchor=north west][inner sep=0.75pt]   [align=left] {\Large $\bm{g^{M_d}}$};
\draw (340.83,40.41) node [anchor=north west][inner sep=0.75pt]   [align=left] {\Large $\bm{h^{M_d}}$};
\draw (80.67,85.81) node [anchor=north west][inner sep=0.75pt]     {$\bm{\lambda}_{M_{1}}$};
\draw (317.17,85.47) node [anchor=north west][inner sep=0.75pt]     {$\bm{\lambda}_{M_{d}}$};
\draw (554.33,84.81) node [anchor=north west][inner sep=0.75pt]     {$\bm{\lambda}_{M_{D}}$};
\draw (555,335.47) node [anchor=north west][inner sep=0.75pt]    {\Large$\tau$};
\draw (328,335.47) node [anchor=north west][inner sep=0.75pt]    {\Large$\bm{y}$};
\draw (310.67,230.48) node [anchor=north west][inner sep=0.75pt]  [rotate=-359.99] {\Large$\bm{\mathcal{W}}^{(d)}$};
\draw (80.67,230.48) node [anchor=north west][inner sep=0.75pt]  [rotate=-359.99]  {\Large$\bm{\mathcal{W}}^{(1)}$};
\draw (550.33,230.48) node [anchor=north west][inner sep=0.75pt]  [rotate=-359.99]  {\Large$\bm{\mathcal{W}}^{(D)}$};
\draw (530,288.07) node [anchor=north west][inner sep=0.75pt]   [align=left] {\Large$a$};
\draw (583,285.07) node [anchor=north west][inner sep=0.75pt]   [align=left] {\Large$b$};
\draw (237.32,128.09) node [anchor=north west][inner sep=0.75pt]   [align=left] {\Large $\bm{c}_{d}$};
\draw (286.18,125.09) node [anchor=north west][inner sep=0.75pt]   [align=left] {\Large $\bm{d}_{d}$};
\draw (260.27,171.34) node [anchor=north west][inner sep=0.75pt]    {$\bm{\lambda} _{R_{d}}$};
\draw (53.5,39.74) node [anchor=north west][inner sep=0.75pt]   [align=left] {\Large $\bm{g^{M_1}}$};
\draw (104.17,40.07) node [anchor=north west][inner sep=0.75pt]   [align=left] {\Large $\bm{h^{M_1}}$};
\draw (529.83,40.74) node [anchor=north west][inner sep=0.75pt]   [align=left] {\Large $\bm{g^{M_D}}$};
\draw (580.5,41.07) node [anchor=north west][inner sep=0.75pt]   [align=left] {\Large $\bm{h^{M_D}}$};
\draw (349.32,127.09) node [anchor=north west][inner sep=0.75pt]   [align=left] {\Large $\bm{c}_{d-1}$};
\draw (395.18,124.09) node [anchor=north west][inner sep=0.75pt]   [align=left] {\Large $\bm{d}_{d-1}$};
\draw (370.27,170.34) node [anchor=north west][inner sep=0.75pt]    {$\bm{\lambda} _{R_{d-1}}$};
\draw (110.32,128.09) node [anchor=north west][inner sep=0.75pt]   [align=left] {\Large $\bm{c}_2$};
\draw (159.18,125.09) node [anchor=north west][inner sep=0.75pt]   [align=left] {\Large $\bm{d}_2$};
\draw (134.27,171.34) node [anchor=north west][inner sep=0.75pt]    {$\bm{\lambda} _{R_{2}}$};
\draw (475.32,127.09) node [anchor=north west][inner sep=0.75pt]   [align=left] {\Large $\bm{c}_{D}$};
\draw (524.18,124.09) node [anchor=north west][inner sep=0.75pt]   [align=left] {\Large $\bm{d}_{D}$};
\draw (496.27,170.34) node [anchor=north west][inner sep=0.75pt]    {$\bm{\lambda}_{R_{D}}$};

\node at (215,240) {\Huge $\cdots$};  
\node at (460,240) {\Huge $\cdots$};  

        \end{tikzpicture}
    }
    \caption{\footnotesize Representation of BTN-Kernel machines with the TT-decomposed weight vector $\mathbf{w}$ as a probabilistic graphical model showing the hierarchical sparsity inducing priors over the core tensors $\{\bm{\mathcal{W}}^{(d)}\}_{d=1}^D$ by the sparsity parameters $\{\boldsymbol{\lambda}_{R}\}_{d=2}^{D}$ and $\{\boldsymbol{\lambda}_{M_d}\}_{d=1}^D$. The dashed node denotes the observed data $\mathbf{y}$, while the solid nodes represent random variables. Shape and scale hyperparameters of the Gamma priors placed on $\{\boldsymbol{\lambda}_{R}\}_{d=2}^{D}$, $\{\boldsymbol{\lambda}_{M_d}\}_{d=1}^D$ and $\tau$ are shown as unbounded nodes.}
    \label{fig:TT_model}
\end{figure}

\label{tt-discussion}
Tensor Trains are another common choice to parameterize $\mathbf{w}$. These models replace the factor matrices $\bm{W}^{(d)}$ of the CPD with third-order core tensors $\bm{\mathcal{W}}^{(d)} \in \mathbb{R}^{R_d \times M_d \times R_{d+1}}$. Extending the probabilistic model to Tensor Train Kernel Machines is straightforward. The prior of the Tensor Train core tensors is specified as
\begin{equation*}
    p(\operatorname{vec}(\bm{\mathcal{W}}^{(d)})\mid \bm{\lambda}_{R_d}, \bm{\lambda}_{M_d}, \bm{\lambda}_{R_{d+1}}) = \mathcal{N} \left(\operatorname{vec}(\bm{\mathcal{W}}^{(d)}) \mid \bm{0},  \bm{\Lambda}_{R_{d+1}}^{-1} \otimes \bm{\Lambda}_{M_d}^{-1} \otimes \bm{\Lambda}_{R_d}^{-1} \right), \quad \forall d \in [1, D].
\end{equation*}
Since a Tensor Train has $D-1$ ranks $R_2,\ldots,R_D$ there will be $D-1$ corresponding prior precision matrices $\bm{\Lambda}_{R_{2}},\ldots,\bm{\Lambda}_{R_{D}}$. The increase in number of ranks leads to models that are more flexible but hence come at the cost of more parameters to optimize. Just as with the CPD-based model, the computational complexity of updating the posterior is dominated by the update of $q_d(\operatorname{vec}(\bm{\mathcal{W}}^{(d)}))$, which is $\mathcal{O}(N(R_dM_dR_{d+1})^2+(R_dM_dR_{d+1})^3)$ flops.  

\vspace{1mm}
Figure~\ref{fig:TT_model} shows the probabilistic graphical model representing the joint distribution for the case where $\mathbf{w}$ is decomposed using the Tensor Train (TT) format. For simplicity, the deterministic design matrices $\{ \bm{\Phi}^{(d)} \}_{d=1}^D$ are not shown, similar to the CPD case in Figure~\ref{fig:CPD_model}. As illustrated in Figure~\ref{fig:TT_model}, the main difference in the model's parameterization, compared to the CPD decomposed $\mathbf{w}$ model is that the TT model includes $D{-}1$ independent sparsity parameters $\bm{\lambda}_{R_d}$. These parameters are used to infer the tensor rank of each core in the TT decomposition.

\section{Experiments}

We first conducted an experiment using synthetic data to illustrate the effect of penalization through the sparsity parameters $\bm{\lambda}_{M_d}$ and $\bm{\lambda}_R$ on the rows and columns of the factor matrices. Next, using real data, we analyzed the model’s convergence as well as the impact of the initial rank and feature dimension on its performance. Finally, we compared BTN-kernel machines with three state-of-the-art methods: T-KRR \citep{wesel_large-scale_2021}, SP-BTN \citep{konstantinidis2022vbttn}, and GP \citep{rasmussen:06}, evaluating their predictive performance. In all experiments, we applied a polynomial feature mapping followed by normalization to unit norm. To improve numerical stability for large $D$, we added a constant offset of 0.2 to each feature matrix $\bm{\Phi}^{(d)}$ after construction. This shift mitigates the risk of vanishing values in the final output in \eqref{eq:38}, which involves the Hadamard product applied $D$ times. Our Python code is available at \url{https://github.com/afrakilic/BTN-Kernel-Machines} and allows the reproduction of all experiments in this section.

\subsection{Ground Truth Recovery with Synthetic Data}
\label{sec:experiments}

To illustrate the effect of penalization via $\bm{\lambda}_{M_d}$ and $\bm{\lambda}_R$ on the rows and columns of the factor matrices, we first present a synthetic experiment. The ground truth is constructed using the following procedure. Each column of the feature matrix $\mathbf{X} \in \mathbb{R}^{N \times D}$ is independently sampled from a standard normal distribution, $\mathcal{N}(0, \bm{I})$, with $N = 500$ and $D = 3$. Then, $D$ factor matrices $\mathbf{W}^{(d)} \in \mathbb{R}^{M_d \times R}$ are drawn from $\mathcal{N}(0, \bm{I})$ and scaled by 10 to ensure non-zero values, with the true rank set to $R = 3$. For each feature dimension $M_d$, a random subset of feature indices (up to 5) is selected, resulting in $M_1 = 1$, $M_2 = 4$, and $M_3 = 3$. The true observed data is then constructed as $\mathbf{y} = \mathbf{\Phi}^T \mathbf{w} + \mathbf{e}$, where $\epsilon \sim \mathcal{N}(0, \sigma_e^2)$ with $\sigma_e = 0.001$ denotes i.i.d. additive noise. 

\vspace{1mm}
To test whether $\bm{\lambda}_{M_d}$ and $\bm{\lambda}_R$ penalize noisy rows and columns respectively, the model is initialized as described in Section 3.2.5, with an initial rank of 5 and feature dimension of 5 for all modes.  The posterior means of the factor matrices, $\tilde{\mathbf{W}}^{(d)}$ are illustrated in Figure \ref{fig:plot1} along with  $\tilde{\bm{\lambda}}_{R} $ and $\tilde{\bm{\lambda}}_{M_d} $. Three factor matrices are inferred in which two columns become zero, resulting in the correct estimation of the tensor rank as 3. Furthermore, in order to asses the matrix ranks of the estimated factor matrices, singular value analysis was performed. For $\tilde{\mathbf{W}}^{(1)}$, whose true rank is 1, the second singular value drops by over five orders of magnitude from 447.34 to $1.60 \times 10^{-3}$, indicating rank 1. For $\tilde{\mathbf{W}}^{(2)}$ and $\tilde{\mathbf{W}}^{(3)}$, both with true rank 3, the fourth singular values are near zero ($1.40 \times 10^{-16}$ and $1.64 \times 10^{-17}$), showing drops of more than 15 orders of magnitude from the third singular values, indicating rank 3.

\begin{figure}[t]
  \centering
  \includegraphics[width=1\linewidth]{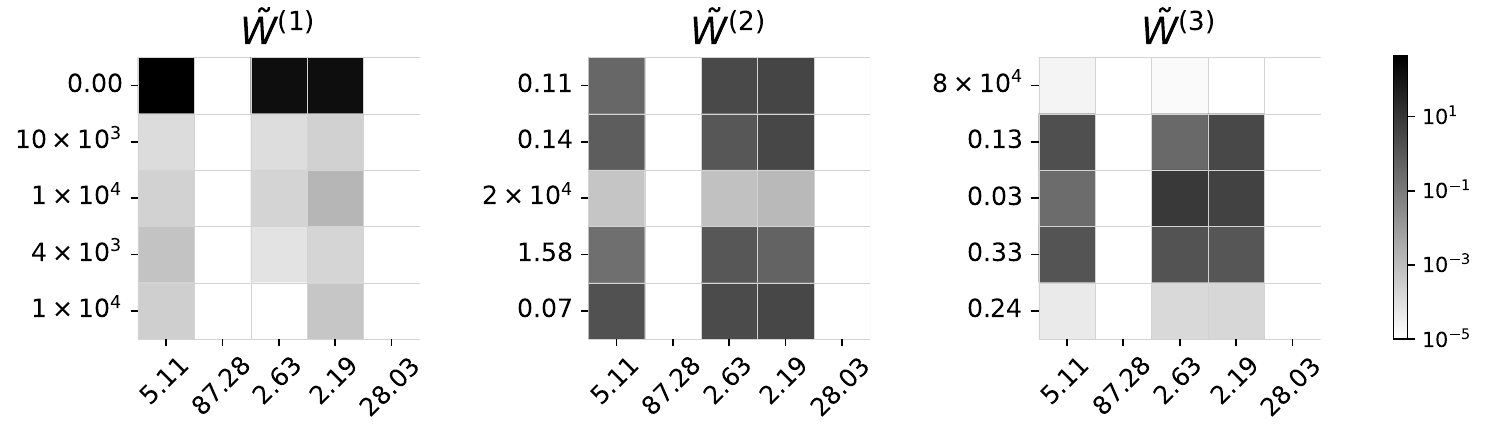}
\caption{\footnotesize Posterior means of the factor matrices $\tilde{\mathbf{W}}^{(d)}$ for all $d \in [1, 3]$, along with $\tilde{\bm{\lambda}}_{R}$ and $\tilde{\bm{\lambda}}_{M_d}$. The row labels to the left of each matrix show the posterior mean of $\tilde{\bm{\lambda}}_{M_d}$, specific to each factor matrix. The column labels below the matrices represent the posterior mean of $\tilde{\bm{\lambda}}_{R}$, shared across all factor matrices. The model correctly identifies the tensor rank and feature dimensions by assigning larger precision values to unnecessary rows and columns, which makes their values shrink zero or close to zero.}

  \label{fig:plot1}
\end{figure}

Figure \ref{fig:plot1} shows that the posterior mean of $\tilde{\lambda}_r$ is smaller for columns that remain nonzero after estimation and larger for those that shrink to zero. In particular, it is at least 5.5 times larger for zero columns compared to nonzero ones. Likewise, for the correct feature dimensions of each factor matrix, rows with large corresponding $\tilde{\lambda}_{m_d}$ values are strongly penalized, causing their entries to become close to zero. To illustrate this, consider the first factor matrix $\mathbf{W}^{(1)}$, where the true feature dimension is $M_1 = 1$. In the posterior mean $\mathbf{\tilde{W}}^{(1)}$, the smallest nonzero entry in the first row is approximately five orders of magnitude larger than the largest entry in the other rows, consistent with the corresponding precision value  $\lambda_{m_1 = 1} \approx 0$.  

\vspace{1mm}
An interesting observation is that the range of  $\lambda_M$ values is much larger than that of $\lambda_R$, spanning from approximately $5.5 \times 10^{-6}$ to $7.9 \times 10^4$, which covers over ten orders of magnitude. In contrast, $\lambda_R$ ranges only from 2.2 to 87.3, just under two orders of magnitude. However, even when $\lambda_{m_d}$ takes very large values, in most cases the corresponding rows in the factor matrices do not become exactly zero, unlike the columns affected by large $\lambda_r$. This difference is due to the structure of CPD, which involves an elementwise (Hadamard) product across the columns of the factor matrices. This structure appears in the design matrix $\textbf{G}^{(d)}= \mathbf{\Phi} ^{(d)} \khatri  \left(\mathop{\circledast}_{k \neq d}  \mathbf{W}^{{(k)}^T} \mathbf{\Phi}^{(k)}\right)$ which is used in the posterior mean update of the factor matrices (\ref{eq:21}).  When any column is strongly penalized by a large $\lambda_r$, its values become close to zero.  Since the CPD model multiplies these columns elementwise, the entire product for that component also approaches zero, causing those columns themselves to become zero. In contrast, rows contribute through summation as in $\mathbf{W}^{(k)T} \mathbf{\Phi}^{(k)}$, rather than through multiplication.. As a result, even when strongly penalized, their values usually become very small rather than exactly zero. Moreover, as shown in Figure \ref{fig:plot1},   $\bm{\lambda}_{R}$ is shared across all factor matrices and penalizes the same columns in every mode. In contrast, $\bm{\lambda}_{M_d}$ is independent across modes and penalizes the rows of its corresponding factor matrix, resulting in different rows being penalized in each mode. Finally, the posterior mean of the noise precision was estimated as $\tau \approx 204000$, indicates the method's effectiveness in denoising, with the estimated noise variance $\tilde{\sigma}_e \approx 0.002$.

\subsection{Model Convergence}

\begin{figure}[t]
  \centering
  \includegraphics[width=1\linewidth]{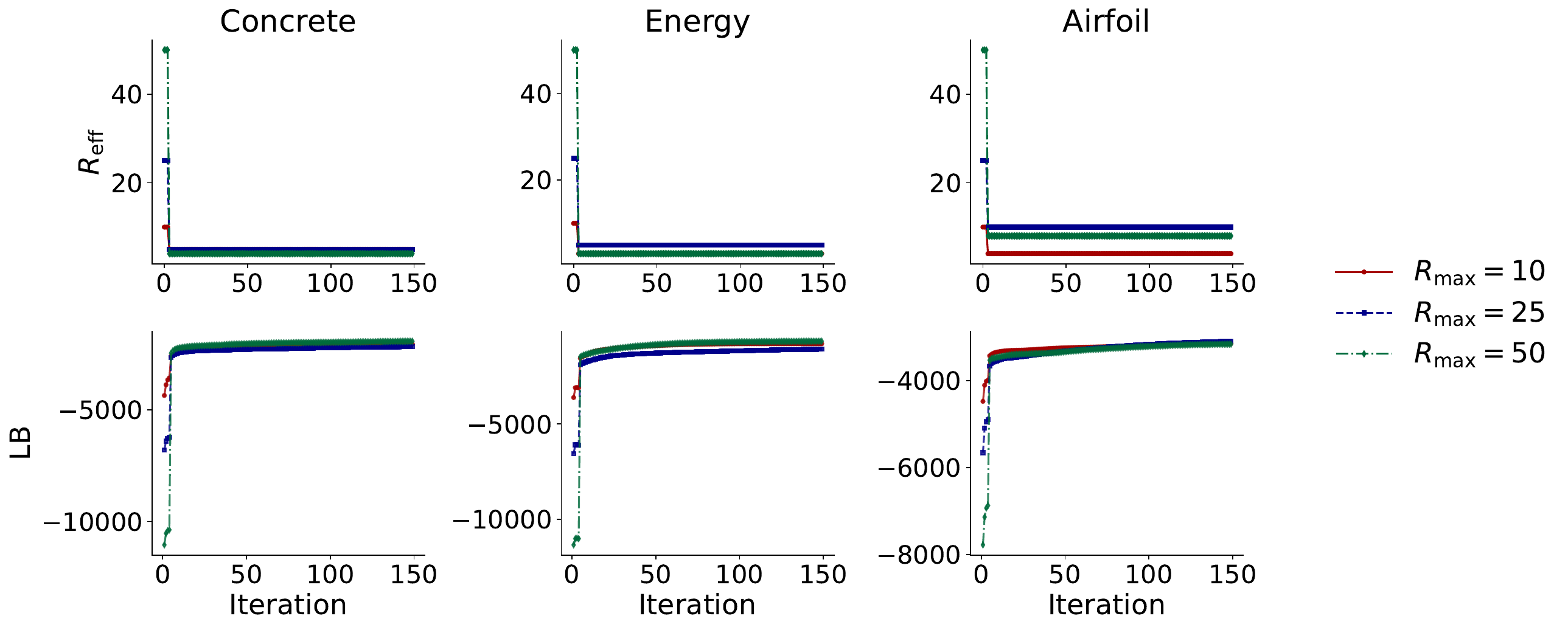}
  \caption{\footnotesize $R_{\text{eff}}$ and the variational lower bound (LB) during the training for varying initial rank values $R_{\text{max}}$. The variational LB and $R_{\text{eff}}$ quickly stabilize early in training, with all models converging to low-rank solutions regardless of the initial $R_{\text{max}}$.}
  \label{fig:plot2}
\end{figure}

In this experiment, we examine the convergence behavior of learning the model by tracking the variational lower bound and effective rank during training. The variational lower bound, which by definition should not decrease at each iteration, serves as a convergence criterion. To evaluate model convergence, we analyzed the behavior of the variational lower bound during training on three UCI data sets: Airfoil, Concrete, and Energy. Furthermore, since the rank is reduced during training, we also analyzed the behavior of the effective rank $R_{\text{eff}}$, specifically to assess whether it converges to a stable value or continues to decrease throughout training. We define $R_{\text{eff}}$ as the number of components that explain more than $10^{-5}$ of the total variance. To compute it, we stack the factor matrices column-wise to form the matrix $\tilde{\mathbf{W}} = [\tilde{\mathbf{W}}^{(1)}, \tilde{\mathbf{W}}^{(2)}, \ldots, \tilde{\mathbf{W}}^{(D)}]$, compute the variance contributions from the diagonal of $\tilde{\mathbf{W}^T} \tilde{\mathbf{W}}$, and count the number of components exceeding the threshold.

\vspace{1mm}
The data sets are split into random training and test sets with $90\%$ and $10\%$ of the data, respectively. Each data set is normalized within the training set so that the features and target values had zero mean and unit variance. All models are initialized as described in Section 3.2.5, with the feature dimension fixed at 20 across all modes. To examine whether the effective rank and lower bound convergence change with different initial rank values, we consider three initial rank values, $R_{\text{max}}$, 10, 25, and 50. To promote low-rank solutions, we set the hyperparameters to $c_0 = 10 ^{-5}$ and $d_0 = 10 ^ {-6}$. 

\vspace{1mm}
In Figure \ref{fig:plot2}, observe that the variational lower bound converges after only a few iterations. Similarly, the effective rank $R_{\text{eff}}$ stabilizes early in training and remains consistent, regardless of the initial rank. In all cases, the models converged to low-rank solutions. Therefore, we assume convergence when the relative change in the variational lower bound is smaller than a user-defined threshold of $10^{-4}$ between iterations; otherwise, training is stopped when the maximum number of iterations, set to 50, is reached.

\subsection{Initial rank and feature dimension}
\label{initial_rank_feature_dim}

\begin{table}[t]
\centering
\scriptsize  
\setlength\tabcolsep{2pt}  
\renewcommand{\arraystretch}{1.1}  
\begin{tabular}{lcc ccc ccc}
& & 
& \multicolumn{3}{c}{\textbf{RMSE}} 
& \multicolumn{3}{c}{\textbf{NLL}} \\
\cmidrule(lr){4-6} \cmidrule(lr){7-9}
\textbf{Data set} & $N$ & $D$ 
& $R_{\text{max}}=10$ & $R_{\text{max}}=25$ & $R_{\text{max}}=50$
& $R_{\text{max}}=10$ & $R_{\text{max}}=25$ & $R_{\text{max}}=50$ \\
\midrule
Airfoil  & 1503 & 5 
& 2.047 $\pm$ 0.204 & \textbf{1.723 $\pm$ 0.151} & 1.786 $\pm$ 0.209 
& 3.491 $\pm$ 0.110 & \textbf{2.865 $\pm$ 0.152} & 2.884 $\pm$ 0.160 \\
Concrete & 1030 & 8 
& 6.165 $\pm$ 1.207 & 5.452 $\pm$ 1.242 & \textbf{5.343 $\pm$ 0.860} 
& 3.403 $\pm$ 0.091 & 3.387 $\pm$ 0.171 & \textbf{3.317 $\pm$ 0.161} \\
Energy   & 768  & 9 
& 1.321 $\pm$ 0.227 & \textbf{1.114 $\pm$ 0.312} & 1.419 $\pm$ 0.288 
& 3.268 $\pm$ 0.450 & \textbf{2.174 $\pm$ 0.542} & 2.466 $\pm$ 0.982 \\
\bottomrule
\end{tabular}
\caption{\footnotesize Average test RMSE and NLL with standard errors for varying initial rank values \( R_{\text{max}} \). Increasing \(R_{\text{max}}\) from 10 to 25 generally improves performance, while further increasing it to 50 yields no consistent improvement in RMSE or NLL.}
\label{tab:rmax_performance}
\end{table}

\begin{table}[t]
\centering
\scriptsize
\setlength\tabcolsep{4pt}
\begin{tabular}{lcc ccc ccc}
& & 
& \multicolumn{3}{c}{\textbf{$R_{\text{eff}} \pm$ Std.}} 
& \multicolumn{3}{c}{\textbf{\% of $R_{\text{max}}$}} \\
\cmidrule(lr){4-6} \cmidrule(lr){7-9}
\textbf{Dataset} & $N$ & $D$ 
& $R_{\text{max}}=10$ & $R_{\text{max}}=25$ & $R_{\text{max}}=50$
& $R_{\text{max}}=10$ & $R_{\text{max}}=25$ & $R_{\text{max}}=50$ \\
\midrule
Airfoil  & 1503 & 5 
& 4.0 $\pm$ 0.0 & 9.5 $\pm$ 0.5 & 7.7 $\pm$ 1.4 
& 40.0\% & 38.0\% & 15.4\% \\
Concrete & 1030 & 8 
& 4.4 $\pm$ 0.7 & 5.3 $\pm$ 0.5 & 5.2 $\pm$ 0.7 
& 44.0\% & 21.2\% & 10.4\% \\
Energy   & 768  & 9 
& 3.0 $\pm$ 0.0 & 4.9 $\pm$ 0.9 & 3.0 $\pm$ 0.0 
& 30.0\% & 19.6\% & 6.0\% \\
\bottomrule
\end{tabular}
\caption{\footnotesize Average effective rank with standard errors and percentage of $R_{\text{eff}}$ relative to $R_{\text{max}}$. The model automatically lowers its $R_{\text{eff}}$ despite a high initial $R_{\text{max}}$ by applying stronger penalties that remove unnecessary components.}
\label{tab:reff_percentages}
\end{table}

We performed another series of experiments to evaluate the impact of varying initial rank values $R_{\text{max}}$ and initial feature dimension values $M_{\text{max}}$ on model performance. Additionally, in this section, we present results illustrating the effect of penalization by $\bm{\Lambda}_{M_d}$ on the rows of the factor matrices, independently of $\Lambda_R$. In the experiments, the same UCI data sets and training procedures from the previous section are used. For each data set, the data splitting is repeated 10 times. The normalization on the targets is removed for prediction and the average test performance of each method is reported.

\vspace{1mm}
To assess how the initial rank affects the test performance, we compared models with different $R_{\text{max}}$ values. Table \ref{tab:rmax_performance}, shows the average test root mean squared error (RMSE) and negative log-likelihood (NLL) for each initial rank.  The best result for each data set is highlighted in bold. Increasing $ R_{\text{max}}$ from 10 to 25 generally improved results, but raising it to 50 did not consistently improve RMSE or NLL. 

\vspace{1mm}
Table~\ref{tab:reff_percentages} presents the average effective rank with standard errors, along with its percentage relative to the initial rank ($R_{\text{max}}$). As $R_{\text{max}}$ increases, the percentage of effective rank consistently decreases across all data sets. This shows that the penalization term $\Lambda_R$ limits model complexity by reducing the use of unnecessary rank components. Although the percentage of retained rank is highest at $R_{\text{max}} = 10$, more than half of the columns in the factor matrices are still eliminated. Consequently, the average effective rank is relatively small. This helps explain the improvement in RMSE observed in Table~\ref{tab:rmax_performance} when $R_{\text{max}}$ is increased from 10 to 25. Given that the model promotes low rank solutions with the hyperparameters $c_0 = 10^{-5}$ and $d_0 = 10^{-6}$, initializing with a low $R_{\text{max}}$ limits the model’s capacity, often leading to underfitting.

\vspace{1mm}
The results also suggest that the degree of penalization increases with higher $R_{\text{max}}$, as reflected by the slightly lower average effective rank when $R_{\text{max}} = 50$ compared to $R_{\text{max}} = 25$. However, as shown in Table~\ref{tab:rmax_performance}, the performance difference between $R_{\text{max}} = 25$ and $R_{\text{max}} = 50$ is marginal. This can be attributed to the fact that once $R_{\text{max}}$ is set sufficiently high, the model is given enough flexibility at initialization and tends to converge to a similar level of complexity. Therefore, we recommend initializing $R_{\text{max}}$ at a sufficiently high value, given the available computational resources. For instance, in data sets with higher dimensionality ($D$), initializing $R_{\text{max}}$ to lower values with lower $\Lambda_R$ can be more computationally efficient, especially since the rank remains fixed during the initial iterations.

\begin{figure}[t]
  \centering
  \includegraphics[width=1\linewidth]{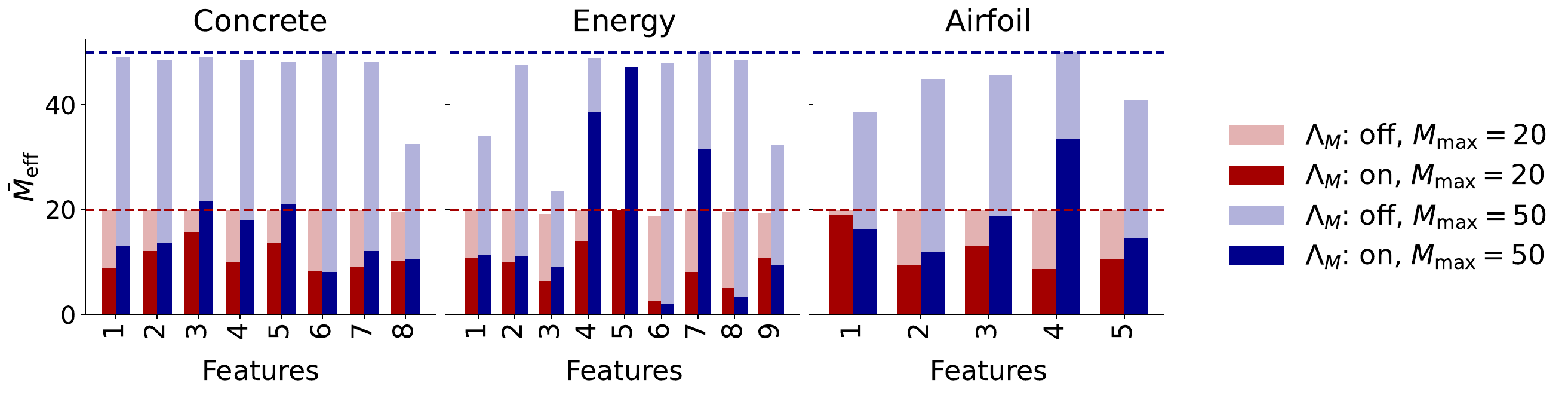}
  \caption{ \footnotesize The average feature dimension, $\Bar{M}_{\text{eff}}$, for each feature across the data sets when $\bm{\Lambda}_{M_d}$ is on and off. The average feature dimension $\bar{M}_{\text{eff}}$ is consistently lower when the row-wise penalization term $\bm{\Lambda}_{M_d}$ is applied.
}
  \label{fig:plot3}
\end{figure}

\begin{table}[t]
\centering
\small
\resizebox{\textwidth}{!}{%
\begin{tabular}{l c
                cc cc cc}

& 
& \multicolumn{2}{c}{\textbf{RMSE}} 
& \multicolumn{2}{c}{\textbf{NLL}} 
& \multicolumn{2}{c}{$\mathbf{R_{\text{eff}}}$} \\
\cmidrule(lr){3-4} \cmidrule(lr){5-6} \cmidrule(lr){7-8}
Dataset & $M_{\text{max}}$ 
& $\lambda_M$-on & $\lambda_M$-off 
& $\lambda_M$-on & $\lambda_M$-off 
& $\lambda_M$-on & $\lambda_M$-off \\
\midrule

Concrete & 20 
& 5.452 $\pm$ 1.242 & \textbf{5.328 $\pm$ 1.192} 
& \textbf{3.387 $\pm$ 0.171} & 4.095 $\pm$ 0.168 
& 5.3 $\pm$ 0.5 & 5.3 $\pm$ 0.5 \\
 & 50 
& \textbf{5.241 $\pm$ 1.075} & 5.490 $\pm$ 0.679 
& \textbf{3.811 $\pm$ 0.203} & 5.381 $\pm$ 0.314 
& 6.3 $\pm$ 0.8 & 6.3 $\pm$ 0.8 \\

Energy & 20 
& \textbf{1.114 $\pm$ 0.312} & 1.150 $\pm$ 0.155 
& \textbf{2.174 $\pm$ 0.543} & 2.810 $\pm$ 0.621 
& 4.9 $\pm$ 0.9 & 4.5 $\pm$ 0.7 \\
 & 50 
& \textbf{1.353 $\pm$ 0.181} & 1.441 $\pm$ 0.202 
& \textbf{3.738 $\pm$ 0.379} & 4.697 $\pm$ 0.374 
& 2.7 $\pm$ 0.6 & 2.7 $\pm$ 0.6 \\

Airfoil & 20 
& \textbf{1.723 $\pm$ 0.151} & 1.746 $\pm$ 0.182 
& \textbf{2.865 $\pm$ 0.152} & 3.046 $\pm$ 0.113 
& 9.5 $\pm$ 0.5 & 9.3 $\pm$ 0.5 \\
 & 50 
& 1.727 $\pm$ 0.181 & \textbf{1.700 $\pm$ 0.155} 
& \textbf{3.557 $\pm$ 0.288} & 3.971 $\pm$ 0.216 
& 7.9 $\pm$ 0.7 & 7.8 $\pm$ 0.7 \\

\bottomrule
\end{tabular}%
}
\caption{\footnotesize Average predictive RMSE and NLL with standard errors, along with $R_{\text{eff}}$, for varying initial input dimension values \( M_{\text{max}} \) when \(\bm{\Lambda}_{M_d}\) is on and off. The row-wise penalization \(\bm{\Lambda}_{M_d}\) improves uncertainty estimation by reducing NLL and promotes row sparsity without affecting the effective rank, helping prevent overfitting especially as \(M_{\text{max}}\) increases.}
\label{tab:lambda_M_on_and_off}
\end{table}

\vspace{2mm}
As shown with the synthetic data experiment in Figure \ref{fig:plot1}, the Hadamard structure of CPD leads strongly penalized columns to become exactly zero, while the corresponding rows tend to remain small but nonzero. Therefore, for both structural and computational reasons, we eliminate zero columns during training but retain the rows. Since this behavior has already been demonstrated on synthetic data, the primary goal of the following experiment is twofold: first, to confirm that penalization on the rows is also effective for real data sets, and second, to assess the impact of the initial feature dimension on test performance. To this end, we fix $R_{\text{max}} = 25$ and consider two values for $M_{\text{max}}$: 25 and 50. We train the models both when the row-wise penalization term $\bm{\Lambda}_{M_d}$ is enabled and disabled. Specifically, setting $\bm{\Lambda}_{M_d} = \bm{I}$ and keeping it fixed throughout training results in no penalization on the rows; we refer to this setting as $\bm{\Lambda}_{M_d}$ : off. Conversely, when $\bm{\Lambda}_{M_d}$ is updated during training, we denote it as $\bm{\Lambda}_{M_d}$ : on.

\vspace{1mm}
We define the effective feature dimension, $M_{\mathrm{eff}}$, for each factor matrix $\tilde{\mathbf{W}}^{(d)}$ as the count of rows whose variance contribution, given by the corresponding diagonal entry of $ \tilde{\mathbf{W}}^{(d)} \tilde{\mathbf{W}}^{(d)T}$, exceeds 0.25\% of the total variance, defined as the sum of all diagonal entries.
Figure~\ref{fig:plot3} shows the average feature dimension, $\Bar{M}_{\text{eff}}$, per feature, computed over 10 train-test splits for each data set with and without the row-wise penalization term $\bm{\Lambda}_{M_d}$. The results show that for both values of $M_{\text{max}}$, the average feature dimension, $\Bar{M}_{\text{eff}}$, is consistently lower when  $\bm{\Lambda}_{M_d}$ is on compared to when it is off. Furthermore, $\bm{\Lambda}_{M_d}$ provides insights into feature relevance. For example, in the Energy data set, feature 6 contributes less to the total variance in the outcome variable compared to feature 5. This concludes that $\bm{\Lambda}_{M_d}$ penalizes less relevant feature dimensions while reinforcing those with higher contributions. 

\vspace{1mm}
Table~\ref{tab:lambda_M_on_and_off} shows that the row-wise penalization $\bm{\Lambda}_{M_d}$ generally reduces NLL across all data sets and initial feature dimension choices, indicating improved uncertainty estimation. RMSE is also slightly better or comparable in all cases. The effective rank $R_{\text{eff}}$ remains stable across penalized and unpenalized settings, indicating that $\bm{\Lambda}_{M_d}$ mainly promotes sparsity in the row space without affecting the columns space \textit{or} the rank structure.

\vspace{1mm}
 Increasing $M_{\text{max}}$ from 20 to 50 does not consistently improve RMSE but leads to higher NLL across all data sets, regardless of whether $\bm{\Lambda}_{M_d}$ is on or off. However, the increase in NLL is consistently smaller when $\bm{\Lambda}_{M_d}$ is on. For example, in the Energy data set, RMSE slightly improves with larger $M_{\text{max}}$ and $\bm{\Lambda}_{M_d}$ on, but NLL worsens without regularization. This suggests that $\bm{\Lambda}_{M_d}$ helps prevent overfitting by controlling complexity through row sparsity.

\subsection{Predictive Performance with baseline}

We considered six UCI data sets in order to compare the predictive performance of BTN-Kernel machines with GP, T-KRR and SP-BTN. The data sets are split into random training and test sets with $90\%$ and $10\%$ of the data, respectively. This splitting process is repeated 10 times and the average test performance is reported. For regression tasks, each data set is normalized within the training set so that the features and target values have zero mean and unit variance. The normalization on the targets is removed for prediction. For binary classification tasks, only the features are normalized based on the training data, and inference is performed by considering the sign of the model response \citep{suykens1999}.

\vspace{2mm}
For the GP models, we followed the procedure described in \cite{wesel_large-scale_2021}, while for the SP-BTN models, we reported the average test RMSE values from the original paper, as the implementation was not available. NLL values or other uncertainty quantification metrics for SP-BTN were not reported in the original work, so we have excluded them from our comparison. For a fair comparison between T-KRR and BTN-Kernel machines, we used polynomial feature mapping with unit norm also for T-KRR. Since BTN-Kernel machines retain the number of rows, we used the same feature dimensionality, denoted as $M$ in Table~\ref{tab:RMSE_NLL}, for both models. For T-KRR, the rank $R$ was set to 10 for all data sets, and for BTN-Kernel machines, we initialized the model with $R_{\text{max}} = 25$ for all data sets except the Adult data set. Due to the high number of features ($D$) in the Adult data set, we set $R_{\text{max}} = 10$ to maintain computational efficiency.

\begin{table}[t]
\centering
\small 

\resizebox{0.99\textwidth}{!}{ 
\begin{tabular}{c c c c c  c c c c  c c}
\multirow{2}{*}{} & \multirow{2}{*}{} & \multirow{2}{*}{} & \multirow{2}{*}{} & \multirow{2}{*}{}
& \multicolumn{4}{c}{\textbf{RMSE}} 
& \multicolumn{2}{c}{\textbf{NLL}} \\
\cmidrule(lr){6-9} \cmidrule(lr){10-11}
Dataset & $N$ & $D$ & $M$ & $R_{\text{eff}}$ 
& \textbf{GP} & \textbf{BTN-Kernels} & \textbf{T-KRR} & \textbf{SP-BTN} 
& \textbf{GP} & \textbf{BTN-Kernels} \\
\midrule
Yacht & 308 & 6 & 20 & 5.2 $\pm$ 0.6 
& 0.402 $\pm$ 0.131 & \textbf{0.379 $\pm$ 0.132} & 0.894 $\pm$ 0.478 & 0.506 $\pm$ 0.091 
&\textbf{0.105 $\pm$ 0.336} & 0.598 $\pm$ 0.782 \\
Energy & 768 & 9 & 20 & 10.1 $\pm$ 1.3 
& 1.296 $\pm$ 0.290 & \textbf{0.456 $\pm$ 0.069} & 0.641 $\pm$ 0.135 & 0.549 $\pm$ 0.200 
& 1.680 $\pm$ 0.223 & \textbf{1.530 $\pm$ 0.271} \\
Concrete & 1030 & 8 & 20 & 5.3 $\pm$ 0.5 
& 5.565 $\pm$ 0.520 & 5.452 $\pm$ 1.242 & \textbf{4.959 $\pm$ 0.620} & 5.500 $\pm$ 0.230 
& \textbf{3.078 $\pm$ 0.080} & 3.387 $\pm$ 0.171 \\
Airfoil & 1503 & 5 & 20 & 6.2 $\pm$ 0.4 
& 2.293 $\pm$ 0.199 & \textbf{1.723 $\pm$ 0.151} & 1.806 $\pm$ 0.144 & - 
& \textbf{2.281 $\pm$ 0.084} & 2.865 $\pm$ 0.152 \\
Spambase & 4601 & 57 & 30 & 8.6 $\pm$ 1.2 
& 0.095 $\pm$ 0.016 & 0.075 $\pm$ 0.015 & \textbf{ 0.066 $\pm$ 0.012} & - 
& 0.720 $\pm$ 0.146 & \textbf{0.499 $\pm$ 0.047} \\
Adult & 45222 & 96 & 40 & 6.3 $\pm$ 0.5 
& N/A & \textbf{0.144 $\pm$ 0.005} & 0.1658 $\pm$ 0.0069 & - 
& N/A & \textbf{0.674 $\pm$ 0.003 }\\
\bottomrule
\end{tabular}
}
\caption{\footnotesize Average predictive RMSE (for regression) or misclassification error (for classification) and NLL with standard errors. BTN-Kernel machines and T-KRR achieve the lowest overall errors, with BTN-Kernel machines matching or outperforming GP in both prediction accuracy and uncertainty estimation, especially on higher-dimensional datasets, while also remaining scalable to large datasets like Adult.}
\label{tab:RMSE_NLL}
\end{table}

\vspace{0.3mm}
All BTN-Kernel machines models are initialized as described in Section~3.2.5. To encourage low-rank solutions, we set the hyperparameters to $c_0 = 10^{-5}$ and $d_0 = 10^{-6}$, with the exception of the Adult data set, for which $R_{\max}$ is initialized to a lower value. The precision hyperparameters are initialized with a slightly more informative prior, using $a_0 = 10^{-2}$ and $b_0 = 10^{-3}$ for all data sets, except for the Airfoil and Concrete data sets. 

\vspace{0.3mm}
Table \ref{tab:RMSE_NLL} shows the average RMSE or misclassification rate for each method, along with the average NLL for GP and BTN-Kernel machines. The best results for each data set are highlighted in bold. Overall, BTN-Kernel machines and T-KRR have the lowest errors, with BTN-Kernel machines achieving the lowest RMSE on 4 out of 6 data sets. SP-BTN was tested only on three low-dimensional data sets, where its performance was similar to other methods, but its effectiveness on higher-dimensional data sets is not known. Although GP performs slightly worse for some data sets like Energy, its RMSE and misclassification rate remain close to the other methods. However, GP is not suitable for very large data sets, such as the Adult data set. All methods perform well on the test data sets, with BTN-Kernel machines and T-KRR generally doing better. Both methods are scalable and suitable for high-dimensional data, but unlike BTN-Kernel machines, T-KRR does not provide uncertainty estimates.

\vspace{0.3mm}
GP and BTN-Kernel machines both offer uncertainty estimates, which is an advantage over the other methods. GP shows slightly better average test NLL on the Yacht, Concrete, and Airfoil data sets, while BTN-Kernel machines performs better on Energy and Spambase. However, the differences between them are small overall. This indicates that BTN-Kernel machines performs similarly to GP, with some advantage on higher-dimensional data sets for both prediction accuracy and uncertainty estimation. For the Adult data set, which is large and high-dimensional, BTN-Kernel machines achieves a considerable level of performance performance with a misclassification rate of 14.4\% and an NLL score of $0.674 \pm 0.003$.

\section{Conclusion}

\label{sec:conclusion}
We have presented BTN Kernel Machines, a probabilistic extension of tensor network-based kernel methods. By placing sparsity-inducing hierarchical priors on the tensor network factors, the model can automatically infer the appropriate tensor rank and feature dimensions. This enables the selection of relevant features, enhances interpretability, and helps prevent overfitting by controlling model complexity. To make Bayesian learning tractable, we employed mean-field variational inference, resulting in a Bayesian ALS algorithm with the same computational complexity as its conventional deterministic counterpart. Hence, BTN Kernel Machines provide uncertainty quantification without incurring additional computational cost. Numerical experiments demonstrate the model’s effectiveness in automatic complexity inference, overfitting prevention, and interpretability. Empirical evaluations on six data sets demonstrate that BTN Kernel Machines achieve competitive performance in both predictive accuracy and uncertainty quantification. The results further show that the model effectively scales to large data sets with high-dimensional inputs, delivering state-of-the-art performance while capturing predictive uncertainty. Potential future work can be the implementation of Tensor Train Kernel machines described in Section \ref{tt-discussion} and the comparison with the CPD case. Furthermore, instead of sparsity-inducing hierarchical priors, alternative prior structures on tensor network factors can be explored, considering both CPD and TT decomposed model weights.

\acks{This publication is part of the project Sustainable learning for Artificial Intelligence from noisy large-scale data (with project number VI.Vidi.213.017) which is financed by the Dutch Research Council (NWO).}


\newpage
\appendix
\section{}

\textbf{1. Proof of Theorem 3}

Derivation of the term in (20) only for one sample. We make use of the multi-linearity property of the CPD and rely on re-ordering the summations:

\small{\begin{align*}
     \bm{\varphi}(x_n)^T\mathbf{w} &= \left\langle \sum_{r=1}^{R} \mathbf{w}_{r}^{(1)} \otimes \mathbf{w}_{r}^{(2)} \otimes \dots \otimes \mathbf{w}_{r}^{(D)}, \bm{\varphi}^{(1)} \otimes \bm{\varphi}^{(2)} \otimes \dots \otimes \bm{\varphi}^{(D)} \right\rangle \\
     &= \sum_{r=1}^{R} \left\langle  \mathbf{w}_{r}^{(1)} \otimes \mathbf{w}_{r}^{(2)} \otimes \dots \otimes \mathbf{w}_{r}^{(D)}, \bm{\varphi}^{(1)} \otimes \bm{\varphi}^{(2)} \otimes \dots \otimes \bm{\varphi}^{(D)} \right\rangle \\
     &=\sum_{r=1}^{R} \sum_{m_1=1}^{M_1} \dots \sum_{m_d=1}^{M_D}  w_{m_1 r}^{(1)} \varphi_{m_1}^{(1)} \dots w_{m_d r}^{(d)} \varphi_{m_d}^{(d)} \dots w_{m_D r}^{(D)} \varphi_{m_D}^{(D)} \\
     &=\sum_{r=1}^{R} \sum_{m_1=1}^{M_1} w_{m_1 r}^{(1)} \varphi_{m_1}^{(1)}  \dots  \sum_{m_d=1}^{M_d} w_{m_d r}^{(d)} \varphi_{m_d}^{(d)} \dots \sum_{m_D=1}^{M_D} w_{m_D r}^{(D)} \varphi_{m_D}^{(D)} \\
      &=\sum_{r=1}^{R} \prod_{d=1}^D  \sum_{m_d=1}^{M_d}  w_{m_d r}^{(d)} \varphi_{m_d}^{(d)}  \\
    &= \sum_{r=1}^{R} \sum_{m_d=1}^{M_d}  w_{m_d r}^{(d)} \left( \varphi_{m_d}^{(d)} \prod_{k \neq d}^D  \sum_{m_k=1}^{M_k}  w_{m_k r}^{(k)} \varphi_{m_k}^{(k)} \right) \\
     &= \sum_{r=1}^{R}\mathbf{w}_{r}^{{(d)}^T} \left( \bm{\varphi}^{(d)} \prod_{k \neq d}^D   \mathbf{w}_{. r}^{{(k)}^T}\bm{\varphi}^{(k)} \right) \\
     &= \operatorname{vec}(\mathbf{W}^{(d)})^T \left(  \left(\mathop{\circledast}_{k \neq d}   \mathbf{W}^{{(k)}^T} \bm{\varphi}^{(k)}  \right) \kronecker \bm{\varphi}^{(d)} \right) \\
\end{align*}
}
\noindent and for $n=1...N$: 

\begin{equation}
    \begin{aligned}
    \mathbf{\Phi}^T \mathbf{w} &=\operatorname{vec}(\mathbf{W}^{(d)})^T \left(   \left(\mathop{\circledast}_{k \neq d}  \mathbf{W}^{{(k)}^T} \mathbf{\Phi}^{(k)}\right) \khatri \mathbf{\Phi} ^{(d)} \right)  \\
    &= \operatorname{vec}(\mathbf{W}^{(d)})^T  \textbf{G}^{(d)}.
\end{aligned}
\end{equation}

Notice that if we do not isolate the $d$-th factor matrix, the data fitting term can also be written as  

\begin{equation*}
    \mathbf{\Phi}^T \mathbf{w} =  \boldsymbol1_R^T \left( \mathop{\circledast}_{d=1}^D\mathbf{W}^{{(d)}^T}  \mathbf{\Phi}^{(d)} \right).  \hspace{15mm}  \hfill \qedsymbol
\end{equation*}

\newpage

\small
\begin{adjustwidth}{-1cm}{} 
\textbf{2. The log of the joint distribution}

\vspace{-1mm}
\begin{equation}
\begin{aligned}
    l(\Theta) &= \sum_{n=1}^{N} \left( \frac{1}{2} \ln \tau - \frac{\tau}{2} \left( y_n -  \bm{\varphi}(x_n)^T\mathbf{w} \right)^2 \right) + \sum_{d=1}^D \sum_r \sum_{m_d} \left( \frac{1}{2} \ln |\lambda_r \lambda_{m_d}| - \frac{1}{2} ( w^{(d)}_{m_dr} \lambda_r \lambda_{m_d} w^{(d)}_{m_dr}) \right) \\
    &+ \sum_r \left( (c_0^r -1 ) \operatorname{ln} \lambda_r - d_0^r\lambda_r \right) + \sum_d^D \sum_{m_d} \left( (g_0^{dm_d} - 1) \operatorname{ln}\lambda_{m_d}^{d} - g_0^{dm_d} \lambda_{m_d}^{d} \right) + (a_0 -1) \operatorname{ln} \tau - b_0 \tau + \text{const} \\
    &= \frac{N}{2} \operatorname{ln} \tau - \frac{\tau}{2} \sum_{n=1}^{N} \left( y_n -  \bm{\varphi}(x_n)^T\mathbf{w} \right)^2+ \frac{ \sum_d M_d}{2} \operatorname{ln} |\bm{\Lambda}_R| + \frac{R}{2} \sum_{d=1}^D \operatorname{ln}|\bm{\Lambda}_{M_d}| 
    - \frac{1}{2}  \sum_{d=1}^D \sum_r \sum_{m_d}   w^{(d)}_{m_dr} \lambda_r \lambda_{m_d} w^{(d)}_{m_dr}  \\
    &+ \sum_r \left( (c_0^r -1)\operatorname{ln}\lambda_r -d_0^r\lambda_r \right) + \sum_d^D \sum_{m_d} \left( (g_0^{dm_d} - 1) \operatorname{ln}\lambda_{m_d}^{d} - g_0^{dm_d} \lambda_{m_d}^{d} \right) + (a_0 -1) \operatorname{ln} \tau - b_0 \tau + \text{const} \\
    &= - \frac{\tau}{2} \| \mathbf{y} -\mathbf{\Phi}^T\mathbf{w} \|_F^{2}   - \frac{1}{2}  \sum_{d=1}^D \sum_r \sum_{m_d}  w^{(d)}_{m_dr} \lambda_r \lambda_{m_d} w^{(d)}_{m_dr}  + \left( \frac{N}{2} + a_0 -1\right)  \operatorname{ln} \tau + \sum_r  \left( \frac{\sum_d M_d}{2} + (c_0^r - 1) \right)\operatorname{ln} \lambda_r \\
    &+  \sum_d^D \sum_{m_d} \left(\frac{R}{2} + (g_0^{dm_d} - 1) \right) \operatorname{ln}\lambda_{m_d}^{d} - \sum_d^D \sum_{m_d}  h_0^{dm_d} \lambda_{m_d}^{d} - \sum_r d_0^r \lambda_r - b_0 \tau + \text{const}
\end{aligned}
\end{equation}

This form will be used frequently in variational Bayesian inference, because the inference of each $\Theta_j$ can be done by 
\begin{equation}
    \operatorname{ln} \hspace{0.3mm} q_j (\Theta_j) = \mathbb{E}_{q(\Theta \backslash \Theta_j)} [\operatorname{ln} p(\mathbf{y}, \mathbf{\Theta})] + \text{const}
\end{equation}

\textbf{3. Factor matrix update}

\begin{equation}
\begin{aligned}
    \operatorname{ln} \hspace{0.3mm} q\left( \operatorname{vec}(
    \mathbf{W}^{(d)})\right) &= \mathbb{E}_{q(\Theta \backslash  \mathbf{W} ^ {(d)})} [\operatorname{ln} p(\mathbf{y}, \{ \mathbf{W} ^ {(d)} \} , \bm{\lambda}_{M_d}, \bm{\lambda}_R, \tau)] + \text{const} \\
    &=\mathbb{E} \left[ - \frac{\tau}{2} \| \mathbf{y} -\mathbf{\Phi}^T\mathbf{w} \|_F^{2}   - \frac{1}{2}  \sum_{d=1}^D \sum_r \sum_{m_d} w^{(d)}_{m_dr} \lambda_r \lambda_{m_d} w^{(d)}_{m_dr} \right] + \text{const} \\            
      &= \mathbb{E} \left[  - \frac{\tau}{2} \left(\mathbf{y} - \mathbf{\Phi}^T\mathbf{w} \right)^2  + - \frac{1}{2} \operatorname{vec}(\mathbf{W}^{(d)})^T\left(\bm{\Lambda}_R \otimes \bm{\Lambda}_{M_d}\right) \operatorname{vec}(\mathbf{W}^{(d)})\right] +  \text{const} \\
    &= \mathbb{E} \left[  - \frac{\tau}{2} \left(\mathbf{y} - \operatorname{vec}(\mathbf{W}^{(d)})^T \textbf{G}^{(d)} \right)^2  + - \frac{1}{2} \operatorname{vec}(\mathbf{W}^{(d)})^T\left(\bm{\Lambda}_R \otimes \bm{\Lambda}_{M_d}\right) \operatorname{vec}(\mathbf{W}^{(d)})\right] +  \text{const} \\
      &= \mathbb{E}\left[-\frac{\tau}{2} \operatorname{vec}(\mathbf{W}^{(d)})^T \textbf{G}^{(d)} \textbf{G}^{(d)T} \operatorname{vec}(\mathbf{W}^{(d)}) - \frac{1}{2} \operatorname{vec}(\mathbf{W}^{(d)})^T \left(\bm{\Lambda}_R \otimes \bm{\Lambda}_{M_d}\right) \operatorname{vec}(\mathbf{W}^{(d)})\right. + \left. \tau \operatorname{vec}(\mathbf{W}^{(d)})^T  \textbf{G}^{(d)} \textbf{y} \right] \\ &+ \text{const} \\
      &= -\frac{1}{2} \operatorname{vec}(\mathbf{W}^{(d)})^T \left( \mathbb{E}\left[\tau\right] \mathbb{E}\left[\textbf{G}^{(d)} \textbf{G}^{(d)T} \right] + \mathbb{E}\left[\bm{\Lambda}_R \otimes \bm{\Lambda}_{M_d}\right] \right)\operatorname{vec}(\mathbf{W}^{(d)}) + \operatorname{vec}(\mathbf{W}^{(d)})^T \mathbb{E}\left[ \tau\right] \mathbb{E}\left[ \textbf{G}^{(d)} \right] \textbf{y} \\
      &+ \text{const} \\
      \\
\end{aligned}
\end{equation}

\newpage
\vspace*{-1.5cm} 
\textbf{4. Proof of Theorem 4}
\small{\begin{equation}
    \begin{aligned}
        \mathbb{E}\left[\mathbf{G}^{(d)} \mathbf{G}^{(d)T} \right] &= \mathbb{E}\left[ \left( \left( \mathop{\circledast}_{k \neq d}  \mathbf{W}^{{(k)}^T} \mathbf{\Phi}^{(k)}\right) \khatri \mathbf{\Phi}^{(d)} \right) \left( \left( \mathop{\circledast}_{k \neq d}  \mathbf{W}^{{(k)}^T} \mathbf{\Phi}^{(k)}\right) \khatri \mathbf{\Phi}^{(d)} \right)^T \right] \\
        &= \sum_{n=1}^{N} \mathbb{E} \left[ \left( \left( \prod_{k \neq d}^{D} \mathbf{W}^{{(k)}^T} \bm{\varphi}^{(k)} \right) \otimes \bm{\varphi}^{(d)} \right) \left( \left( \prod_{k \neq d}^{D} \mathbf{W}^{{(k)}^T} \bm{\varphi}^{(k)} \right) \otimes \bm{\varphi}^{(d)} \right)^T \right] \\
         &=  \sum_{n=1}^{N}  \mathbb{E} \left[\left( \left( \prod_{k \neq d}^{D} \sum_{m_k=1}^{M_k} \mathbf{w}_{m_k .}^{(k)} \varphi_{m_k}^{(k)} \right) \otimes \bm{\varphi}^{(d)} \right)
        \left( \left( \prod_{k \neq d}^{D} \sum_{j_k=1}^{M_k} \mathbf{w}_{j_k .}^{(k)} \varphi_{j_k}^{(k)} \right) \otimes \bm{\varphi}^{(d)} \right)^T \right]\\
        &= \sum_{n=1}^{N} \mathbb{E} \left[ \mathcal{R}\left\{ \left( \bm{\varphi}^{(d)} \kronecker \bm{\varphi}^{(d)} \right) \prod_{k \neq d}^{D} \sum_{m_k=1}^{M_k} \sum_{j_k=1}^{M_k} \varphi_{m_k}^{(k)} \varphi_{j_k}^{(k)} \mathbf{w}_{m_k .}^{(k)} \kronecker \mathbf{w}_{j_k .}^{(k)} \right \}_{M_dR \,\times \,M_dR} \right] \\
        &= \sum_{n=1}^{N}\mathcal{R}\left\{ \left( \bm{\varphi}^{(d)} \kronecker \bm{\varphi}^{(d)} \right) \prod_{k \neq d}^{D} \sum_{m_k=1}^{M_k} \sum_{j_k=1}^{M_k} \varphi_{m_k}^{(k)} \varphi_{j_k}^{(k)} \mathbb{E} \left[ \mathbf{w}_{m_k .}^{(k)} \kronecker \mathbf{w}_{j_k .}^{(k)} \right] \right \}_{M_dR \,\times \,M_dR}\\
        &= \sum_{n=1}^{N} \mathcal{R}\left\{\left( \bm{\varphi}^{(d)} \kronecker \bm{\varphi}^{(d)} \right) \prod_{k \neq d}^{D} \left( \bm{\varphi}^{(k)} \kronecker \bm{\varphi}^{(k)} \right)^T \mathbb{E} \left[ \mathbf{W}^{(k)} \kronecker \mathbf{W}^{(k)} \right] \right\}_{M_dR \,\times \,M_dR}\\
        &= \mathcal{R}\left\{\left( \bm{\Phi}^{(d)} \khatri \bm{\Phi}^{(d)} \right) \prod_{k \neq d}^{D} \left( \bm{\Phi}^{(k)} \khatri \bm{\Phi}^{(k)} \right)^T \mathbb{E} \left[ \mathbf{W}^{(k)} \kronecker \mathbf{W}^{(k)} \right]\right\}_{M_dR \,\times \,M_dR} \\
        &= \mathcal{R}\left\{ \left( \bm{\Phi}^{(d)} \khatri \bm{\Phi}^{(d)} \right) \prod_{k \neq d}^{D} \left( \bm{\Phi}^{(k)} \khatri \bm{\Phi}^{(k)} \right)^T \left( \mathbb{E} \left[ \mathbf{W}^{(k)} \right] \kronecker \mathbb{E} \left[ \mathbf{W}^{(k)} \right] + \operatorname{Var} \left[ \mathbf{W}^{(k)} \kronecker \mathbf{W}^{(k)} \right] \right)\right\}_{M_dR \,\times \,M_dR} \\
        &=  \mathcal{R}\left\{\left( \bm{\Phi}^{(d)} \khatri \bm{\Phi}^{(d)} \right) \mathop{\circledast}_{k \neq d}^{D} \left( \bm{\Phi}^{(k)} \khatri \bm{\Phi}^{(k)} \right)^T \left( \tilde{\mathbf{W}}^{(k)} \kronecker \tilde{\mathbf{W}}^{(k)} +  \mathcal{R} \left\{ \mathbf{\Sigma}^{(k)} \right\}_{M_k^2 \times R^2} \right)\right\}_{M_dR \,\times \,M_dR}, \\
        \\        
    \end{aligned}
\end{equation}}

\vspace{5mm}

In above expression the term  $\operatorname{Var} \left[ \mathbf{W}^{(k)} \kronecker \mathbf{W}^{(k)} \right] $ can be evaluated by reshaping the $\mathbf{\Sigma}^{(k)} \in \mathbb{R}^{MR \times MR}$ into a size $M^2 \times R^2$

\begin{equation}
    \operatorname{Var} \left[ \mathbf{W}^{(k)} \kronecker \mathbf{W}^{(k)} \right]  =  \mathcal{R} \left\{ \mathbf{\Sigma}^{(k)} \right\}_{M_k^2 \times R^2} 
\end{equation}

\newpage
\vspace*{-1cm} 

\textbf{5. The variational posterior distribution of hyperparameter $\bm{\Lambda}_R$}

\vspace{-5mm}
\begin{equation}
\begin{aligned}
    \operatorname{ln} \hspace{1.5mm} q(\bm{\lambda}_R) &= \mathbb{E}_{q(\Theta \backslash  \bm{\lambda})} [\operatorname{ln} p(\mathbf{y}, \{ \mathbf{W} ^ {(d)} \}, \bm{\lambda}_{M_d} , \bm{\lambda}_R, \tau) ] + \text{const} \\    
    &= \mathbb{E} \left[  - \frac{1}{2}  \sum_{d=1}^D \sum_r \sum_{m_d}  w^{(d)}_{m_dr} \lambda_r \lambda_{m_d} w^{(d)}_{m_dr} + \sum_r  \left( \frac{\sum_d M_d}{2} + (c_0^r - 1) \right)\operatorname{ln} \lambda_r - \sum_r d_0^r \lambda_r  \right] + \text{const} \\
    &= \mathbb{E}\left[\sum_r \left\{ - \frac{1}{2} \sum_d \sum_{m_d}  (w^{(d)}_{m_dr} \lambda_{m_d} w^{(d)}_{m_dr})  \lambda_r + \left( c_0^r + \frac{\sum_d M_d}{2}  - 1 \right) \operatorname{ln} \lambda_r -  d_0^r \lambda_r \right\} \right] + \text{const} \\
    &= \mathbb{E}\left[\sum_r \left\{ \left( c_0^r + \frac{\sum_d M_d}{2}  - 1 \right) \operatorname{ln} \lambda_r - \left(d_0^r + \frac{1}{2} \sum_d \sum_{m_d}  w^{(d)}_{m_dr} \lambda_{m_d} w^{(d)}_{m_dr}\right) \lambda_r\right\} \right] + \text{const} \\
    &= \mathbb{E} \left[\sum_r \left\{ \left( c_0^r + \frac{\sum_d M_d}{2}  - 1 \right) \operatorname{ln} \lambda_r - \left(d_0^r + \frac{1}{2} \sum_d  \mathbf{w}^{(d)T}_{r} \bm{\Lambda}_{M_d} \mathbf{w}^{(d)}_{r}\right) \lambda_r\right\} \right] + \text{const} \\
    &= \sum_r \left\{ \left( c_0^r + \frac{\sum_d M_d}{2}  - 1 \right) \operatorname{ln} \lambda_r - \left(d_0^r + \frac{1}{2} \sum_d  \mathbb{E} \left[ \mathbf{w}^{(d)T}_{r} \bm{\Lambda}_{M_d} \mathbf{w}^{(d)}_{r} \right] \right) \lambda_r\right\} +  \text{const}
\end{aligned}
\end{equation}

Thus we observe that the posterior is a Gamma distribution with updated parameters . In above expression, the posterior expectation term can be computed by

\vspace{-3mm}

\begin{equation}
    \begin{aligned}
        \mathbb{E}_q \left[ \mathbf{w}^{(d)T}_{r.} \bm{\Lambda}_{M_d} \mathbf{w}^{(d)}_{r.} \right] &= \mathbb{E}_q \left[\sum_{m_d} \lambda_{m_d} \left(  w^{(d)}_{m_dr}\right)^2 \right] = \sum_{m_d} \lambda_{m_d}  \mathbb{E}_q \left[\left(  w^{(d)}_{m_dr}\right)^2 \right] \\
        &= \sum_{m_d} \lambda_{m_d}  \left\{ \left( \mathbb{E}_q \left[ w^{(d)}_{m_dr}\right]\right)^2 + \operatorname{Var} \left( w^{(d)}_{m_dr}\right) \right\} \\
        &= \tilde{\mathbf{w}}^{(d)T}_{r} \bm{\Lambda}_{M_d} \tilde{\mathbf{w}}^{(d)}_{r} + \bm{\lambda}_{M_d}^T \, \operatorname{Var} \left(\mathbf{w}^{(d)}_{r} \right).
    \end{aligned}
\end{equation}

\vspace{1mm}

\textbf{6. The variational posterior distribution of hyperparameter $\bm{\lambda_{M_d}}$}

\vspace{-5mm}
\begin{equation}
    \begin{aligned}
        \operatorname{ln} \hspace{1.5mm} q(\bm{\lambda_{M_d}}) &= \mathbb{E}_{q(\Theta \backslash  \bm{\lambda})} [\operatorname{ln}  p(\mathbf{y}, \{ \mathbf{W} ^ {(d)} \}, \bm{\lambda}_{M_d} , \bm{\lambda}_R, \tau)] + \text{const} \\  
        &= \mathbb{E} \left[  - \frac{1}{2}  \sum_{d=1}^D \sum_r \sum_{m_d} w^{(d)}_{m_dr} \lambda_r \lambda_{m_d} w^{(d)}_{m_dr} +  \sum_d^D \sum_{m_d} \left(\frac{R}{2} + (g_0^{dm_d} - 1) \right) \operatorname{ln}\lambda_{m_d}^{d} - \sum_d^D \sum_{m_d}  h_0^{dm_d} \lambda_{m_d}^{d}\right]
        + \text{const} \\ 
        &= \mathbb{E} \left[ - \frac{1}{2}  \sum_{m_d} \sum_r  (w^{(d)}_r \lambda_{m_d} w^{(d)}_{m_dr}) \lambda_{m_d} +\sum_{m_d}\left(\frac{R}{2} + (g_0^{dm_d} - 1) \right) \operatorname{ln}\lambda_{m_d} - \sum_{m_d}  h_0^{dm_d} \lambda_{m_d} \right]  
        + \text{const} \\ 
        &= \mathbb{E} \left[ \sum_{m_d} \left\{ \left( g_0^{dm_d} + \frac{R}{2} - 1 \right)\ln \lambda_{m_d} - \left( h_0^{dm_d} + \frac{1}{2} \mathbf{w}^{(d)T}_{m_d} \bm{\Lambda}_{R} \mathbf{w}^{(d)}_{m_d} \right)  \lambda_{m_d} \right\} \right]  + \text{const} \\
        &=  \sum_{m_d} \left\{ \left( g_0^{dm_d} + \frac{R}{2} - 1 \right)\ln \lambda_{m_d} - \left( h_0^{dm_d} + \frac{1}{2} \mathbb{E} \left[ \mathbf{w}^{(d)T}_{m_d} \bm{\Lambda}_R \mathbf{w}^{(d)}_{m_d} \right] \right)  \lambda_{m_d} \right\} + \text{const} 
    \end{aligned}
\end{equation}

Thus we observe that the posterior is a Gamma distribution with updated parameters . In above expression, the posterior expectation term can be computed by

\begin{equation}
    \begin{aligned}
        \mathbb{E}_q \left[  \mathbf{w}^{(d)T}_{m_d} \bm{\Lambda}_R \mathbf{w}^{(d)}_{m_d}  \right] &= \mathbb{E}_q \left[\sum_{r} \lambda_{r} \left(  w^{(d)}_{m_dr}\right)^2 \right] = \sum_{r} \lambda_{r}  \mathbb{E}_q \left[\left(  w^{(d)}_{m_dr}\right)^2 \right] \\
        &= \sum_{r} \lambda_{r}  \left\{ \left( \mathbb{E}_q \left[ w^{(d)}_{m_dr}\right]\right)^2 + \operatorname{Var} \left( w^{(d)}_{m_dr}\right) \right\} \\
        &= \tilde{\mathbf{w}}^{(d)T}_{m_d} \bm{\Lambda}_R \tilde{\mathbf{w}}^{(d)}_{m_d} + \operatorname{Var} \left( \mathbf{w}^{(d)}_{m_d}\right)^T \bm{\lambda}_R.
    \end{aligned}
\end{equation}

\textbf{7. The variational posterior distribution of hyperparameter $\tau$}

\begin{equation}
    \begin{aligned}
        \operatorname{ln} \hspace{1.5mm} q(\tau) &= \mathbb{E}_{q(\Theta \backslash \tau)} [\operatorname{ln} \hspace{1.5mm}  p(\mathbf{y}, \{ \mathbf{W} ^ {(d)} \}, \bm{\lambda}_{M_d} , \bm{\lambda}_R, \tau)] + \text{const} \\  
        &= \mathbb{E} \left[ - \frac{\tau}{2} \| \mathbf{y} -\mathbf{\Phi}^T \mathbf{w} \|_F^{2}  + \left( \frac{N}{2} + a_0 -1\right)  \operatorname{ln} \tau  - b_0 \tau \right] + \text{const} \\
        &=  \left( \frac{N}{2} + a_0 -1\right)  \operatorname{ln} \tau - \left( b_0 + \frac{1}{2} \mathbb{E}_q\left[ \| \mathbf{y} -\mathbf{\Phi}^T \mathbf{w}\|_F^{2}\right] \right) \tau + \text{const}
    \end{aligned}
\end{equation}

\textbf{8. Proof of Theorem 6}

    \begin{equation}
    \begin{aligned}
        \mathbb{E} \left[\|\mathbf{\Phi}^T \mathbf{w} \|_F^2 \right] &=  \mathbb{E} \left[  \|\mathbf{1}_R^T \left( \mathop{\circledast}_{d=1}^D \mathbf{W}^{(d)}  \mathbf{\Phi}^{(d)} \right) \|_F^2 \right] \\
        &= \mathbb{E} \left[ \mathbf{1}_R^T \left( \mathop{\circledast}_{d=1}^D \mathbf{W}^{(d)}  \mathbf{\Phi}^{(d)} \right) \left( \mathop{\circledast}_{d=1}^D \mathbf{W}^{(d)}  \mathbf{\Phi}^{(d)} \right)^T  \mathbf{1}_R\right] \\
        &= \mathbb{E} \left[ \sum_{n} \sum_{r_1} \sum_{r_2} \prod_{d=1}^{D} \sum_{m_d} \sum_{j_d}  \varphi^{(d)}_{m_d} \varphi^{(d)}_{j_d} w^{(d)}_{m_dr_1} w^{(d)}_{j_dr_2}  \right] \\
         &= \mathbb{E} \left[ \sum_{n} \sum_{r_1}\sum_{r_2} \prod_{d=1}^{D}   \left(\bm{\varphi}^{(d)} \otimes \bm{\varphi}^{(d)} \right)^T  \left(\mathbf{w}^{(d)}_{.r_1} \otimes \mathbf{w}^{(d)}_{.r_2}\right)  \right] \\
         &= \sum_{n} \sum_{r_1}\sum_{r_2} \prod_{d=1}^{D}   \left(\bm{\varphi}^{(d)} \otimes \bm{\varphi}^{(d)} \right)^T  \mathbb{E} \left[ \mathbf{w}^{(d)}_{.r_1} \otimes \mathbf{w}^{(d)}_{.r_2}\right]  \\
         &= \sum_{n}  \prod_{d=1}^{D}  \left \{  \left(\bm{\varphi}^{(d)} \otimes \bm{\varphi}^{(d)} \right)^T  \mathbb{E} \left[ \mathbf{W}^{(d)} \otimes \mathbf{W}^{(d)}\right] \right\} \mathbf{1}_{R^2} \\
          &= \mathbf{1}_{N}^T  \prod_{d=1}^{D}  \left \{  \left( \bm{\Phi}^{(d)} \khatri \bm{\Phi}^{(d)} \right)^T  \mathbb{E} \left[ \mathbf{W}^{(d)} \otimes \mathbf{W}^{(d)}\right] \right\} \mathbf{1}_{R^2} \\
         &= \mathbf{1}_{N}^T  \prod_{d=1}^{D}  \left \{  \left( \bm{\Phi}^{(d)} \khatri \bm{\Phi}^{(d)} \right)^T \left( \mathbb{E} \left[ \mathbf{W}^{(k)} \right] \kronecker \mathbb{E} \left[ \mathbf{W}^{(k)} \right] + \operatorname{Var} \left[ \mathbf{W}^{(k)} \kronecker \mathbf{W}^{(k)} \right] \right) \right\} \mathbf{1}_{R^2} \\
         &= \mathbf{1}_{N}^T \left(\mathop{\circledast}_{k \neq d}^{D} \left( \bm{\Phi}^{(k)} \khatri \bm{\Phi}^{(k)} \right)^T \left( \tilde{\mathbf{W}}^{(k)} \kronecker \tilde{\mathbf{W}}^{(k)} +  \mathcal{R} \left\{ \mathbf{\Sigma}^{(k)} \right\}_{M_k^2 \times R^2}\right)\right) \mathbf{1}_{R^2} \hspace{10mm}  \hfill \qedsymbol
         \end{aligned}
\end{equation}

\textbf{9. Posterior expectation of model error or residual}

\begin{equation}
\begin{aligned}
 &\mathbb{E}\left[  \| \mathbf{y} - \mathbf{\Phi}^T \mathbf{w} \|_F^{2}\right] \\
 = &\mathbb{E}\left[\|\mathbf{y}\|_F^{2} - 2 \mathbf{y} ^T \mathbf{\Phi}^T \mathbf{w} +\|\mathbf{\Phi}^T \mathbf{w} \|_F^2 \right] \\
= &\|\mathbf{y}\|_F^{2} - 2 \, \mathbf{y}^T \, (\boldsymbol1_R^T \left( \mathop{\circledast}_{d=1}^D \mathbf{E} \left[\mathbf{W}^{{(d)}} \right]^T \mathbf{\Phi}^{(d)} \right))  + \mathbb{E} \left[ \|\boldsymbol1_R^T \left( \mathop{\circledast}_{d=1}^D \mathbf{W}^{(d)}  \mathbf{\Phi}^{(d)}
 \right) \|_F^2  \right] \\
 =  &\|\mathbf{y}\|_F^{2} - 2 \, \mathbf{y}^T  \,  (\boldsymbol1_R^T \left( \mathop{\circledast}_{d=1}^D \tilde{\mathbf{W}}^{(d)T} \mathbf{\Phi}^{(d)} \right))  + \mathbf{1}_{N}^T \left(\mathop{\circledast}_{k \neq d}^{D} \left( \bm{\Phi}^{(k)} \khatri \bm{\Phi}^{(k)} \right)^T \left( \tilde{\mathbf{W}}^{(k)} \kronecker \tilde{\mathbf{W}}^{(k)} +  \mathcal{R} \left\{ \mathbf{\Sigma}^{(k)} \right\}_{M_k^2 \times R^2}\right)\right) \mathbf{1}_{R^2} 
  \end{aligned}
\end{equation}

\newpage

\vspace*{-1.4cm} 
\textbf{10. Lower bound model evidence}

\vspace{-1mm}
\hspace*{-3cm} 
\begin{equation}
    \begin{aligned}
        L(q) &= \mathbb{E}_{q(\Theta)} [\operatorname{ln} \hspace{0.1mm} p(\mathbf{y}, \mathbf{\Theta})] + H(q(\Theta))  \\
        &= \mathbb{E}_{q( \{\mathbf{W}^{(d)}\}, \tau)} \left[ \operatorname{ln} \hspace{1mm} p(\mathbf{y} \mid \{\mathbf{W}^{(d)}\}_{d=1}^{D}, \tau^{-1})  \right] + \mathbb{E}_{q( \{\mathbf{W}^{(d)}\}, \{\bm{\lambda}_{M_d}\}, \bm{\lambda}_R,)} \left[ \sum_{d=1} ^D \operatorname{ln} \hspace{1mm} p(\mathbf{W}^{(d)} \mid  \bm{\lambda}_R, \bm{\lambda}_{M_d}) )\right] \\
        & + \mathbb{E}_{q(\{\bm{\lambda}_{M_d}\})} \left[ \sum_{d=1} ^D  \operatorname{ln} \hspace{1mm} p(\bm{\lambda}_{M_d})\right]+ \mathbb{E}_{q(\bm{\lambda}_R)} \left[ \operatorname{ln} \hspace{1mm} p(\bm{\lambda}_R)\right]  + \mathbb{E}_{q(\tau)} \left[ \operatorname{ln} \hspace{1mm} p(\tau) \right] - \mathbb{E}_{q(\{\mathbf{W}^{(d)}\}} \left[ \sum_{d=1}^D \operatorname{ln} \hspace{1mm} q(\mathbf{W}^{(d)}) \right]\\  
        &- \mathbb{E}_{q(\{\bm{\lambda}_{M_d}\})} \left[ \sum_{d=1} ^D \operatorname{ln} \hspace{1mm} q(\bm{\lambda}_{M_d})\right] - \mathbb{E}_{q(\bm{\lambda}_R)} \left[ \operatorname{ln} \hspace{1mm} q(\bm{\lambda}_R)\right]   -  \mathbb{E}_{q(\tau)} \left[ \operatorname{ln} \hspace{1mm} q(\tau) \right]\\
    \end{aligned}
\end{equation}

All expectations above are with respect to the posterior distribution \( q \). The first term is the expected log-likelihood; the next four terms are the expected log-priors over $\{\mathbf{W}^{(d)}\}_{d=1}^{D}$, \( \{\bm{\lambda}_{M_d}\}_{d=1}^D \), \( \bm{\lambda}_R \), and \( \tau \), respectively. The final four terms represent the entropies of the posterior distributions over the factor matrices and the hyperparameters. Each term can be computed by

\begin{equation}
    \begin{aligned}
        \mathbb{E}_{q} \left[ \operatorname{ln} \hspace{1mm} p(\mathbf{y} \mid \{\mathbf{W}^{(d)}\}_{d=1}^{D}, \tau^{-1})  \right] 
        &= - \frac{N}{2} \operatorname{ln} \hspace{1mm} (2\pi) + \frac{N}{2}  \mathbb{E}_{q} \left[ \operatorname{ln} \hspace{1mm} \tau \right] -\frac{1}{2} \mathbb{E}_q\left[\| \mathbf{y} - \mathbf{\Phi}^T \mathbf{w}\|_F^{2}\right]  \\
        &=  - \frac{N}{2} \operatorname{ln} \hspace{1mm} (2\pi) + \frac{N}{2} (\psi(a_N) - \operatorname{ln} \hspace{1mm} b_N) -\frac{a_N}{2b_N}  \mathbb{E}_q\left[\| \mathbf{y} - \mathbf{\Phi}^T \mathbf{w}\|_F^{2}\right] 
    \end{aligned}
\end{equation}

\noindent where $N$ denotes the number of observations. $\mathbb{E}_q\left[\| \mathbf{y} - \mathbf{\Phi}^\mathbf{w}\|_F^{2}\right] $ can be computed as show in Theorem 3.

\begin{equation}
    \begin{aligned}
        &\mathbb{E}_{q}  \bigg[ \sum_{d=1} ^D \ln p(\mathbf{W}^{(d)} \mid \bm{\lambda}_R, \bm{\lambda}_{M_d} )\bigg] \\ 
        &= \mathbb{E}_{q} \bigg[ \sum_{d=1}^D \sum_r \sum_{m_d} 
        \bigg(-\frac{\ln (2\pi)}{2}  
        +\frac{1}{2} \ln |\lambda_r \lambda_{m_d}| 
        - \frac{1}{2} w^{(d)}_{m_dr} \lambda_r \lambda_{m_d} w^{(d)}_{m_dr} 
        \bigg) \bigg] \\ 
        &= \mathbb{E}_{q} \bigg[ \sum_{d=1}^D \sum_r \sum_{m_d} 
        \bigg(-\frac{\ln (2\pi)}{2}  
        +\frac{1}{2} \big( \ln |\lambda_r| + \ln|\lambda_{m_d}| \big) 
        - \frac{1}{2} w^{(d)}_{m_dr} \lambda_r \lambda_{m_d} w^{(d)}_{m_dr} 
        \bigg) \bigg] \\ 
        &= \sum_d \bigg\{ 
        - \frac{RM_d}{2} \ln (2\pi) 
        + \frac{M_d}{2} \sum_r \mathbb{E}_{q} [\ln \lambda_r ] 
        + \frac{R}{2} \sum_{m_d} \mathbb{E}_{q} [\ln \lambda_{m_d} ]
        - \frac{1}{2} \mathbb{E}_{q} \bigg[ 
        \operatorname{vec}(\mathbf{W}^{(d)})^T 
        \big( \bm{\Lambda}_R \otimes \bm{\Lambda}_{M_d} \big) 
        \operatorname{vec}(\mathbf{W}^{(d)}) \bigg] 
        \bigg\} \\ 
        &= -\frac{R \sum_d M_d}{2} \ln (2\pi) 
        + \frac{\sum_d M_d}{2} \sum_r \mathbb{E}_{q} [\ln \lambda_r ] 
        + \frac{R}{2} \sum_d \sum_{m_d} \mathbb{E}_{q} [\ln \lambda_{m_d} ] \\
        &- \frac{1}{2} \sum_d  
        \bigg\{ \operatorname{Tr} \bigg( 
        \mathbb{E}_{q} [\bm{\Lambda}_R ] \otimes \mathbb{E}_{q} [\bm{\Lambda}_{M_d}] 
        \operatorname{Var} \big( \operatorname{vec}(\mathbf{W}^{(d)}) \big) \bigg) 
        \quad + \mathbb{E}_{q} \big[ \operatorname{vec}(\mathbf{W}^{(d)})^T \big] 
        \big( \mathbb{E}_{q} [\bm{\Lambda}_R ] \otimes \mathbb{E}_{q} [\bm{\Lambda}_{M_d}] \big) 
        \mathbb{E}_{q} \big[ \operatorname{vec}(\mathbf{W}^{(d)}) \big] 
        \bigg\} \\ 
        &= -\frac{R \sum_d M_d}{2} \ln (2\pi) 
        + \frac{\sum_d M_d}{2} \sum_r (\psi(c_N^r) - \ln d_N^r) 
        + \frac{R}{2} \sum_d \sum_{m_d}(\psi(g^{m_d}_N) - \ln h^{m_d}_N) \\ 
        &\quad - \frac{1}{2} \sum_d \bigg\{ 
        \operatorname{Tr} \bigg( 
        \big( \tilde{\bm{\Lambda}}_R \otimes \tilde{\bm{\Lambda}}_{M_d} \big) \mathbf{V}^{(d)} 
        \bigg) \bigg\} - \frac{1}{2} \sum_d \bigg\{ 
        \operatorname{Tr} \bigg( 
        \operatorname{vec}(\tilde{\mathbf{W}}^{(d)})^T 
        \big( \tilde{\bm{\Lambda}}_R \otimes \tilde{\bm{\Lambda}}_{M_d} \big) 
        \operatorname{vec}(\tilde{\mathbf{W}}^{(d)}) 
        \bigg) \bigg\} \\ 
        &= -\frac{R \sum_d M_d}{2} \ln (2\pi) 
        + \frac{\sum_d M_d}{2} \sum_r (\psi(c_N^r) - \ln d_N^r) + \frac{R}{2} \sum_d \sum_{m_d}(\psi(g^{m_d}_N) - \ln h^{m_d}_N) \\
        &- \frac{1}{2} \operatorname{Tr} \bigg\{ 
        \sum_d \big( \tilde{\bm{\Lambda}}_R \otimes \tilde{\bm{\Lambda}}_{M_d} \big) 
        \bigg( \operatorname{vec}(\tilde{\mathbf{W}}^{(d)}) 
        \operatorname{vec}(\tilde{\mathbf{W}}^{(d)})^T + \mathbf{V}^{(d)} \bigg)  
        \bigg\}.
    \end{aligned}
\end{equation}

 \begin{equation}
     \begin{aligned}
         \mathbb{E}_{q} \left[ \operatorname{ln} \hspace{1mm} p(\bm{\lambda}_R)\right] &= \sum_r\left\{\operatorname{ln} \hspace{1mm} \Gamma (c_0^r) + c_0^r\operatorname{ln} \hspace{1mm} d_o^r + (c_0^r - 1) \mathbb{E}_{q} \left[ \operatorname{ln} \hspace{1mm} \lambda_r\right] - d_0^r \mathbb{E}_{q} \left[ \lambda_r \right]\right\} \\
          &= \sum_r\left\{\operatorname{ln} \hspace{1mm} \Gamma (c_0^r) + c_0^r\operatorname{ln} \hspace{1mm} d_o^r + (c_0^r - 1) (\psi(c_N^r) - \operatorname{ln} \hspace{1mm} d_N^r) - d_0^r \frac{c^r_N}{d^r_N}\right\}  
     \end{aligned}
 \end{equation}

\vspace{2mm}

 \begin{equation}
     \begin{aligned}
         \mathbb{E}_{q} \left[ \operatorname{ln} \hspace{1mm} p(\bm{\lambda}_{M_d})\right] &= \sum_d \sum_{m_d} \left\{\operatorname{ln} \hspace{1mm} \Gamma (g_0^{m_d}) + g_0^{m_d}\operatorname{ln} \hspace{1mm} h_o^{m_d} + (g_0^{m_d} - 1) \mathbb{E}_{q} \left[ \operatorname{ln} \hspace{1mm} \lambda_{m_d}\right] - h_0^{m_d} \mathbb{E}_{q} \left[ \lambda_{m_d} \right]\right\} \\
          &= \sum_d \sum_{m_d} \left\{\operatorname{ln} \hspace{1mm} \Gamma (g_0^{m_d}) + g_0^{m_d}\operatorname{ln} \hspace{1mm} h_o^{m_d} + (g_0^{m_d} - 1) (\psi(g^{m_d}_N) - \ln h^{m_d}_N) - h_0^{m_d}\frac{g_N^{m_d}}{h_N^{m_d}}\right\} 
     \end{aligned}
 \end{equation}

\vspace{2mm}
 \begin{equation}
     \begin{aligned}
         \mathbb{E}_{q} \left[ \operatorname{ln} \hspace{1mm} p(\tau)\right] &= -\operatorname{ln} \hspace{1mm} \Gamma (a_0) + a_0 \operatorname{ln} \hspace{1mm} b_0 + (a_0 - 1) \mathbb{E}_{q}\left[ \operatorname{ln} \hspace{1mm} \tau \right] - b_0  \mathbb{E}_{q} \left[ \tau \right] \\
         &= -\operatorname{ln} \hspace{1mm} \Gamma (a_0) + a_0 \operatorname{ln} \hspace{1mm} b_0 + (a_0 - 1) (\psi(a_N) - \operatorname{ln}  \hspace{1mm} b_N) - b_0  \frac{a_N}{b_N} 
     \end{aligned}
 \end{equation}

\vspace{2mm}
 \begin{equation}
     \begin{aligned}
          - \mathbb{E}_{q} \left[ \sum_{d=1} ^D \operatorname{ln} \hspace{1mm} q(\mathbf{W}^{(d)}) \right] &= \sum_d \sum_r \sum_{m_d} \mathbb{E}_q \left[ \operatorname{ln} \hspace{1mm} q(w_{m_d}^{(d)}) \right] \\
          &= \sum_d \left\{ \frac{1}{2} \operatorname{ln} \mid \mathbf{V}_{m_d} ^ {(d)} \mid \right\} + \frac{R \sum_d M_d}{2} [1 + \operatorname{ln}(2\pi)]
     \end{aligned}
 \end{equation}

\vspace{2mm}
 \begin{equation}
     \begin{aligned}
         -\mathbb{E}_q [\operatorname{ln} \hspace{1mm} q(\bm{\lambda}_R)] = \sum_r \left\{ \operatorname{ln} \hspace{1mm} \Gamma (c_N^r) - (c_N^r -1) \psi(c_N^r) - \operatorname{ln} d_N^r + c_N^r \right\}
     \end{aligned}
 \end{equation}

\vspace{2mm}
 \begin{equation}
     \begin{aligned}
         -\mathbb{E}_q [\operatorname{ln} \hspace{1mm} q(\bm{\lambda}_{M_d})] = \sum_d \sum_{m_d} \left\{ \operatorname{ln} \hspace{1mm} \Gamma (g_N^{m_d}) - (g_N^{m_d} -1) \psi(c_N^{m_d}) - \operatorname{ln} h_N^{m_d} + g_N^{m_d} \right\}
     \end{aligned}
 \end{equation}

\vspace{2mm}
 \begin{equation}
     \begin{aligned}
         -\mathbb{E}_q [\operatorname{ln} \hspace{1mm} q(\tau)] = \operatorname{ln} \hspace{1mm} \Gamma (a_N) - (a_N-1) \psi(a_N) - \operatorname{ln} \hspace{1mm} b_N + a_N 
     \end{aligned}
 \end{equation}

\vspace{2mm}
 In these expressions, $\psi(.)$ denotes the digamma function and $\Gamma(.)$ denotes Gamma function. 

 \newpage
 
 Combining all these terms together and omitting the const term, we obtain the following lower bound in a compact form

 \begin{align*}
     L(q) &=  \frac{N}{2} (\psi(a_N) - \operatorname{ln} \hspace{1mm} b_N) -\frac{a_N}{2b_N} \mathbb{E}_q\left[\| \mathbf{y} - \mathbf{\Phi}^T \mathbf{w}\|_F^{2}\right] + \frac{\sum_d M_d}{2} \sum_r (\psi(c_N^r) - \ln d_N^r) \\
     &+ \frac{R}{2} \sum_d \sum_{m_d}(\psi(g^{m_d}_N) - \ln h^{m_d}_N) 
     - \frac{1}{2} \operatorname{Tr} \bigg\{ 
    \sum_d \big( \tilde{\bm{\Lambda}}_R \otimes \tilde{\bm{\Lambda}}_{M_d} \big) 
    \bigg( \operatorname{vec}(\tilde{\mathbf{W}}^{(d)}) 
    \operatorname{vec}(\tilde{\mathbf{W}}^{(d)})^T + \mathbf{V}^{(d)} \bigg)  
    \bigg\} \\
      & +\sum_r\left\{ (c_0^r - 1) (\psi(c_N^r) - \operatorname{ln} \hspace{1mm} d_N^r) - d_0^r \frac{c^r_N}{d^r_N}\right\} +  \sum_d \sum_{m_d} \left\{ (g_0^{m_d} - 1) (\psi(g^{m_d}_N) - \ln h^{m_d}_N) - h_0^{m_d}\frac{g_N^{m_d}}{h_N^{m_d}}\right\} \\
    &+ (a_0 - 1) (\psi(a_N) - \operatorname{ln}  \hspace{1mm} b_N) - b_0  \frac{a_N}{b_N} + \sum_d \left\{ \frac{1}{2} \operatorname{ln} \mid \mathbf{V}^ {(d)} \mid \right\} +  \sum_r \left\{ \operatorname{ln} \hspace{1mm} \Gamma (c_N^r) - (c_N^r -1) \psi(c_N^r) - \operatorname{ln} d_N^r + c_N^r \right\} \\
     &+ \sum_d \sum_{m_d} \left\{ \operatorname{ln} \hspace{1mm} \Gamma (g_N^{m_d}) - (g_N^{m_d} -1) \psi(c_N^{m_d}) - \operatorname{ln} h_N^{m_d} + g_N^{m_d} \right\} + (a_0 - 1) (\psi(a_N) - \operatorname{ln}  \hspace{1mm} b_N) - b_0  \frac{a_N}{b_N} + \text{const}\\
    \\
    &= \left(\frac{N}{2} + a_0 - a_N \right) \psi(a_N) - \left( \frac{N}{2} + a_0\right) \operatorname{ln} \hspace{1mm} b_N - \frac{a_N}{2b_N} \mathbb{E}_q\left[\| \mathbf{y} - \mathbf{\Phi}^T \mathbf{w}\|_F^{2}\right]  \\
    &- \frac{1}{2} \operatorname{Tr} \bigg\{ 
    \sum_d \big( \tilde{\bm{\Lambda}}_R \otimes \tilde{\bm{\Lambda}}_{M_d} \big) 
    \bigg( \operatorname{vec}(\tilde{\mathbf{W}}^{(d)}) 
    \operatorname{vec}(\tilde{\mathbf{W}}^{(d)})^T + \mathbf{V}^{(d)} \bigg)  
    \bigg\} + \sum_d \left\{ \frac{1}{2} \operatorname{ln} \mid \mathbf{V} ^ {(d)} \mid \right\}\\
    &+ \sum_r \left\{ \left( \frac{\sum_d M_d}{2} + c_0^r - c_N^r \right) \psi(c_N^r) - \left(\frac{\sum_d M_d}{2} + c_0^r \right) \operatorname{ln} \hspace{1mm} d_N^r \right\} \\
    &+ \sum_d \sum_{m_d} \bigg \{ \bigg( \frac{R}{2} + g_0^{m_d} - g_N^{m_d} \bigg) \psi(g_N^{m_d}) - \bigg(\frac{R}{2} + g_0^{m_d}\bigg) \ln h_N^{m_d} \bigg \}\\
    &+ \sum_r \left\{ \operatorname{ln}\hspace{1mm} \Gamma (c_N^r) + c_N^r - d_0^r\frac{c_N^r}{d_N^r} \right\} + \sum_d \sum_{m_d}\bigg\{ \ln \Gamma (g_N^{m_d}) + g_N^{m_d} - h_0^{m_d} \frac{g_N^{m_d}}{h_N^{m_d}}\bigg\} + \operatorname{ln}\hspace{1mm} \Gamma (a_N) + a_N - b_0 \frac{a_N}{b_N}  + \text{const} \\
   \\
    &= - \frac{a_N}{2b_N}\mathbb{E}_q\left[\| \mathbf{y} - \mathbf{\Phi}^T \mathbf{w}\|_F^{2}\right] - \frac{1}{2} \operatorname{Tr} \bigg\{ 
    \sum_d \big( \tilde{\bm{\Lambda}}_R \otimes \tilde{\bm{\Lambda}}_{M_d} \big) 
    \bigg( \operatorname{vec}(\tilde{\mathbf{W}}^{(d)}) 
    \operatorname{vec}(\tilde{\mathbf{W}}^{(d)})^T + \mathbf{V}^{(d)} \bigg)  
    \bigg\} \\
    & + \sum_r \left\{ \operatorname{ln} \hspace{1mm} \Gamma(c_N^r) + c_N^r \left( 1 - \operatorname{ln} \hspace{1mm} d_N^r - \frac{d_0^r}{d_N^r} \right) \right\}  + \sum_d \sum_{m_d} \bigg \{ \ln \Gamma(g_N^{m_d}) + g_N^{m_d} \bigg(1 - \ln h_N^{m_d} - \frac{h_0^{m_d}}{h_N^{m_d}} \bigg) \bigg\} \\
    &+ \frac{1}{2}\sum_d \operatorname{ln} \hspace{1mm} \mid \mathbf{V}^{(d)} \mid+ \operatorname{ln} \hspace{1mm} \Gamma(a_N) + a_N(1- \operatorname{ln} \hspace{1mm} b_N - \frac{b_0}{b_N}) + \text{const}
 \end{align*}

 Finally the lower bound can be computed by using the posterior parameters of all unknowns in $\Theta$, where the posterior parameters are updated at each iteration. 

\newpage

\vspace*{-1cm} 
\textbf{11. Predictive Distribution}

\begin{align*}
    p(\tilde{y_i} \mid \mathbf{y}) &= \int p\left(y_i \mid \Theta\right) p\left(\Theta \mid \mathbf{y} \right)d\Theta \\
    \vspace{2cm}
     &\simeq \int \int p \left(\tilde{y_i} \mid \left\{ \mathbf{W}^{(d)} \right\}, \tau^{-1} \right) q\left( \left\{ \mathbf{W}^{(d)} \right\} \right) q(\tau) d\left\{ \mathbf{W}^{(d)} \right\} d \tau \\ 
    \vspace{2cm}
    &= \int \int \mathcal{N} \left(\tilde{y_i} \mid \mathop{\circledast}_{d=1}^D \mathbf{W}^{(d)}  \mathbf{\varphi}_i^{(d)}, \tau^{-1} \right) \prod_d \mathcal{N}\left(\operatorname{vec}(\mathbf{W}^{(d)}) \mid  \operatorname{vec}(\Tilde{\mathbf{W}}^{(d)}), \mathbf{\Sigma}^{(d)}\right) q(\tau) d \left\{ \operatorname{vec}(\mathbf{W}^{(d)}) \right\} d \tau \\ 
    \vspace{2cm}
    &= \int \int \mathcal{N} \left(\tilde{y_i} \mid \operatorname{vec}(\mathbf{W}^{(1)})^T \textbf{g}^{(d)} (x_n)  , \tau^{-1} \right) \mathcal{N}\left(\operatorname{vec}(\mathbf{W}^{(1)}) \mid  \operatorname{vec}(\Tilde{\mathbf{W}}^{(1)}), \mathbf{V}^{(1)}\right) d \left\{ \operatorname{vec}(\mathbf{W}^{(1)}) \right\}... \\
    &\hspace{7mm}... \prod_{d \neq 1} \mathcal{N}\left(\operatorname{vec}(\mathbf{W}^{(d)}) \mid  \operatorname{vec}(\Tilde{\mathbf{W}}^{(d)}), \mathbf{\Sigma}^{(d)}\right) q(\tau) d \left\{ \operatorname{vec}(\mathbf{W}^{(d)}) \right\}_{d \neq 1} d \tau \\
    \vspace{2cm}
    &= \int \int \mathcal{N} \left(\tilde{y_i} \mid \operatorname{vec}(\tilde{\mathbf{W}}^{(1)})^T \textbf{g}^{(d)} (x_n)  , \tau^{-1} +\textbf{g}^{(d)} (x_n)^T \mathbf{V}^{(1)} \textbf{g}^{(d)} (x_n) \right) ... \\
    &\hspace{7mm} ... \prod_{d \neq 1} \mathcal{N}\left(\operatorname{vec}(\mathbf{W}^{(d)}) \mid  \operatorname{vec}(\Tilde{\mathbf{W}}^{(d)}), \mathbf{\Sigma}^{(d)}\right) q(\tau) d \left\{ \operatorname{vec}(\mathbf{W}^{(d)}) \right\}_{d \neq 1} d \tau \\
    \vspace{2cm}
    & = \int \int \mathcal{N} \left(\tilde{y_i} \mid \operatorname{vec}(\tilde{\mathbf{W}}^{(1)})^T \textbf{g}^{(d)} (x_n)  , \tau^{-1} + \textbf{g}^{(d)} (x_n)^T \mathbf{V}^{(1)} \textbf{g}^{(d)} (x_n)\right) ... \\
     &\hspace{7mm}... \prod_{d \neq 1} \mathcal{N}\left(\operatorname{vec}(\mathbf{W}^{(d)}) \mid  \operatorname{vec}(\Tilde{\mathbf{W}}^{(d)}), \mathbf{\Sigma}^{(d)}\right) q(\tau) d \left\{ \operatorname{vec}(\mathbf{W}^{(d)}) \right\}_{d \neq 1} d \tau \\
    \vspace{2cm} 
    &  \hspace{1cm} \vdots \\
    \hspace{0.5cm}
    &=  \int  \mathcal{N} \left(\tilde{y_i} \mid  \mathop{\circledast}_{d=1}^D \tilde{\mathbf{W}}^{(d)}  \mathbf{\varphi}_i^{(d)}  , \tau^{-1} + \sum_d  \textbf{g}^{(d)} (x_n)^T \mathbf{\Sigma}^{(d)} \textbf{g}^{(d)} (x_n) \right )q(\tau)d \tau \\
    \vspace{2cm}
    &=  \int  \mathcal{N} \left(\tilde{y_i} \mid  \mathop{\circledast}_{d=1}^D \tilde{\mathbf{W}}^{(d)}  \mathbf{\varphi}_i^{(d)}  , \tau^{-1} + \sum_d  \textbf{g}^{(d)} (x_n)^T \mathbf{\Sigma}^{(d)} \textbf{g}^{(d)} (x_n) \right )\text{Ga} (\tau \mid a_N, b_N) d\tau \\
    \vspace{2cm}
    &\simeq \mathcal{T} \left(\tilde{y_i} \mid \mathop{\circledast}_{d=1}^D \tilde{\mathbf{W}}^{(d)}  \mathbf{\varphi}_i^{(d)}  ,  \left\{ \frac{b_N}{a_N}\sum_d \textbf{g}^{(d)} (x_n)^T \mathbf{\Sigma}^{(d)} \textbf{g}^{(d)} (x_n) \right\}^{-1}, 2a_N \right) 
 \end{align*}
\noindent where $\mathcal{T}$ is a Student's t-distribution $\tilde{y_i} \mid \mathbf{y} \sim \mathcal{T}(\tilde{y_i}, \mathcal{S}_i, \nu_y)$ with its parameters given by 

\begin{align*}
    \tilde{y_i} &= \mathop{\circledast}_{d=1}^D \tilde{\mathbf{W}}^{(d)}  \mathbf{\varphi}_i^{(d)}, \qquad \nu_y = 2a_N \\
    \mathcal{S}_i &= \left\{ \frac{b_N}{a_N}\sum_d  \textbf{g}^{(d)} (x_n)^T \mathbf{\Sigma}^{(d)} \textbf{g}^{(d)} (x_n)\right\}^{-1}. 
 \end{align*}

\noindent Thus, the predictive variance can be obtained by $\operatorname{Var}(y_i) = \frac{\nu_y}{\nu_y - 2} \mathcal{S}_i ^{-1}$.

\end{adjustwidth}

\label{app:theorem}

\newpage

\vskip 0.2in
\bibliography{bibliography}

\end{document}